\documentclass[10pt]{article}

\usepackage[preprint]{tmlr}

\usepackage{amsmath,amsfonts,bm}

\def\eqref#1{equation~\ref{#1}}

\def\1{\bm{1}}

\def\vmu{{\bm{\mu}}}

\def\vy{{\bm{y}}}
\def\vz{{\bm{z}}}

\DeclareMathAlphabet{\mathsfit}{\encodingdefault}{\sfdefault}{m}{sl}
\SetMathAlphabet{\mathsfit}{bold}{\encodingdefault}{\sfdefault}{bx}{n}

\usepackage{amssymb}
\usepackage{mathtools}
\usepackage{amsthm}
\usepackage{booktabs}
\usepackage{float}
\usepackage{graphicx}
\usepackage{xcolor}
\usepackage{tikz}
\usepackage{subcaption}
\usepackage{placeins}
\usepackage{wrapfig}
\usetikzlibrary{arrows,arrows.meta,calc,fit,positioning,shapes.geometric}
\usepackage{todonotes}

\usepackage{hyperref}
\usepackage[capitalize,noabbrev]{cleveref}
\creflabelformat{equation}{#2#1#3}
\usepackage{url}

\newcommand{\coderepositoryurl}{https://github.com/Nimrais/NGMP}
\newcommand{\codelink}[1]{\href{\coderepositoryurl}{#1}}

\newtheorem{theorem}{Theorem}

\newcommand{\vlambda}{\bm{\lambda}}
\newcommand{\veta}{\bm{\eta}}
\newcommand{\fisher}{G}
\newcommand{\ent}[1]{\mathbb H\!\left[\, #1 \,\right]}
\newcommand{\famproj}[1]{\Pi_{\mathcal E_{#1}}}

\newcommand{\tangentproj}[1]{\mathsf T_{#1}^{\lambda_{#1}}}

\tikzstyle{line} = [draw, -latex,>=latex]
\tikzstyle{box} = [draw, minimum size=7.0mm]

\title{Information Geometry of Message Passing}

\author{\name Mykola Lukashchuk \email m.lukashchuk@tue.nl \\
      \addr Eindhoven University of Technology 
      \AND
      \name Kyrylo Yemets \\
      \addr Lviv Polytechnic National University
      \AND
      \name Alex Ledbetter \\
      \addr Eindhoven University of Technology
      \AND
      \name{{\.{I}}smail {\c{S}en\"{o}z}} \\
      \addr Lazy Dynamics, Utrecht, the Netherlands
      }

\def\month{MM}
\def\year{YYYY}
\def\openreview{\url{https://openreview.net/forum?id=XXXX}}

\begin{document}

\maketitle

\begin{abstract}
We show that the natural-gradient stationary condition of variational inference has an
edge-local form on a Forney-style factor graph. We start from the Bethe free energy and
constrain a selected edge marginal to an exponential family. At a stationary point, the
natural parameter of that edge equals the sum of two projected messages, one from each
incident factor. Each projected message is the natural-gradient projection
of the exact belief-propagation log-message at the current receiving marginal, or
 equivalently, the gradient of its expectation in the so-called mean coordinates. We call the resulting
scheme natural-gradient message passing (NGMP). The rule is local; each edge may carry its
own exponential family, and the message a factor sends depends on the marginal that
receives it. Compared with variational message passing, NGMP keeps the part of the exact
message that the receiving family can represent instead of averaging the factor under
the neighboring beliefs. The two coincide when the uncertainty on the edges entering a
non-conjugate factor vanishes, and NGMP is more accurate when that uncertainty persists,
for example, along a partially observed latent chain or when parameters are filtered
through successive data batches. Experiments on Poisson smoothing, heteroskedastic
regression, and hourly ETTh forecasting confirm this and show that the gain appears
mainly in uncertainty calibration.
\end{abstract}

\section{Introduction}
Bayesian inference provides a coherent framework for reasoning under uncertainty, justified by the axioms of consistent belief \citep{cox_algebra_2001}. To apply it, we need a model that specifies a joint density $p(\vz, \vy)$ over observations $\vy$ and latent
quantities $\vz$, and Bayes' rule turns this joint density into a posterior
distribution over the unknowns after the observations are revealed
\begin{equation}
    p(\vz\mid\vy)
    \;=\;
    \frac{p(\vy | \vz) p(\vz)}
    {\int p(\vy,\vz)\,\mathrm{d}\vz}.
    \label{eq:intro-bayes-rule}
\end{equation} This posterior is the object one would like to use for prediction, filtering,
decision making, and model comparison. The only difficulty is evaluating the
normalizing constant $\int p(\vy,\vz)\,\mathrm{d}\vz$.

The normalizing constant is often a high-dimensional integral for which closed-form solutions exist only for special model classes, such as conjugate exponential-family models,
linear-Gaussian state-space models, or graphical models with sufficiently simple
structures. When a closed-form solution does exist, the obstruction is usually computational: exact probabilistic inference in Bayesian networks is
NP-hard \citep{cooper_computational_1990}, and even guaranteed approximate
inference is NP-hard in the worst case \citep{DAGUM1993141}. This is
the basic reason why practical Bayesian inference is usually approximate.

Variational inference is one of the standard ways to address this, replacing integration with optimization. Instead of trying
to integrate the posterior exactly, it chooses a tractable family of candidate
distributions and solves an optimization problem inside that family
\citep{jordan_graphical_2001,wainwright_graphical_2008,blei_variational_2017}.
This exchange of integration for optimization is not free: the answer is now limited by the family and the constraints we impose.

There are two complementary ways to do Variational Inference. The first is
\emph{global}: choose one approximating distribution $q_{\lambda}(\vz)$ for the
entire latent state and optimize its parameter $\vlambda$. When $q_{\lambda}$ is an exponential family,
the Fisher geometry gives a clean natural-gradient characterization of
variational stationarity; this is the viewpoint developed by
\citet{khan_information_2025}. The price is that the approximation is specified
as one global object. For a large graphical model, the sufficient statistics,
Fisher matrix, or required expectations of such a global family may be too large
or poorly matched to the factorization that made the model useful in the
first place.

The second route is \emph{local}: it starts from a factorization that enables the conversion from a global problem into one of local computations towards the same objective\footnote{We drop the relation on $\vy$ on the right side of the equation because $\vy$ can always be absorbed into the factors $f_{a}$.} 
\begin{equation}
    p(\vz\mid\vy)
    \;\propto\;
    \prod_{a\in\mathcal V} f_a(\vz_a).
    \label{eq:intro-factorization}
\end{equation}
Each factor $f_a$ involves only a small subset of variables $\vz_a$. In the
factor-graph representation of a probabilistic model formalized in \Cref{sec:forney-factor-graphs}, these variables are
represented by edges, and a belief propagation (BP) message is a local
density function passed across such an edge. On trees, products of these messages compute exact
edge marginals, while on loopy graphs, the same local view leads to Bethe and
structured variational approximations \citep{yedidia_bethe_2001}. More
generally, variational objectives can often be decomposed into local pieces, and
their stationary conditions can be interpreted as message-passing algorithms
\citep{yedidia_constructing_2005,senoz_variational_2021}. The local view is
attractive because it preserves the modularity of the model. A factor
contributes a local computation, an edge carries a local belief (message), and
different modeling components can be combined without rebuilding a single,
monolithic posterior approximation.

Locality alone, however, does not make the messages finitely parameterized objects. An exact BP factor-to-edge update through a
non-conjugate factor may return a function that is not Gaussian, Gamma, or any
other finite-dimensional message supported by the surrounding graph.
Variational message passing (VMP) can restore closure by changing local
constraints, but this also changes what information enters the outgoing message.

This paper asks how the global and local views meet. Can the global
natural-gradient stationary condition of variational inference be localized
into a rule for messages on a factor graph? The answer developed below is yes.
We start from the Bethe free-energy formulation of message passing and impose an
exponential-family form constraint on selected edge marginals.

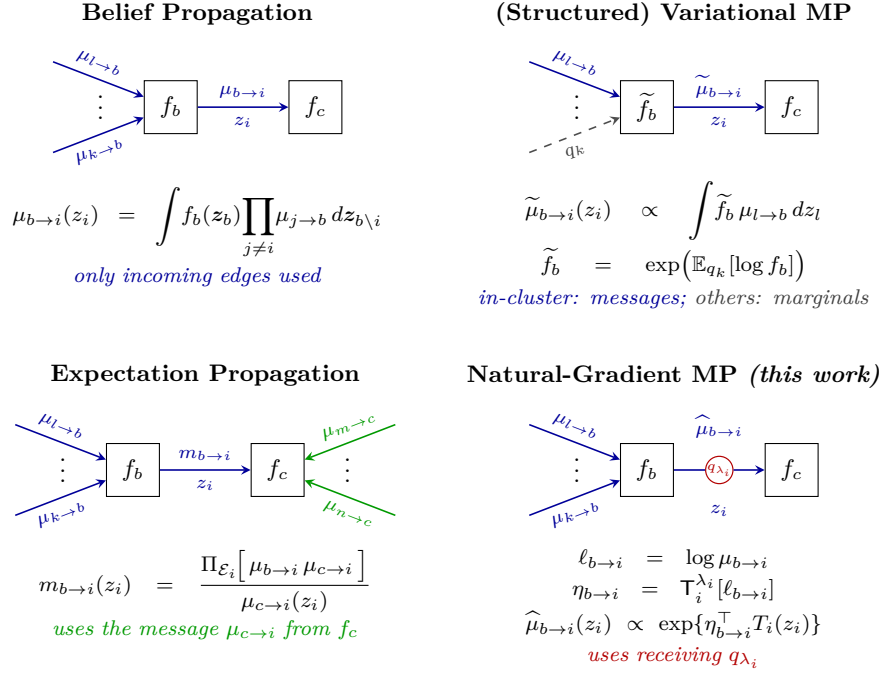
\begin{figure}[t!]
    \centering
%

\begin{tikzpicture}[
    node distance=15mm,
    auto,
    >=stealth',
    every node/.style={font=\small},
    msg/.style       ={-stealth, semithick},
    nmsg/.style      ={-stealth, semithick, blue!60!black},
    omsg/.style      ={-stealth, semithick, green!60!black},
    pmsg/.style      ={-stealth, semithick, red!70!black, line width=0.7pt},
    qcurl/.style     ={red!70!black, thick, -stealth},
    qmarg/.style     ={draw=red!70!black, circle, line width=0.45pt, inner sep=0pt, font=\tiny, text=red!70!black},
    panelttl/.style  ={font=\bfseries\small, anchor=south},
    panelfm/.style   ={font=\footnotesize, anchor=north, align=center, text width=55mm},
]


\begin{scope}[local bounding box=PBP]
    \node[box] (fb)              {$f_b$};
    \node[box, right=12mm of fb] (fc) {$f_c$};
    \coordinate[left=12mm of fb, yshift= 6mm] (cl);
    \coordinate[left=12mm of fb, yshift=-6mm] (ck);
    \node[left=4mm of fb, yshift= 1mm] {$\vdots$};
    \draw[nmsg] (cl) -- node[pos=0.45,above,font=\scriptsize,sloped]{$\mu_{l\to b}$} (fb);
    \draw[nmsg] (ck) -- node[pos=0.45,below,font=\scriptsize,sloped]{$\mu_{k\to b}$} (fb);
    \draw[nmsg] (fb) -- node[midway,above,font=\scriptsize]{$\mu_{b\to i}$}
                       node[midway,below,font=\scriptsize]{$z_i$}        (fc);
\end{scope}
\node[panelttl] at (PBP.north) {\strut Belief Propagation};
\node[panelfm]  at (PBP.south)
    {$\mu_{b\to i}(z_i) \;=\; \displaystyle\int\! f_b(\vz_b)\!\prod_{j\neq i}\!\mu_{j\to b}\,d\vz_{b\setminus i}$ \\[2pt]
     \textcolor{blue!60!black}{\emph{only incoming edges used}}};

%
\begin{scope}[xshift=63mm, local bounding box=PVMP]
    \node[box]                   (vfb) {$\widetilde f_b$};
    \node[box, right=12mm of vfb] (vfc) {$f_c$};
    \coordinate[left=12mm of vfb, yshift= 6mm] (vcl);
    \coordinate[left=12mm of vfb, yshift=-6mm] (vck);
    \node[left=4mm of vfb, yshift= 1mm] {$\vdots$};
    \draw[nmsg] (vcl) -- node[pos=0.45,above,font=\scriptsize,sloped]{$\mu_{l\to b}$} (vfb);
    \draw[-stealth, semithick, gray!60!black, dashed] (vck) -- node[pos=0.45,below,font=\scriptsize,sloped, gray!50!black]{$q_k$} (vfb);
    \draw[nmsg] (vfb) -- node[midway,above,font=\scriptsize]{$\widetilde\mu_{b\to i}$}
                        node[midway,below,font=\scriptsize]{$z_i$}        (vfc);
\end{scope}
\node[panelttl] at (PVMP.north) {\strut (Structured) Variational MP};
\node[panelfm]  at (PVMP.south)
    {$\widetilde\mu_{b\to i}(z_i) \;\propto\; \displaystyle\int\!\widetilde f_b\,\mu_{l\to b}\,dz_l$ \\[2pt]
     $\widetilde f_b \;=\; \exp\!\big(\mathbb E_{q_k}[\log f_b]\big)$ \\[2pt]
     \textcolor{blue!60!black}{\emph{in-cluster: messages;}} \textcolor{gray!50!black}{\emph{others: marginals}}};

%
\begin{scope}[xshift=-5mm, yshift=-48mm, local bounding box=PEP]
    \node[box]                   (efb) {$f_b$};
    \node[box, right=12mm of efb] (efc) {$f_c$};
    \coordinate[left=12mm of efb, yshift= 6mm] (ecl);
    \coordinate[left=12mm of efb, yshift=-6mm] (eck);
    \node[left=4mm of efb, yshift= 1mm] {$\vdots$};
    \draw[nmsg] (ecl) -- node[pos=0.45,above,font=\scriptsize,sloped]{$\mu_{l\to b}$} (efb);
    \draw[nmsg] (eck) -- node[pos=0.45,below,font=\scriptsize,sloped]{$\mu_{k\to b}$} (efb);
    \draw[nmsg] (efb) -- node[midway,above,font=\scriptsize]{$m_{b\to i}$}
                        node[midway,below,font=\scriptsize]{$z_i$}        (efc);
    \coordinate[right=12mm of efc, yshift= 6mm] (ecm);
    \coordinate[right=12mm of efc, yshift=-6mm] (ecn);
    \node[right=4mm of efc, yshift= 1mm] {$\vdots$};
    
    \draw[omsg] (ecm) -- node[pos=0.45,above,font=\scriptsize,sloped]{$\mu_{m\to c}$} (efc);
    \draw[omsg] (ecn) -- node[pos=0.45,below,font=\scriptsize,sloped]{$\mu_{n\to c}$} (efc);
\end{scope}
\node[panelttl] at (PEP.north) {\strut Expectation Propagation};
\node[panelfm]  at (PEP.south)
    {$m_{b\to i}(z_i) \;=\; \dfrac{\famproj{i}\!\big[\,\mu_{b\to i} \, \mu_{c\to i}\,\big]}{\mu_{c\to i}(z_i)}$ \\[2pt]
     \textcolor{green!60!black}{\emph{uses the message $\mu_{c\to i}$ from $f_c$}}};

\begin{scope}[xshift=63mm, yshift=-48mm, local bounding box=PNGMP]
    \node[box]                   (nfb) {$f_b$};
    \node[box, right=12mm of nfb] (nfc) {$f_c$};
    \coordinate[left=12mm of nfb, yshift= 6mm] (ncl);
    \coordinate[left=12mm of nfb, yshift=-6mm] (nck);
    \node[left=4mm of nfb, yshift= 1mm] {$\vdots$};
    \draw[nmsg] (ncl) -- node[pos=0.45,above,font=\scriptsize,sloped]{$\mu_{l\to b}$} (nfb);
    \draw[nmsg] (nck) -- node[pos=0.45,below,font=\scriptsize,sloped]{$\mu_{k\to b}$} (nfb);
    \coordinate[right=6mm of nfb] (nmid);
    \node[qmarg] at (nmid) (nq) {$q_{\scalebox{0.75}{$\lambda_i$}}$};
    \node[above=1.2mm of nq, font=\scriptsize, blue!60!black] {$\widehat{\mu}_{b\to i}$};
    \draw[semithick, blue!60!black] (nfb) -- (nq);
    \draw[nmsg] (nq) -- (nfc);
    \node[below=1.5mm of nq, font=\scriptsize, blue!60!black] {$z_i$};
\end{scope}
\node[panelttl] at (PNGMP.north) {\strut Natural-Gradient MP \emph{(this work)}};
\node[panelfm]  at (PNGMP.south)
    {$\ell_{b\to i}=\log\mu_{b\to i}$ \\[1pt]
     $\eta_{b\to i}=\tangentproj{i}[\ell_{b\to i}]$ \\[2pt]
     $\widehat{\mu}_{b\to i}(z_i)\propto\exp\{\eta_{b\to i}^{\top}T_i(z_i)\}$ \\[2pt]
     \textcolor{red!70!black}{\emph{uses receiving $q_{\lambda_i}$}}};

\end{tikzpicture}
    \caption{Messages produced by belief propagation (BP) rely on the local factor $f_b$ and its other incoming messages. Variational message passing (VMP), including its structured variant, first forms a surrogate factor $\widetilde f_b$ by variational expectation of $f_b$ with respect to the current marginals of the out-of-cluster variables, then passes a BP message through this new factor; non-conjugate and projective VMP keep this flow and additionally project the resulting message at the receiving marginal (\Cref{sec:related-work}). Expectation propagation (EP) instead forms a cavity distribution from the opposing side, restores the exact factor contribution to obtain a tilted marginal, projects that marginal onto the chosen form-constraint family via $\famproj{i}[\cdot]$, and divides out the cavity to obtain the new site \citep{minka_expectation_2001}. Finally, natural-gradient message passing uses the receiving marginal to define the tangent space for natural-gradient projection of the exact BP log-message. The final message towards $z_i$ is obtained as the exponential of the resulting projected natural parameters $\eta_{b\to i}$ applied to the chosen family's sufficient statistics $T_i(z_i)$.}
    \label{fig:compare-inference}
\end{figure}

The idea can be read informally before the formal derivation. Exact BP messages describe the unconstrained local update that a factor would send to an edge if the edge marginal could take an arbitrary functional form. In general, this update does not lie within the finite-dimensional tangent space to the manifold defined by the chosen form constraint (see \cref{fig:ngmp-tangent-projection} for a visual intuition). Understanding variational inference as an iterative process moving variable marginals towards stationary points of an approximate Bayesian objective, we therefore project the exact BP update onto the tangent space of the constrained marginal family at the current marginal belief. This yields a natural-gradient message that encodes the component of the BP update that can be represented by the chosen family's sufficient statistics. The resulting message moves the edge marginal within the constrained family, preserving the locally relevant part of the full BP update while discarding directions that the family cannot encode. This process continues until the BP message encodes only information that the chosen family cannot represent, the difference between the marginal natural parameters and the tangent component is zero, and thus a stationary point has been reached.

\cref{fig:compare-inference} describes and contrasts different local message-computation strategies based on the information sources they draw from and the message-approximations they make in comparison to BP. Belief propagation is included as a basis, but is understood to not produce conjugate messages in general. What makes natural-gradient message passing (NGMP) unique is its optimal use of the exact BP message information visible to the form-constraint family of the receiving edge. Experiments in \cref{sec:edge-uncertainty} locate where this matters: whenever the uncertainty on the edges entering a non-conjugate factor cannot be reduced, the projected messages estimate marginals more accurately than VMP and its projective variants at comparable complexity. Hourly ETTh forecasting supports the corresponding predictive consequence: neural gates can preserve competitive point error while their implied predictive scale collapses, whereas a precision-gated NGMP ensemble remains substantially better calibrated.

The paper is organized to support different entry points.
\Cref{sec:background,sec:variational-message-passing} reviews Forney-style
factor graphs, variational inference, and classical variational message passing
as constrained Bethe-free-energy optimization; readers familiar with
variational inference and factor graphs may begin with
\Cref{sec:ng-message-passing}, which contains our main theoretical
contribution. \Cref{sec:surrogate-models} gives a conceptual account of how to
implement the resulting updates rather than providing a further theoretical contribution,
and may be skipped by readers interested mainly in empirical behavior.
\Cref{sec:related-work} situates the method among existing approaches, while
\Cref{sec:edge-uncertainty,sec:experiments} explains when the method improves
upon classical variational message passing and evaluates it in larger models.
The final section discusses implications and future directions.

\paragraph{Code availability.}
Code and reproducibility materials for all experiments are available in the
\codelink{accompanying repository}.

\section{Background}\label{sec:background}
\subsection{Forney-style Factor Graphs}\label{sec:forney-factor-graphs}

A Forney-style Factor Graph (FFG) $\mathcal{G} = (\mathcal{V},\mathcal{E})$ represents a factorized function,
\begin{equation}
    f(\vz) \;=\; \prod_{a \in \mathcal{V}}f_a(\vz_a)\,, \label{eq:f_fact}
\end{equation}
where $\vz_a$ collects the variables that enter factor $f_a$. A Forney-style
Factor Graph (FFG) represents such a factorization by drawing factors as nodes
and variables as edges \citep{forney_codes_2001,loeliger_factor_2007,
senoz_variational_2021}. An edge is incident to a node exactly when the
corresponding variable is an argument of that factor.

For example, \Cref{fig:example-ffg} represents the factorization
\begin{equation}
\begin{aligned}
    f(\vz)
    \;=\;&
    f_a(z_1)\,f_b(z_2)\,f_c(z_1,z_2,z_3)\,
    f_d(z_4)f_e(z_3,z_4,z_5)\,f_g(z_5,z_6)\,f_h(z_6).
\end{aligned}
\label{eq:example-ffg-factorization}
\end{equation}
The boxes in the figure are the seven local factors, and the six edges exist 
because both factors depend on $z_3$, while the edge $z_6$ connects the pair
$f_g$ and $f_h$. The dashed boxes are not additional factors; they only mark the
parts of the graph whose local eliminations produce the displayed messages.

\begin{figure}[tb]
\centering
\begin{tikzpicture}[node distance=1cm]
    \node[box] (a) {$f_a$};
    \node[box, right=1.3cm of a] (b) {$f_c$};
    \node[box, above=1cm of b]   (c) {$f_b$};
    \node[box, right=1.7cm of c] (d) {$f_d$};
    \node[box, right=1.7cm of b] (e) {$f_e$};
    \node[box, right=1.5cm of e] (g) {$f_g$};
    \node[box, above=1cm of g] (h) {$f_h$};

    \node[draw, dashed, box, inner sep=0.4cm, fit={(d.north west) (h.north east) (g.south east) (e.south west)}] {};
    \node[draw, dashed, box, inner sep=0.4cm, fit={(a.north west) (c.north east) (b.south east) (a.south west)}] {};
    \node[draw, dashed, box, inner sep=0.3cm, fit={(h.north west) (h.north east) (g.south east) (g.south west)}] {};

    \path[line] (a) edge[-] node[pos=0.5, anchor=south, font=\scriptsize]{$z_1$} (b);
    \path[line] (b) edge[-] node[pos=0.5, anchor=east, font=\scriptsize]{$z_2$} (c);
    \path[line] (b) edge[-]
        node[pos=0.14, anchor=south, font=\scriptsize]{$z_3$}
        node[pos=0.5, anchor=south]{$\substack{\mu_{3e}\\\rightarrow}$}
        node[pos=0.5, anchor=north]{$\substack{\leftarrow \\ \mu_{3c}}$}(e);
    \path[line] (d) edge[-] node[pos=0.5, anchor=east, font=\scriptsize]{$z_4$} (e);
    \path[line] (e) edge[-]
        node[pos=0.16, anchor=south, font=\scriptsize]{$z_5$}
        node[pos=0.5, anchor=north]{$\substack{\leftarrow \\ \mu_{5e}}$}(g);
    \path[line] (g) edge[-] node[pos=0.5, anchor=east, font=\scriptsize]{$z_6$} (h);
    
\end{tikzpicture}
\caption{Forney-style Factor Graph representation of the factorization in
\eqref{eq:example-ffg-factorization}.}
\label{fig:example-ffg}
\end{figure}
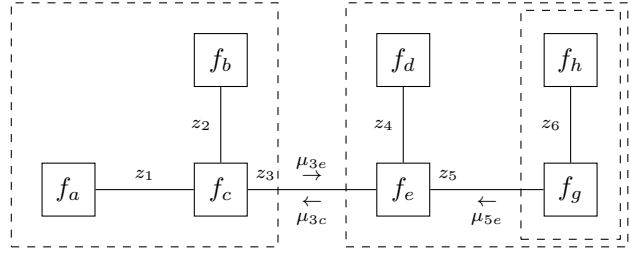

The basic computational idea is to avoid integrating all variables at once. To
form the marginal belief on the middle edge $z_3$, the left subgraph can first
be summarized by a function of $z_3$ alone,
\begin{equation}
    \mu_{3e}(z_3)
    =
    \int f_a(z_1)f_b(z_2)f_c(z_1,z_2,z_3)\,dz_1\,dz_2 .
    \label{eq:example-ffg-left-message}
\end{equation}
Similarly, the right subgraph sends a summary toward $z_3$. In the notation of
the figure,
\begin{subequations}\label{eq:example-ffg-messages}
\begin{align}
    \mu_{5e}(z_5)
    &=
    \int f_g(z_5,z_6)f_h(z_6)\,dz_6, \\
    \mu_{3c}(z_3)
    &=
    \int f_d(z_4)f_e(z_3,z_4,z_5)\mu_{5e}(z_5)\,dz_4\,dz_5, \\
    q_3(z_3)
    &\propto
    \mu_{3e}(z_3)\mu_{3c}(z_3).
\end{align}
\end{subequations}
The updates above are just a repeated use of the distributive law, but the graph
makes the required local computations explicit. Belief propagation is the
systematic version of this operation: each factor combines its local function
with incoming messages from the neighboring edges and sends the resulting
function along another edge.

We denote the neighboring edges of a node $a\in \mathcal{V}$ by
$\mathcal{E}(a)$. Conversely, for an edge $i \in \mathcal{E}$, the notation
$\mathcal{V}(i)$ collects its neighboring nodes. As a notational convention, we
index nodes by $a, b, c$ and edges by $i, j$, unless stated otherwise. In this
paper, we will frequently refer to a subgraph. We define an edge-induced
subgraph by $\mathcal{G}(i) = (\mathcal{V}(i), i)$, and a node-induced subgraph
by $\mathcal{G}(a) = (a, \mathcal{E}(a))$. Furthermore, we denote a local
subgraph by $\mathcal{G}(a, i) = (\mathcal{V}(i), \mathcal{E}(a))$, which
collects all local nodes and edges around $i$ and $a$, respectively. The FFG
formalism of Forney allows edges of degree at most two, $\max |\mathcal{V}(i)| =
2$. We use the terminated FFG formalism of \citet{senoz_variational_2021}, in
which every edge has degree two, $|\mathcal{V}(i)|=2$, by terminating each
half-edge with a factor proportional to $1$.

\subsection{Variational Inference}\label{sec:variational-inference}

For the remainder of this section, we keep the model \eqref{eq:f_fact} but momentarily \emph{forget that it factorizes}. We treat $f(\vz)$ as a single unnormalized density on the joint state $\vz=(\vz_i)_{i\in\mathcal E}$, so that the Bayesian posterior of interest is
\begin{equation}
    p(\vz) \;=\; \frac{f(\vz)}{Z},\qquad Z \;=\; \int f(\vz)\,d\vz, \label{eq:vi-target}
\end{equation}
where $Z$ is the model evidence. As remarked in the introduction, evaluating $Z$ and the marginals of $p$ is intractable in all but the most favorable cases. Variational inference \citep{jordan_graphical_2001,wainwright_graphical_2008,hoffman_stochastic_2013} sidesteps this by replacing integration with optimization: pick a tractable family of candidate densities $\mathcal Q$ and select the $q\in\mathcal Q$ closest to $p$ in some divergence. The canonical choice is the (reverse) Kullback–Leibler divergence.
\begin{equation}
    \mathbb D_{\text{KL}}\!\left[\, q \,\|\, p \,\right] \;=\; \int q(\vz)\log\frac{q(\vz)}{p(\vz)}\,d\vz \;\ \geqslant \; 0,\label{eq:vi-kl}
\end{equation}
with equality iff $q=p$ almost everywhere. Substituting \eqref{eq:vi-target} and rearranging,
\begin{equation}
    \mathbb D_{\text{KL}}\!\left[\, q \,\|\, p \,\right] \;=\; \underbrace{\int q(\vz)\log\frac{q(\vz)}{f(\vz)}\,d\vz}_{\displaystyle =:\,\mathcal F[q]} \;+\;\log Z. \label{eq:vi-vfe-decomposition}
\end{equation}


The functional $\mathcal F[q]$ is the \emph{variational free energy}. Because $\log Z$ is constant in $q$, minimizing $\mathcal F[q]$ over $\mathcal Q$ is equivalent to minimizing $\mathbb D_{\text{KL}}[\, q \,\|\, p \,].$ The further rewriting makes the role of $\mathcal F$ transparent

\begin{equation}
    \mathcal F[q] \;=\; \underbrace{\mathbb E_q[-\log f(\vz)]}_{\text{average energy}} \;\; - \! \underbrace{\ent{q}}_{\text{entropy of }q} \! = \mathbb D_{\text{KL}}\!\left[\, q \,\|\, p \,\right] \;-\; \log Z.\label{eq:vi-vfe-forms}
\end{equation}
The energy–entropy split is the form that appears in statistical physics and underlies the maximum-entropy principle of \citet{jaynes_probability_2003}; the KL form shows that $\mathcal F$ is an upper bound on $-\log Z$, tight when $q=p$ holds almost everywhere. Equivalently, $-\mathcal F$ is the Evidence Lower Bound (ELBO) used throughout the modern variational inference literature \citep{hoffman_stochastic_2013,kucukelbir_automatic_2017}.

\paragraph{Restriction to an exponential family.}
The minimization \eqref{eq:vi-vfe-decomposition} is still over the infinite-dimensional set of densities. The standard tractability move is to restrict $\mathcal Q$ to a regular minimal exponential family \citep{brown_fundamentals_1986,wainwright_graphical_2008}
\begin{equation}
    q_{\lambda}(\vz) \;=\; h(\vz)\exp\!\left\{\vlambda^{\top}T(\vz) - A(\vlambda)\right\},\qquad \vlambda\in\Lambda\subseteq\mathbb R^{d}, \label{eq:vi-expfam}
\end{equation}
with sufficient statistic $T(\vz)$, base measure $h(\vz)$, and \emph{log-partition function}
\begin{equation}
    A(\vlambda) \;=\; \log\int h(\vz)\exp\!\big(\vlambda^{\top}T(\vz)\big)\,d\vz.\label{eq:vi-logpartition}
\end{equation}
Two further objects that we will need throughout the paper follow from the direct differentiation of \eqref{eq:vi-logpartition}:
\begin{subequations}
\begin{align}
    &\vmu(\vlambda) \;=\; \nabla_{\lambda}A(\vlambda) \;=\; \mathbb E_{q_{\lambda}}\!\big[T(\vz)\big],\\&\fisher(\vlambda) \;=\; \nabla^{2}_{\lambda}A(\vlambda) \;=\; \operatorname{Cov}_{q_{\lambda}}\!\big[T(\vz)\big]\succ 0.\label{eq:vi-mean-fisher}
\end{align}
\end{subequations}
 The map $\vlambda\mapsto\vmu(\vlambda)$ is a global diffeomorphism: bijective and smooth, with the smooth inverse $\vmu\mapsto\vlambda(\vmu)$. Thus $\vlambda\in\Lambda$ and $\vmu\in\mathcal M$ are two coordinate systems for the same exponential-family distribution, namely natural parameters and mean parameters. Here $\mathcal M$ is the interior of the convex hull of the support of $T$, and the Jacobian of the coordinate map is precisely the Fisher information $\fisher(\vlambda)$ in natural coordinates \citep{amari_information_2016}. The Legendre dual of $A$ is 
\begin{equation}
    A^{*}(\vmu) \;=\; \sup_{\lambda}\; \left (\vlambda^{\top}\vmu - A(\vlambda) \right) \;=\; \mathbb E_{q_{\lambda(\mu)}}\!\big[\log q_{\lambda(\mu)}(\vz)\big] - \mathbb E_{q_{\lambda(\mu)}}[\log h(\vz)], \label{eq:vi-legendre}
\end{equation} so that $-\ent{q_{\lambda}}$ is, up to the base-measure constant, the dual $A^{*}(\vmu)$ \citep{wainwright_graphical_2008}. Substituting \eqref{eq:vi-legendre} into the free energy, \eqref{eq:vi-vfe-forms}, gives the free energy in natural-parameter coordinates
\begin{equation}
    \mathcal F(\vlambda) \;=\; A^{*}(\vmu(\vlambda)) \;-\; \mathbb E_{q_{\lambda}}[\log f(\vz)] \;+\;\text{const}, \label{eq:vi-parametric-fe}
\end{equation}
which the variational problem now minimizes over the finite-dimensional set $\Lambda$.

\paragraph{The natural gradient and Khan stationary condition.}

The Fisher information \eqref{eq:vi-mean-fisher} equips the exponential family with the Fisher geometry: local displacements in $\vlambda$ are measured by their effect on the corresponding distribution, rather than by the Euclidean distance between parameter vectors. Informally, the \emph{natural gradient} is the steepest ascent direction for a scalar $F:\Lambda\to\mathbb R$ measured in this geometry. More precisely, it is the Riesz representation of the differential of $F$ under the Fisher metric \citep{amari_natural_1998,amari_information_2016}
\begin{equation}
    \widetilde\nabla_{\lambda} F \;:=\; \fisher(\vlambda)^{-1}\nabla_{\lambda}F.\label{eq:vi-natural-gradient}
\end{equation}
A clean identity simplifies this enormously. Because $\vmu = \nabla A(\vlambda)$ has Jacobian $\fisher(\vlambda)$, the chain rule gives $\nabla_{\lambda}F = \fisher(\vlambda)\,\nabla_{\mu}F$ for any $F$ that is smooth on the manifold, hence
\begin{equation}
    \widetilde\nabla_{\lambda} F \;=\; \nabla_{\mu} F.\label{eq:vi-ng-equals-mu-grad}
\end{equation}
\emph{The natural gradient in natural coordinates equals the ordinary gradient in mean coordinates.} This duality is the cornerstone of \citet{amari_natural_1998} and underlies the natural gradient variational inference literature \citep{hoffman_stochastic_2013,khan_fast_2018a}. Applied to the variational free energy \eqref{eq:vi-parametric-fe}, there are two elementary derivative facts \begin{subequations}\label{eq:info-gem-basic-info}
    \begin{align}
        \nabla_{\lambda}A^{*}(\vmu(\vlambda)) &= \fisher(\vlambda)\vlambda, \\
        \nabla_{\lambda}\mathbb E_{q_{\lambda}}[\log f] &= \operatorname{Cov}_{q_{\lambda}}\!\big[T(\vz),\log f(\vz)\big]
    \end{align}
\end{subequations} that combine into the stationary condition
\begin{equation}
    \nabla_{\lambda}\mathcal F(\vlambda) \;=\; 0 \;\Longleftrightarrow\; \vlambda  \;=\; \fisher(\vlambda)^{-1} \nabla_{\lambda} \mathbb E_{q_{\lambda}}[\log f(\vz)]\;=\;  \widetilde\nabla_{\lambda} \mathbb E_{q_{\lambda}}[\log f(\vz)]\;=\; \nabla_{\mu}\mathbb E_{q_{\lambda}}[\log f(\vz)].\label{eq:vi-khan-stationary}
\end{equation}
This is the \textbf{Khan stationary condition} \citep{khan_conjugatecomputation_2017,khan_fast_2018a,khan_bayesian_2023,khan_information_2025}: at any minimizer of the variational free energy over an exponential family, the natural parameter $\vlambda$ equals the natural gradient of the expected log-model. Equation \eqref{eq:vi-khan-stationary} is therefore both an information-geometric \emph{characterization} of the variational fixed point and a \emph{recipe} for finding it by mirror-descent or natural-gradient iteration \citep{khan_conjugatecomputation_2017,khan_bayesian_2023}.

\section{Variational Message Passing}\label{sec:variational-message-passing}

Section~\ref{sec:variational-inference} treated $f(\vz)$ as a structureless joint and committed the variational density to a single global exponential family. We now do the opposite: we keep the variational objective \eqref{eq:vi-vfe-decomposition}, but reinstate the factorization of $f$ through the FFG $\mathcal G$. The functional form of the approximation is then not prescribed as one global density $q_{\lambda}$, as in \eqref{eq:vi-expfam}. It is induced by the admissible belief space and by any additional constraints imposed on that space \citep{senoz_variational_2021}. This is the point we need for the next section: the variational view makes the admissible belief space itself a knob.

\paragraph{Bethe inference and the local polytope.}
The Bethe construction \citep{yedidia_bethe_2001,yedidia_constructing_2005} replaces the joint $q$ by local beliefs: a node belief $q_a(\vz_a)$ for every factor $a\in\mathcal V$ and an edge belief $q_i(z_i)$ for every edge $i\in\mathcal E$. On a TFFG, where every edge has degree two, these beliefs define the Bethe approximation
\begin{equation}
    q(\vz) \;=\; \frac{\prod_{a\in\mathcal V} q_a(\vz_a)}{\prod_{i\in\mathcal E} q_i(z_i)}, \label{eq:vmp-bethe-factorization}
\end{equation}
with admissible beliefs constrained by the local polytope \citep{wainwright_graphical_2008}
\begin{equation}
    \mathcal L(\mathcal G) \;=\;
    \left\{\,(q_a,q_i)\;:\;
    \begin{aligned}
        &\textstyle\int q_a(\vz_a)\,d\vz_a \;=\; 1, &&\forall\,a\in\mathcal V, \\
        &\textstyle\int q_a(\vz_a)\,d\vz_{a\setminus i} \;=\; q_i(z_i), \qquad &&\forall\,a\in\mathcal V,\ i\in\mathcal E(a)
    \end{aligned}
    \right\}. \label{eq:vmp-local-polytope}
\end{equation}
Substituting \eqref{eq:vmp-bethe-factorization} into the variational free energy gives the Bethe free energy \citep{yedidia_constructing_2005,senoz_variational_2021}
\begin{equation}
    \mathcal F_{\mathcal B}\big[\{q_a\},\{q_i\}\big] \;=\; \sum_{a\in\mathcal V}\int q_a(\vz_a)\log\frac{q_a(\vz_a)}{f_a(\vz_a)}\,d\vz_a \;+\; \sum_{i\in\mathcal E} \int q_i(z_i)\log\frac{1}{q_i(z_i)}\,dz_i. \label{eq:vmp-bfe}
\end{equation}
Thus Bethe inference is the constrained variational problem
\[
    \min_{\{q_a,q_i\}\in\mathcal L(\mathcal G)}
    \mathcal F_{\mathcal B}\big[\{q_a\},\{q_i\}\big].
\]
We will use the standard Lagrangian notation for this problem. With scalar multipliers $\psi_a,\psi_i$ for normalization and function-valued multipliers $\lambda_{ai}(z_i)$ for marginalization,
\begin{equation}
\begin{aligned}
    \mathcal L \;=\;\;& \mathcal F_{\mathcal B}\big[\{q_a\},\{q_i\}\big] + \sum_{a\in\mathcal V}\psi_a\!\left(\int q_a(\vz_a)\,d\vz_a - 1\right)
    + \sum_{i\in\mathcal E}\psi_i\!\left(\int q_i(z_i)\,dz_i - 1\right) \\ +\;& \sum_{a\in\mathcal V}\sum_{i\in\mathcal E(a)}\int \lambda_{ai}(z_i)\!\left(q_i(z_i) - \int q_a(\vz_a)\,d\vz_{a\setminus i}\right)dz_i.
\end{aligned}\label{eq:vmp-lagrangian}
\end{equation}

\paragraph{Known stationary consequences.}
By Theorem~1 of \citet{senoz_variational_2021}, stationarity of \eqref{eq:vmp-lagrangian} over the unmodified local polytope gives the ordinary sum-product equations, recovering the Bethe stationary conditions of \citet{yedidia_bethe_2001,yedidia_constructing_2005}. Writing the marginalization multipliers as messages,
\begin{equation}
    \mu_{i\to a}(z_i) \;:=\; \exp\!\big(\lambda_{ai}(z_i)\big), \label{eq:vmp-msg-def}
\end{equation}
the stationary beliefs have the BP product form
\begin{equation}
    q_a^{*}(\vz_a) \;\propto\; f_a(\vz_a)\prod_{i\in\mathcal E(a)}\mu_{i\to a}(z_i), \label{eq:vmp-qa-bp}
\end{equation}
and, on a degree-two edge $\mathcal V(i)=\{b,c\}$,
\begin{equation}
    q_i^{*}(z_i) \;\propto\; \mu_{b\to i}(z_i)\,\mu_{c\to i}(z_i),\qquad \mu_{b\to i}(z_i) \;:=\; \exp\!\big(\lambda_{bi}(z_i)\big). \label{eq:vmp-qi-bp}
\end{equation}
Closing the stationary system with the marginalization constraints recovers the BP update
\begin{equation}
    \mu_{a\to j}^{(k+1)}(z_j) \;=\; \int f_a(\vz_a)\prod_{i\in\mathcal E(a)\setminus j} \mu_{i\to a}^{(k)}(z_i)\,d\vz_{a\setminus j}. \label{eq:vmp-sp-update}
\end{equation}
The important point for us is that BP is what one obtains when the admissible belief space is exactly $\mathcal L(\mathcal G)$: no finite-dimensional form for $q_a$ or $q_i$ has been imposed.

A known way to tune the admissible belief space is to add factorization constraints on node beliefs. Structured VMP splits the incident edges of a factor $a$ into a cluster partition $\mathcal C(a)\subseteq\mathcal P(\mathcal E(a))$ and requires \citep{dauwels_variational_2007,senoz_variational_2021}
\begin{equation}
    q_a(\vz_a) \;=\; \prod_{n\in\mathcal C(a)} q_a^{n}(\vz_a^{n}), \label{eq:vmp-svmp-constraint}
\end{equation}
where $\vz_a^{n}$ collects the variables of cluster $n$. By Theorem~2 of \citet{senoz_variational_2021}, the corresponding stationary update is
\begin{equation}
    q_a^{n,*}(\vz_a^{n}) \;\propto\; \widetilde f_a^{\,n}(\vz_a^{n})\prod_{i\in n}\mu_{i\to a}(z_i),\qquad \widetilde f_a^{\,n}(\vz_a^{n}) \;:=\; \exp\!\left(\,\mathbb E_{\prod_{m\neq n} q_a^{m}}\!\big[\log f_a(\vz_a)\big]\right). \label{eq:vmp-svmp-update}
\end{equation}
The coarsest partition recovers BP, while the finest partition gives mean-field VMP; intermediate partitions give the structured hierarchy described by \citet{senoz_variational_2021}.

\paragraph{Why constraints matter.}
The computational effect is visible in the Normal--Gamma observation model
\[
    x\sim\mathcal N(\mu,\sigma^2),\qquad
    \tau\sim\mathcal G(a,b),\qquad
    y\mid x,\tau\sim\mathcal N(x,\tau^{-1}).
\]
Exact BP through the likelihood factor $f_y(x,\tau)=\mathcal N(y\mid x,\tau^{-1})$ does not preserve Gaussian and Gamma messages simultaneously: a Gamma belief on $\tau$ yields a heavy-tailed, non-Gaussian message to $x$, while a Gaussian belief on $x$ yields a non-Gamma message to $\tau$. The mean-field split $q_y(x,\tau)=q_y^x(x)q_y^\tau(\tau)$ instead replaces $f_y$ by tilted factors of the form \eqref{eq:vmp-svmp-update}. Since $\log f_y(x,\tau)=\frac12\log\tau-\frac12\tau(y-x)^2+\mathrm{const}$, these factors are Gaussian in $x$ and Gamma in $\tau$. Factorization therefore restores finite closure by averaging out the coupling; \Cref{sec:surrogate-vmp} shows the uncertainty correction that this averaging omits.

\paragraph{From node factorizations to edge form constraints.}
Structured VMP illustrates the new degree of freedom exposed by the variational formulation: we can change the feasible set and derive the corresponding message updates from stationarity. A factorization constraint changes the dependence structure of a node belief, but it does not by itself introduce a natural parameter carried by an edge. The next section studies a different modification of the admissible belief space. On selected edges, we impose an exponential-family form constraint
\[
    q_i\in\mathcal E_i,\qquad q_i=q_{\lambda_i}.
\]
This is the step that introduces a finite coordinate $\vlambda_i$ into the otherwise functional Bethe Lagrangian. The exact Lagrange multipliers remain functions, but stationarity with respect to $\vlambda_i$ sees only the tangent directions of the receiving edge family. The result is a finite message: the Fisher-metric projection of the local log-message onto that edge family. With no additional node factorization, this local log-message is the exact BP log-message; under structured VMP, it is the tilted log-message induced by \eqref{eq:vmp-svmp-update}.

\section{Natural-Gradient Message Passing}\label{sec:ng-message-passing}

We now localize the global natural-gradient stationarity condition \eqref{eq:vi-khan-stationary}. In the global formulation, one chooses a single exponential-family density $q_{\lambda}(\vz)$ over all latent variables and obtains
\[
    \vlambda
    =
    \nabla_{\mu}\mathbb E_{q_{\lambda}}[\log f(\vz)].
\]
The factorization of $f$ is invisible to this identity except through the global expectation. The Bethe formulation provides a different knob: instead of choosing one family for the whole joint density, we can change the admissible belief space edge by edge. The question of this section is whether imposing an exponential-family form constraint on a single edge belief recovers a local version of Khan's identity.

\paragraph{Edge-constrained belief space.}
Fix an edge $i\in\mathcal E$ with incident factors $\mathcal V(i)=\{b,c\}$, and equip it with a regular minimal exponential family $\mathcal E_i = \{q_{\lambda_i}\,:\,\vlambda_i\in\Lambda_i\}$. We absorb the carrier into the edge reference measure, so
\begin{equation}
    q_{\lambda_i}(z_i) \;=\; \exp\!\left\{\vlambda_i^{\top}T_i(z_i) - A_i(\vlambda_i)\right\},\qquad \vlambda_i\in\Lambda_i\subseteq\mathbb R^{d_i}, \label{eq:ng-edge-expfam}
\end{equation}
with mean parameter $\vmu_i = \nabla_{\lambda_i}A_i(\vlambda_i)$ and Fisher information $\fisher_i(\vlambda_i)=\nabla^2_{\lambda_i}A_i(\vlambda_i)$. Compared with the local polytope \eqref{eq:vmp-local-polytope}, the new admissible belief space replaces the free edge density $q_i$ with $q_{\lambda_i}$:
\begin{equation}
    \mathcal L_{\mathcal E_i}(\mathcal G)
    =
    \left\{\,
    (\{q_a\},\{q_j\}_{j\neq i},\vlambda_i)
    \;:\;
    \begin{aligned}
        &\textstyle\int q_a(\vz_a)\,d\vz_a=1, &&\forall a\in\mathcal V,\\
        &\textstyle\int q_a(\vz_a)\,d\vz_{a\setminus j}=q_j(z_j), &&\forall a\in\mathcal V,\ j\in\mathcal E(a)\setminus\{i\},\\
        &\textstyle\int q_a(\vz_a)\,d\vz_{a\setminus i}=q_{\lambda_i}(z_i), &&\forall a\in\mathcal V(i)
    \end{aligned}
    \right\}. \label{eq:ng-edge-constrained-polytope}
\end{equation}
The factor beliefs and exact marginalization multipliers remain functional; only the edge marginal has acquired the finite coordinate $\vlambda_i$.

Substituting \eqref{eq:ng-edge-expfam} into the Bethe Lagrangian gives the only new edge-coordinate term
\begin{equation}
    \mathcal L_i\big(\vlambda_i,\{\lambda_{ai}\}_{a\in\mathcal V(i)}\big) \;=\; -\,A_i^{*}\!\big(\vmu_i(\vlambda_i)\big) \;+\; \sum_{a\in\mathcal V(i)}\mathbb E_{q_{\lambda_i}}\!\big[\lambda_{ai}(z_i)\big]. \label{eq:ng-local-lagrangian}
\end{equation}
The first term is the edge entropy written as the negative Legendre dual, and the edge normalization multiplier drops out because $q_{\lambda_i}$ is normalized. Differentiating \eqref{eq:ng-local-lagrangian} gives the new marginal stationarity condition
\begin{equation}
    0
    =
    -\fisher_i(\vlambda_i)\vlambda_i
    +
    \sum_{a\in\mathcal V(i)}
    \operatorname{Cov}_{q_{\lambda_i}}\!\big[T_i,\lambda_{ai}\big].
    \label{eq:ng-lambda-stationary}
\end{equation}

\begin{theorem}[Natural-gradient message passing on a form-constrained edge]
\label{thm:ngmp-edge}
Consider the Bethe free energy \eqref{eq:vmp-bfe} over the edge-constrained belief space $\mathcal L_{\mathcal E_i}(\mathcal G)$ in \eqref{eq:ng-edge-constrained-polytope}. For a real-valued function $r(z_i)$, define the tangent update at the receiving edge marginal by
\begin{equation}
    \tangentproj{i}[r]
    :=
    \fisher_i(\vlambda_i)^{-1}
    \operatorname{Cov}_{q_{\lambda_i}}\!\big[T_i,r\big]
    =
    \nabla_{\mu_i}\,\mathbb E_{q_{\lambda_i}}[r].
    \label{eq:ng-projection}
\end{equation}
The operator $\tangentproj{i}$ maps a function-valued update on the edge to the
finite tangent, equivalently natural-parameter directions visible at the
current receiving marginal $q_{\lambda_i}$.
Let $\lambda_{aj}$ denote the marginalization multipliers of the constrained Bethe Lagrangian. For each incident factor $a\in\mathcal V(i)$, define the exact BP log-message
\begin{equation}
\begin{aligned}
    \mu_{a\to i}(z_i)
    &\;:=\; \int f_a(\vz_a)\prod_{j\in\mathcal E(a)\setminus i}\exp\!\big(\lambda_{aj}(z_j)\big)\,d\vz_{a\setminus i},\\
    \ell_{a\to i}(z_i)
    &\;:=\; \log\mu_{a\to i}(z_i),
\end{aligned}\label{eq:ng-bp-logmsg}
\end{equation}
and its projected natural parameter
\begin{equation}
    \veta_{a\to i}
    :=
    \tangentproj{i}\!\big[\ell_{a\to i}\big]
    =
    \nabla_{\mu_i}
    \mathbb E_{q_{\lambda_i}}\!\big[\ell_{a\to i}(z_i)\big].
    \label{eq:ng-message-eta}
\end{equation}
At any stationary point of the constrained Bethe Lagrangian on the edge $\mathcal V(i)=\{b,c\}$, the edge natural parameter satisfies
\begin{equation}
    \boxed{\;\vlambda_i \;=\; \veta_{b\to i} \;+\; \veta_{c\to i}\;.\;} \label{eq:ng-main-theorem}
\end{equation}
For each $a\in\mathcal{V}(i)$, the corresponding projected message is
\begin{equation}
    \widehat{\mu}_{a\to i}(z_i)
    \;\propto\;
    \exp\!\big\{\veta_{a\to i}^{\top}T_i(z_i)\big\}.
    \label{eq:ng-message-function}
\end{equation}
\end{theorem}

\paragraph{Proof sketch.}
The only new variational variable on edge $i$ is $\vlambda_i$, so differentiating \eqref{eq:ng-local-lagrangian} gives \eqref{eq:ng-lambda-stationary}. Factor-side stationarity still has the BP product form. Closing the marginalization constraint between $q_a$ and $q_{\lambda_i}$ gives, for each $a\in\mathcal V(i)$, the identity
\[
    \log q_{\lambda_i}(z_i)
    =
    \lambda_{ai}(z_i)+\ell_{a\to i}(z_i)+\mathrm{const}.
\]
Applying $\tangentproj{i}$ to this identity and combining the two incident factors with \eqref{eq:ng-lambda-stationary} yields \eqref{eq:ng-main-theorem}. The full algebra is given in \cref{app:ngmp-proof}. \Cref{app:info-gem-message-passing} gives a smaller tutorial derivation of the same mechanism on a one-variable, two-factor graph. Readers who are less familiar with constrained Bethe Lagrangians may find it useful to read that appendix before the full proof.

\paragraph{From the stationary condition to a finite message.}
\Cref{thm:ngmp-edge} should be read as a projected stationarity statement.
It is the localized form of the Khan stationary condition
\eqref{eq:vi-khan-stationary}: the global $\log f(\vz)$ is replaced edge by
edge by the incoming BP
\emph{log-message} $\ell_{a\to i}(z_i)$, the only function of the
surrounding graph visible from edge $i$. The exact BP log-message need not lie
in the tangent space of the receiving exponential family. At the current
marginal $q_{\lambda_i}$, it can be decomposed into the part visible through the
sufficient statistics and an orthogonal residual,
\[
    \ell_{a\to i}
    =
    \veta_{a\to i}^{\top}T_i
    +
    \ell_{a\to i}^{\perp}
    +
    \mathrm{const},
    \qquad
    \tangentproj{i}[\ell_{a\to i}^{\perp}]=0.
\]
Here, the projection is written in log-message coordinates: a log-potential
$r$ represents the infinitesimal multiplicative update
$q_{\lambda_i}^{\epsilon}\propto q_{\lambda_i}\exp\{\epsilon r\}$, whose density
tangent is $q_{\lambda_i}(r-\mathbb E_{q_{\lambda_i}}[r])$. In these
coordinates, the tangent–normal decomposition and the induced projected
message are shown schematically in \Cref{fig:ngmp-tangent-projection}.

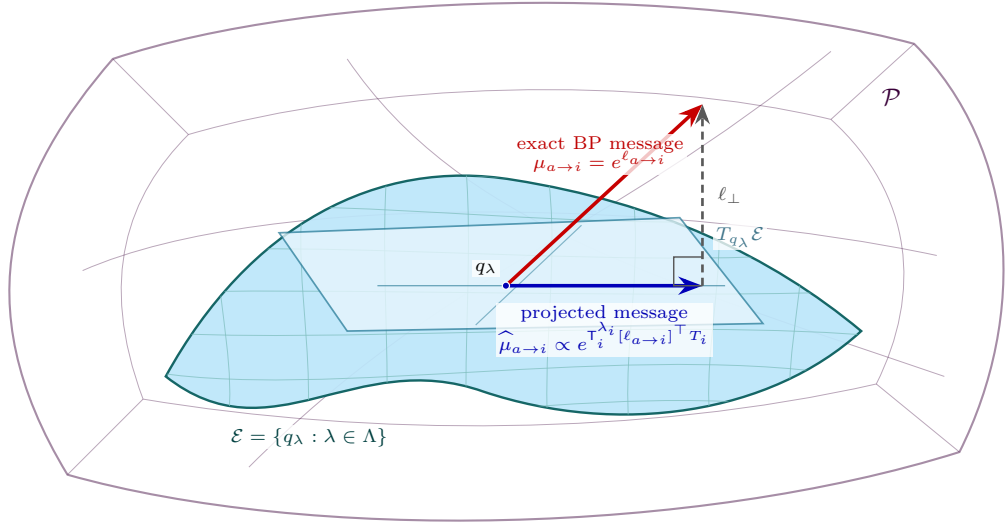
\begin{figure}[tbh]
    \centering

\begin{tikzpicture}[
    x=1mm,
    y=1mm,
    >=Stealth,
    every node/.style={font=\small},
    ambient/.style={draw=violet!45!black, line width=0.35pt, opacity=0.35},
    ambient edge/.style={draw=violet!45!black, line width=0.8pt, opacity=0.45},
    surface/.style={fill=cyan!24, draw=teal!70!black, line width=0.9pt, opacity=0.92},
    surface grid/.style={draw=teal!50!white, line width=0.32pt, opacity=0.75},
    tangent plane/.style={fill=cyan!8!white, draw=cyan!55!black,
                          line width=0.65pt, opacity=0.78},
    exact msg/.style={->, red!78!black, line width=1.35pt},
    projected msg/.style={->, blue!72!black, line width=1.35pt},
    normal msg/.style={->, gray!70!black, line width=1.0pt, densely dashed},
    point/.style={circle, fill=blue!70!black, draw=white,
                  line width=0.45pt, inner sep=1.35pt},
    label/.style={font=\small, inner sep=1.2pt,
                  fill=white, fill opacity=0.78, text opacity=1},
    smalllabel/.style={font=\scriptsize, inner sep=1.0pt,
                       fill=white, fill opacity=0.78, text opacity=1}
]

\coordinate (Pfl) at (-58,-25);
\coordinate (Pfr) at ( 60,-22);
\coordinate (Pbr) at ( 54, 32);
\coordinate (Pbl) at (-52, 30);
\coordinate (Qfl) at (-48,-15);
\coordinate (Qfr) at ( 49,-13);
\coordinate (Qbr) at ( 43, 22);
\coordinate (Qbl) at (-42, 20);

\draw[ambient edge]
    (Pfl) .. controls (-72,-2) and (-65,15) .. (Pbl)
          .. controls (-25,39) and ( 24,40) .. (Pbr)
          .. controls ( 68,18) and ( 69,-4) .. (Pfr)
          .. controls ( 25,-33) and (-26,-34) .. cycle;
\draw[ambient]
    (Qfl) .. controls (-54,1) and (-50,12) .. (Qbl)
          .. controls (-18,27) and ( 18,28) .. (Qbr)
          .. controls ( 54,12) and ( 55,-1) .. (Qfr)
          .. controls ( 18,-20) and (-18,-20) .. cycle;
\draw[ambient] (Pfl) -- (Qfl);
\draw[ambient] (Pfr) -- (Qfr);
\draw[ambient] (Pbr) -- (Qbr);
\draw[ambient] (Pbl) -- (Qbl);
\draw[ambient]
    (-56,2) .. controls (-31,13) and (24,11) .. (57,4);
\draw[ambient]
    (-34,-24) .. controls (-8,2) and (12,6) .. (43,31);
\draw[ambient]
    (-21,30) .. controls (-6,8) and (20,1) .. (58,-12);
\node[label, text=violet!45!black] at (51,25) {$\mathcal P$};

\def\surfacepath{
    (-45,-12)
    .. controls (-34,  8) and (-17, 17) .. (  1, 14)
    .. controls ( 20, 11) and ( 35,  4) .. ( 47, -6)
    .. controls ( 35,-17) and ( 15,-20) .. ( -3,-14)
    .. controls (-20, -8) and (-31,-23) .. (-45,-12)
}

\path[surface] \surfacepath -- cycle;

\begin{scope}
    \clip \surfacepath -- cycle;
    \foreach \yy/\bend in {-16/-7,-11/-5,-6/-2,-1/1,4/3,9/5,14/7} {
        \draw[surface grid]
            (-49,\yy) .. controls (-25,\yy+\bend) and
            (21,\yy-\bend) .. (51,\yy+0.4*\bend);
    }
    \foreach \xx/\bend in {-38/5,-27/3,-16/1,-5/-1,6/-2,17/-3,28/-4,39/-5} {
        \draw[surface grid]
            (\xx,-24) .. controls (\xx+\bend,-9) and
            (\xx-\bend,5) .. (\xx+0.6*\bend,19);
    }
\end{scope}
\node[smalllabel, text=teal!50!black, align=center] at (-26,-20)
    {$\mathcal E=\{q_\lambda:\lambda\in\Lambda\}$};

\coordinate (q)      at (0,0);
\coordinate (proj)   at (26,0);
\coordinate (exact)  at (26,24);

\path[tangent plane]
    (-21,-6) -- (34,-5) -- (23,9) -- (-30,7) -- cycle;
\draw[cyan!55!black, line width=0.45pt, opacity=0.75] (-17,0) -- (29,0);
\draw[cyan!55!black, line width=0.45pt, opacity=0.55] (-4,-5.2) -- (10,8.0);
\node[smalllabel, text=cyan!45!black] at (31,6.4) {$T_{q_\lambda}\mathcal E$};

\draw[projected msg] (q) -- (proj)
    node[midway, below=2.0mm, smalllabel, text=blue!70!black, align=center]
    {projected message\\[-1pt]
     $\widehat\mu_{a\to i}\propto
      e^{\tangentproj{i}[\ell_{a\to i}]^{\top}T_i}$};
\draw[normal msg] (proj) -- (exact)
    node[midway, right=1.5mm, smalllabel, text=gray!65!black]
    {$\ell_{\perp}$};
\draw[exact msg] (q) -- (exact)
    node[pos=0.48, above=3.0mm, smalllabel, text=red!75!black, align=center]
    {exact BP message\\[-1pt] $\mu_{a\to i}=e^{\ell_{a\to i}}$};

\fill[point] (q) circle[radius=1.35pt];
\node[smalllabel, above left=1.1mm and 0.7mm of q] {$q_\lambda$};

\draw[gray!70!black, line width=0.55pt]
    (proj) -- ++(-3.8,0) -- ++(0,3.8) -- ++(3.8,0);

\end{tikzpicture}
    \caption{Geometric view of the projected message. The exact BP message is
    $\mu_{a\to i}=\exp\{\ell_{a\to i}\}$, a tangent in the ambient probability
    space $\mathcal{P}$. The decomposition is drawn in log-message coordinates:
    $\ell_{a\to i}$ is split into the tangent component
    $\veta_{a\to i}^{\top}T_i$ and the orthogonal residual
    $\ell_{a\to i}^{\perp}$ at the current marginal $q_{\lambda_i}$. The
    actual finite message sent along the constrained edge is obtained by
    exponentiation of the tangent component, $\widehat{\mu}_{a\to i}\propto
    \exp\{\veta_{a\to i}^{\top}T_i\}$. The residual may remain in the
    functional Lagrange multipliers, but it is invisible to the
    finite-dimensional edge marginal. Making the ambient space picture fully
    rigorous would require a non-parametric density-manifold formulation; see
    \citet{wang2020information} for tangent and cotangent spaces of
    probability densities and Fisher--Rao/Wasserstein information metrics. We
    use the ambient manifold here informally to explain what is the finite-dimensional
    projection.}
    \label{fig:ngmp-tangent-projection}
\end{figure}

The exact functional Lagrange multipliers can contain this residual, but the
finite edge marginal responds only to the tangent component defined in \eqref{eq:ng-message-eta}.

\paragraph{Natural-gradient message passing.}
We call natural-gradient message passing the finite message passing scheme that
uses this tangent component as the natural parameter of the message propagated
to the rest of the graph.
That is, instead of sending the exact BP message
$\mu_{a\to i}$ when it is not representable in the chosen
family, NGMP sends
\[
    \widehat{\mu}_{a\to i}(z_i)
    \propto
    \exp\{\veta_{a\to i}^{\top}T_i(z_i)\}.
\]
Equivalently, NGMP sets the orthogonal residual
$\ell_{a\to i}^{\perp}$ to zero for communication purposes and keeps only the
actual finite-dimensional change visible to the receiving family. Thus
\Cref{thm:ngmp-edge} motivates NGMP by identifying the projected fixed-point
condition that these finite messages must satisfy.

The extension from one constrained edge to an arbitrary subset
$\mathcal S\subseteq\mathcal E$ is edge-local. Imposing
$q_i\in\mathcal E_i$ on every $i\in\mathcal S$ gives one equation of the form
\eqref{eq:ng-main-theorem} per constrained edge, while the exact BP
log-messages \eqref{eq:ng-bp-logmsg} still couple those equations through the
surrounding graph. Unconstrained edges continue to obey the classical BP update
of Section~\ref{sec:variational-message-passing}. If the edge form constraint
is the only additional constraint touching factor $a$, then $\ell_{a\to i}$ in
\eqref{eq:ng-message-eta} is the exact BP log-message. If a structured-VMP
factorization constraint is also imposed on $a$, the same gradient rule applies
with $\ell_{a\to i}$ replaced by the tilted-factor log-message induced by
\eqref{eq:vmp-svmp-update}. In that mean-field setting one projection step is
the non-conjugate VMP update of \citet{knowles_nonconjugate_2011}; see
\Cref{sec:related-work}.

\paragraph{The message depends on the marginal it is sent to.}
A striking feature of \eqref{eq:ng-message-eta} is now immediate:
$\veta_{a\to i}$, and hence $\widehat{\mu}_{a\to i}$, depends on the receiving
edge's own marginal $q_{\lambda_i}$. This is the formal version of the NGMP
information flow previewed in \Cref{fig:compare-inference}: the factor still
forms the exact BP log-message from its local factor and incoming messages, but
the finite message sent onward is obtained by projecting that log-message at the
current receiving marginal. In BP and structured VMP, the outgoing message is
computed from the factor and its other incoming messages; in EP, the outgoing
site update is built from a cavity/tilted construction. NGMP is different
because the receiving marginal itself is the projection point. Operationally,
the update is therefore a Fisher-metric fixed-point iteration that repeatedly
re-projects the exact BP log-message at the current edge marginal.

The next section gives the same replacement a surrogate-model interpretation:
the projected message is implemented by replacing a non-representable local
contribution with a conjugate surrogate whose natural parameter is the tangent
projection.

\section{Surrogate Models}
\label{sec:surrogate-models}

\Cref{thm:ngmp-edge} gives a projected message, not a new global posterior
family. This section explains how such messages can be implemented by replacing
non-representable local contributions with conjugate surrogate leaves. The
first part treats non-conjugate unary leaves. This is essentially conjugate-computation variational inference in FFG language: replacing a difficult local likelihood with a conjugate pseudo-observation creates an auxiliary graph on which ordinary BP can be run. \Cref{fig:poisson-normal-surrogate} shows the complete tree both as a Gaussian surrogate model and as NGMP on the
original graph. Later in \cref{sec:surrogate-multi-interface}, We treat factors with multiple constrained interfaces, such as $\mathcal N(y\mid x,\tau^{-1})$. This is where the local
FFG view matters more. Rather than inventing one global Normal--Gamma
variational family for the joint $(x,\tau)$ block, each interface receives its
own projected conjugate message. \Cref{fig:iid-mean-precision-updates} then
shows how replication turns these per-interface updates into a loop.

\subsection{Surrogate Leaves}
\label{sec:surrogate-leaves}

Consider the Gaussian state-space model with non-conjugate Poisson
observations used by \citet[App.~E.2]{khan_conjugatecomputation_2017}. Writing
$\mathcal P(y\mid r)$ for a Poisson likelihood with rate $r$,
\begin{equation}
    p(z_{0:K},y_{1:K})
    \;=\;
    \mathcal N(z_0\mid m_0,v_0)
    \prod_{k=1}^{K}\mathcal N(z_k\mid z_{k-1},\sigma^2)
    \prod_{k=1}^{K}\mathcal P(y_k\mid \exp z_k).
    \label{eq:interp-poisson-chain}
\end{equation}

Its FFG is a tree: the Gaussian transition factors form a chain, and each
Poisson likelihood is a unary leaf attached to one state edge.
The transition factors are conjugate to Gaussian messages, but the observation
leaf is not. The exact BP message from the observation leaf to the edge $z_k$
is the likelihood itself,
\begin{equation}
    \begin{aligned}
        \mu_{y_k\to k}(z_k)
        &\;\propto\;
        \mathcal P(y_k\mid \exp z_k),\\
        \ell_k(z_k)
        &\;:=\;
        \log\mu_{y_k\to k}(z_k)
        \;=\;
        y_k z_k-\exp z_k+\mathrm{const}.
    \end{aligned}
    \label{eq:interp-poisson-log-message}
\end{equation}
This exact log-message is not quadratic in $z_k$, so it cannot be passed
unchanged inside a Gaussian message-passing scheme. The surrogate-leaf
implementation is to keep the exact log-message as the local object being
projected, but communicate only its Gaussian tangent component at the receiving
edge marginal. \Cref{fig:poisson-normal-surrogate} places this unary
replacement back into the full model for two observation slices. The right
panel keeps the original Poisson tree and labels the projected NGMP messages;
the left panel represents the same messages as Gaussian pseudo-observation
factors. Because every Poisson leaf touches only one latent edge, both views
remain trees. The iteration is required because each projected site must be recomputed at its updated Gaussian marginal, not because of a cycle in the FFG.

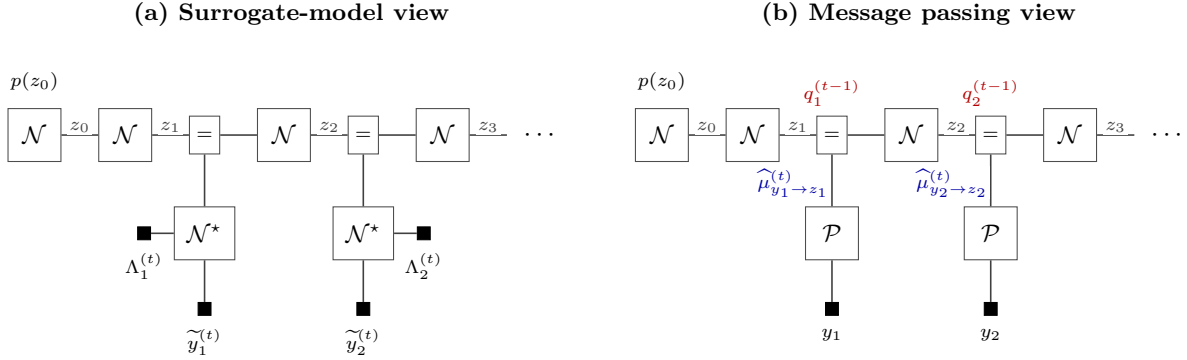
\begin{figure}[ht]
    \centering

\begin{tikzpicture}[
    x=1mm,
    y=1mm,
    every node/.style={font=\small},
    paneltitle/.style={font=\bfseries\small, align=center},
    fact/.style={draw=black!70, minimum width=7mm, minimum height=7mm,
                 fill=white, inner sep=1pt},
    surrogate/.style={fact, minimum width=8mm},
    eqfact/.style={draw=black!70, minimum width=4mm, minimum height=5mm,
                   inner sep=0pt, fill=white, font=\scriptsize},
    obsdot/.style={rectangle, fill=black, draw=black, inner sep=0pt,
                   minimum width=1.7mm, minimum height=1.7mm},
    edge/.style={semithick, black!75},
    varlabel/.style={font=\scriptsize, fill=white, inner sep=0.7pt},
    nglabel/.style={font=\scriptsize, text=blue!65!black,
                    fill=white, inner sep=0.7pt},
    qlabel/.style={font=\scriptsize, text=red!70!black,
                   fill=white, inner sep=0.7pt}
]

\begin{scope}[shift={(0,0)}, local bounding box=PSURR]
    \node[paneltitle] at (34,22) {(a) Surrogate-model view};

    \node[fact] (sprior) at (0,6) {$\mathcal N$};
    \node[above=1mm of sprior,font=\scriptsize] {$p(z_0)$};
    \node[fact] (str1) at (12,6) {$\mathcal N$};
    \node[eqfact] (seq1) at (22.5,6) {$=$};
    \node[fact] (str2) at (33,6) {$\mathcal N$};
    \node[eqfact] (seq2) at (43.5,6) {$=$};
    \node[fact] (str3) at (54,6) {$\mathcal N$};
    \coordinate (scontinue) at (62.5,6);
    \node[right=1mm of scontinue,font=\normalsize] {$\cdots$};

    \draw[edge] (sprior) -- node[above,varlabel] {$z_0$} (str1);
    \draw[edge] (str1) -- node[above,varlabel] {$z_1$} (seq1);
    \draw[edge] (seq1) -- (str2);
    \draw[edge] (str2) -- node[above,varlabel] {$z_2$} (seq2);
    \draw[edge] (seq2) -- (str3);
    \draw[edge] (str3) -- node[above,varlabel] {$z_3$} (scontinue);

    \node[surrogate] (sn1) at (22.5,-7) {$\mathcal N^{\star}$};
    \node[surrogate] (sn2) at (43.5,-7) {$\mathcal N^{\star}$};
    \draw[edge] (seq1) -- (sn1);
    \draw[edge] (seq2) -- (sn2);

    \node[obsdot] (sy1) at (22.5,-17) {};
    \node[below=0.5mm of sy1,font=\scriptsize] {$\widetilde y_1^{(t)}$};
    \draw[edge] (sn1) -- (sy1);
    \node[obsdot] (sl1) at (14.5,-7) {};
    \node[below=0.5mm of sl1,font=\scriptsize] {$\Lambda_1^{(t)}$};
    \draw[edge] (sl1) -- (sn1);

    \node[obsdot] (sy2) at (43.5,-17) {};
    \node[below=0.5mm of sy2,font=\scriptsize] {$\widetilde y_2^{(t)}$};
    \draw[edge] (sn2) -- (sy2);
    \node[obsdot] (sl2) at (51.5,-7) {};
    \node[below=0.5mm of sl2,font=\scriptsize] {$\Lambda_2^{(t)}$};
    \draw[edge] (sn2) -- (sl2);
\end{scope}

\begin{scope}[shift={(83,0)}, local bounding box=PNGMP]
    \node[paneltitle] at (34,22) {(b) Message passing view};

    \node[fact] (nprior) at (0,6) {$\mathcal N$};
    \node[above=1mm of nprior,font=\scriptsize] {$p(z_0)$};
    \node[fact] (ntr1) at (12,6) {$\mathcal N$};
    \node[eqfact] (neq1) at (22.5,6) {$=$};
    \node[fact] (ntr2) at (33,6) {$\mathcal N$};
    \node[eqfact] (neq2) at (43.5,6) {$=$};
    \node[fact] (ntr3) at (54,6) {$\mathcal N$};
    \coordinate (ncontinue) at (62.5,6);
    \node[right=1mm of ncontinue,font=\normalsize] {$\cdots$};

    \draw[edge] (nprior) -- node[above,varlabel] {$z_0$} (ntr1);
    \draw[edge] (ntr1) -- node[above,varlabel] {$z_1$} (neq1);
    \draw[edge] (neq1) -- (ntr2);
    \draw[edge] (ntr2) -- node[above,varlabel] {$z_2$} (neq2);
    \draw[edge] (neq2) -- (ntr3);
    \draw[edge] (ntr3) -- node[above,varlabel] {$z_3$} (ncontinue);

    \node[fact] (np1) at (22.5,-7) {$\mathcal P$};
    \node[fact] (np2) at (43.5,-7) {$\mathcal P$};
    \draw[edge] (neq1) --
        node[pos=0.55,left,nglabel] {$\widehat\mu_{y_1\to z_1}^{(t)}$} (np1);
    \draw[edge] (neq2) --
        node[pos=0.55,left,nglabel] {$\widehat\mu_{y_2\to z_2}^{(t)}$} (np2);

    \node[above=1mm of neq1,qlabel] {$q_1^{(t-1)}$};
    \node[above=1mm of neq2,qlabel] {$q_2^{(t-1)}$};

    \node[obsdot] (ny1) at (22.5,-17) {};
    \node[below=0.5mm of ny1,font=\scriptsize] {$y_1$};
    \draw[edge] (np1) -- (ny1);
    \node[obsdot] (ny2) at (43.5,-17) {};
    \node[below=0.5mm of ny2,font=\scriptsize] {$y_2$};
    \draw[edge] (np2) -- (ny2);
\end{scope}

\end{tikzpicture}
    \caption{Two equivalent implementations of the projected updates on two
    displayed observation slices of \eqref{eq:interp-poisson-chain}; both
    factor graphs are trees and the dots indicate continuation to later
    states.
    \textbf{(a)} Surrogate-model view. Each Poisson observation leaf is
    represented by a Gaussian pseudo-observation
    $\mathcal N^{\star}=\mathcal N(\widetilde y_k^{(t)}\mid
    z_k,(\Lambda_k^{(t)})^{-1})$, so ordinary Gaussian BP runs on the entire
    auxiliary tree. \textbf{(b)} NGMP message view. The original Poisson
    factors remain visible; each blue label is the projected observation
    message and the adjacent red marginal defines its tangent space. Inserting
    a Gaussian factor with the blue message's natural parameters produces the
    corresponding leaf in panel (a). Applying this conversion independently
    at every $k$ gives \eqref{eq:interp-normal-normal-surrogate}. The same idea
    generalizes per interface: a multi-interface factor sends a separate
    projected message through each constrained interface, as shown in
    \Cref{fig:iid-mean-precision-updates}.}
    \label{fig:poisson-normal-surrogate}
\end{figure}

Constrain the edge marginal to the univariate Gaussian family
$q_k(z_k)=\mathcal N(z_k\mid m_k,v_k)$ with sufficient statistics
$T_k(z_k)=(z_k,z_k^2)^{\top}$. It is convenient to use mean coordinates
$(m_k,s_k)$, where $s_k=\mathbb E_{q_k}[z_k^2]=m_k^2+v_k$. Under $q_k$,
\begin{equation}
    \mathbb E_{q_k}[\ell_k]
    \;=\;
    y_km_k-\rho_k+\mathrm{const},
    \qquad
    \rho_k
    \;:=\;
    \mathbb E_{q_k}[\exp z_k]
    \;=\;
    \exp\!\left(m_k+\frac{v_k}{2}\right).
    \label{eq:interp-expected-poisson-logmsg}
\end{equation}
Since $v_k=s_k-m_k^2$, differentiating
\eqref{eq:interp-expected-poisson-logmsg} with respect to the Gaussian mean
parameters gives the projected natural parameter
\begin{equation}
    \eta_{y_k\to k}^{\star}
    \;=\;
    \nabla_{(m_k,s_k)}
    \mathbb E_{q_k}[\ell_k]
    \;=\;
    \begin{bmatrix}
        y_k+(m_k-1)\rho_k\\[2pt]
        -\rho_k/2
    \end{bmatrix}.
    \label{eq:interp-poisson-projection}
\end{equation}
This is \eqref{eq:ng-message-eta} in the unary case: the non-quadratic
log-message is replaced by the Gaussian message whose natural parameter is the
mean-coordinate gradient of its expectation.

Writing the projected Gaussian message in canonical form,
\begin{equation}
    \widehat{\mu}_{y_k\to k}^{\star}(z_k)
    \;\propto\;
    \exp\!\left\{\xi_k^{\star}z_k-\frac{1}{2}\tau_k^{\star}z_k^2\right\},
    \qquad
    \xi_k^{\star}=y_k+(m_k-1)\rho_k,
    \quad
    \tau_k^{\star}=\rho_k,
    \label{eq:interp-gaussian-message-canonical}
\end{equation}
and completing the square yields the pseudo-observation form
\begin{equation}
    \widehat{\mu}_{y_k\to k}^{\star}(z_k)
    \;\propto\;
    \mathcal N(\widetilde y_k^{\star}\mid z_k,(\Lambda_k^{\star})^{-1}),
    \qquad
    \Lambda_k^{\star}=\rho_k,
    \qquad
    \widetilde y_k^{\star}
    =
    m_k+\frac{y_k-\rho_k}{\rho_k}.
    \label{eq:interp-pseudo-observation}
\end{equation}
After every Poisson leaf has been rewritten this way, the auxiliary model at an
outer iteration $t$ is a normal-normal chain,
\begin{equation}
    \widetilde p^{(t)}(z_{0:K}\mid \widetilde y_{1:K}^{(t)})
    \;\propto\;
    \mathcal N(z_0\mid m_0,v_0)
    \prod_{k=1}^{K}\mathcal N(z_k\mid z_{k-1},\sigma^2)
    \prod_{k=1}^{K}
    \mathcal N(\widetilde y_k^{(t)}\mid z_k,(\Lambda_k^{(t)})^{-1}).
    \label{eq:interp-normal-normal-surrogate}
\end{equation}
On this surrogate graph, ordinary Gaussian BP is sufficient. The outer loop
carries the non-conjugacy: run Gaussian BP, read the updated Gaussian edge
marginals, recompute the projected Poisson leaves, and repeat. This is the FFG
version of conjugate-computation variational inference. The old ingredient is
the conversion of a non-conjugate local term into a conjugate surrogate site;
the NGMP theorem identifies that site's natural parameter as the projected BP
log-message at the receiving marginal. Since the Poisson factor has only one
latent interface, the two views in \Cref{fig:poisson-normal-surrogate} are
particularly simple. We next retain this local conversion but let one factor
touch two shared latent interfaces.

\subsection{From Leaves to Multi-Interface Factors}
\label{sec:surrogate-multi-interface}

The leaf case hides the main reason for the edge-local formulation. Consider
the simplest repeated factor with two latent interfaces: $N$ conditionally
i.i.d. Normal observations sharing one unknown mean and one unknown precision,
\begin{equation}
    \begin{aligned}
        x &\sim \mathcal N(m_0,v_0),
        \qquad \tau \sim \mathcal G(a_0,b_0),\\
        y_n\mid x,\tau &\sim \mathcal N(x,\tau^{-1}),
        \qquad n=1,\ldots,N.
    \end{aligned}
    \label{eq:surrogate-normal-precision-iid}
\end{equation}
The observation factor
$f_n(x,\tau)=\mathcal N(y_n\mid x,\tau^{-1})$ touches both the Gaussian mean
interface $x$ and the Gamma precision interface $\tau$. As discussed in
\Cref{sec:variational-message-passing}, exact BP through this factor does not
preserve both families simultaneously. A single joint Normal--Gamma surrogate
family for all local quantities would be awkward, and it is unnecessary for
message passing. The FFG view only requires each constrained interface to
receive a finite message in its own family.

Despite having only two unknowns, the FFG is already loopy with two
observations. In \Cref{fig:iid-mean-precision-updates}, the two Normal
likelihoods give two paths between the equality constraints for $x$ and
$\tau$, closing the diamond cycle.

\begin{figure}
    \centering

\begin{tikzpicture}[
    x=1mm,
    y=1mm,
    >=stealth',
    every node/.style={font=\small},
    paneltitle/.style={font=\bfseries\small, align=center},
    prior/.style={draw=black!70, minimum width=8mm, minimum height=7mm,
                  fill=white, inner sep=1pt},
    fact/.style={draw=black!70, minimum width=9mm, minimum height=8mm,
                 fill=white, inner sep=1pt},
    eqfact/.style={draw=black!70, minimum width=5mm, minimum height=5mm,
                   inner sep=0pt, fill=white, font=\scriptsize},
    obsdot/.style={rectangle, fill=black, draw=black, inner sep=0pt,
                   minimum width=1.7mm, minimum height=1.7mm},
    edge/.style={semithick, black!75},
    tauedge/.style={semithick, dashed, black!75},
    edgelabel/.style={font=\scriptsize, fill=white, inner sep=0.7pt},
    nglabel/.style={font=\scriptsize, text=blue!65!black,
                    fill=white, inner sep=0.7pt},
    qlabel/.style={font=\scriptsize, text=red!70!black,
                   fill=white, inner sep=0.7pt}
]

\newcommand{\twonormalnodes}[1]{%
    \node[prior] (#1xprior) at (10,21) {$\mathcal N$};
    \node[above=1.0mm of #1xprior, font=\scriptsize] {$p(x)$};
    \node[prior] (#1tprior) at (60,21) {$\mathcal G$};
    \node[above=1.0mm of #1tprior, font=\scriptsize] {$p(\tau)$};

    \node[eqfact] (#1xeq) at (10,0) {$=$};
    \node[eqfact] (#1teq) at (60,0) {$=$};
    \draw[edge] (#1xprior) -- node[left,font=\scriptsize] {$x$} (#1xeq);
    \draw[tauedge] (#1tprior) -- node[right,font=\scriptsize] {$\tau$} (#1teq);

    \node[fact] (#1n1) at (35,10) {$\mathcal N$};
    \node[fact] (#1n2) at (35,-12) {$\mathcal N$};

    \node[obsdot, above=3mm of #1n1] (#1y1) {};
    \draw[edge] (#1n1) -- (#1y1);
    \node[above=0.5mm of #1y1, font=\scriptsize] {$y_1$};
    \node[obsdot, below=3mm of #1n2] (#1y2) {};
    \draw[edge] (#1n2) -- (#1y2);
    \node[below=0.5mm of #1y2, font=\scriptsize] {$y_2$};
}

\begin{scope}[shift={(0,0)}, local bounding box=PVMP]
    \node[paneltitle] at (35,35) {(a) Variational message passing};
    \twonormalnodes{v}

    \draw[edge] (vxeq) --
        node[pos=0.50,above,sloped,edgelabel] {$q^{(t-1)}(x)$} (vn1);
    \draw[edge] (vxeq) -- (vn2);
    \draw[tauedge] (vn1) --
        node[pos=0.52,above,sloped,edgelabel]
        {$\widetilde\mu_{y_1\to\tau}^{(t)}$} (vteq);
    \draw[tauedge] (vteq) -- (vn2);
\end{scope}

\begin{scope}[shift={(82,0)}, local bounding box=PNGMP]
    \node[paneltitle] at (35,35) {(b) Natural-gradient message passing};
    \twonormalnodes{n}

    \draw[edge] (nxeq) --
        node[pos=0.50,above,sloped,nglabel] {$\mu_{x\setminus1}^{(t)}$} (nn1);
    \draw[tauedge] (nn1) --
        node[pos=0.40,above,sloped,nglabel]
        {$\widehat\mu_{y_1\to\tau}^{(t)}$}
        node[pos=0.65,below,sloped,qlabel] {$q^{(t-1)}(\tau)$} (nteq);
    \draw[tauedge] (nteq) --
        node[pos=0.52,below,sloped,nglabel] {$\mu_{\tau\setminus2}^{(t)}$} (nn2);
    \draw[edge] (nn2) --
        node[pos=0.40,below,sloped,nglabel]
        {$\widehat\mu_{y_2\to x}^{(t)}$}
        node[pos=0.65,above,sloped,qlabel] {$q^{(t-1)}(x)$} (nxeq);
\end{scope}

\end{tikzpicture}
    \caption{The two-observation instance of
    \eqref{eq:surrogate-normal-precision-iid} and the message from the first
    observation toward $\tau$. The Normal prior for $x$ and Gamma prior for
    $\tau$ are at the top; the two $\mathcal N(y_i\mid x,\tau^{-1})$ factors
    close the diamond loop. Solid edges carry $x$ and dashed edges carry
    $\tau$. \textbf{(a)} VMP uses the previous global marginal
    $q^{(t-1)}(x)\propto p(x)\prod_{n=1}^{N}
    \widetilde\mu_{y_n\to x}^{(t-1)}(x)$ to compute
    $\widetilde\mu_{y_1\to\tau}^{(t)}$. \textbf{(b)} NGMP instead uses the
    cavity $\mu_{x\setminus1}^{(t)}$ and projects the resulting exact BP
    log-message at $q_\tau^{(t-1)}$. The blue labels enumerate the four
    messages around the loop; their subscripts give the directions. The red
    labels are the receiving marginals at which the two blue projected
    messages are evaluated. The new $y_1\!\to\!\tau$ message changes the
    cavity entering the lower Normal, whose message toward $x$ changes the
    cavity returning to the upper Normal on the next sweep. The message toward
    $x$ is analogous.}
    \label{fig:iid-mean-precision-updates}
\end{figure}


Thus, VMP iterates between two global marginals, whereas NGMP keeps the
cavity-specific information required by the joint-belief Bethe fixed point.
With $N$ observations, the same pattern couples $2N$ projected likelihood
messages, so even exact local projections require loopy message passing on
this replicated FFG.

\subsection{VMP as the Classical Multi-Interface Surrogate}
\label{sec:surrogate-vmp}

Classical VMP restores finite messages by replacing the coupled observation
factor with tilted factors obtained by averaging the log-likelihood over the
other interface. For
$f_n(x,\tau)=\mathcal N(y_n\mid x,\tau^{-1})$,
\begin{equation}
    \widetilde f_n^{\,x}(x)
    =
    \exp\!\left(
    \mathbb E_{q(\tau)}[\log f_n(x,\tau)]
    \right),
    \qquad
    \widetilde f_n^{\,\tau}(\tau)
    =
    \exp\!\left(
    \mathbb E_{q(x)}[\log f_n(x,\tau)]
    \right).
    \label{eq:surrogate-vmp-tilted-factors}
\end{equation}
Averaging the local log-factor over the other interface gives the conjugate
forms:
\begin{equation}
    \log \widetilde f_n^{\,x}(x)
    =
    -\frac12\mathbb E[\tau](y_n-x)^2+\mathrm{const},
    \qquad
    \log \widetilde f_n^{\,\tau}(\tau)
    =
    \frac12\log\tau
    -\frac12\tau\mathbb E[(y_n-x)^2]
    +\mathrm{const}.
    \label{eq:surrogate-vmp-normal-gamma}
\end{equation}
Thus, each observation contributes a Gaussian site on $x$ and a Gamma site on
$\tau$. On the precision interface, the $\frac12\log\tau$ term increases the
Gamma shape by $\frac12$ per observation. For an incoming Gaussian cavity with
variance $V$, however, the exact BP log-message contains the determinant
correction $-\frac12\log(1+V\tau)$, which VMP omits. This omission produces the
familiar precision overconfidence illustrated in \cite[Figure 10.4]{bishop_pattern_2006};
\Cref{sec:edge-uncertainty} measures its size and when it matters.

\subsection{NGMP as the Improved Local Surrogate}
\label{sec:surrogate-ngmp}

NGMP keeps the same surrogate-graph construction as VMP--one conjugate
surrogate leaf per constrained interface--but changes how the leaf parameters
are chosen.
For the
message from $f_n$ to $x$, form the exact BP log-message by integrating out
the other interface against the incoming message it sends to the factor:
\begin{equation}
    \ell_{n\to x}(x)
    =
    \log
    \int
    \mathcal N(y_n\mid x,\tau^{-1})
    \,\mu_{\tau\setminus n}(\tau)\,d\tau.
    \label{eq:surrogate-ngmp-x-logmsg}
\end{equation}
For the message from $f_n$ to $\tau$,
\begin{equation}
    \ell_{n\to \tau}(\tau)
    =
    \log
    \int
    \mathcal N(y_n\mid x,\tau^{-1})
    \,\mu_{x\setminus n}(x)\,dx.
    \label{eq:surrogate-ngmp-tau-logmsg}
\end{equation}
Each exact log-message is then projected at its own receiving marginal,
\begin{equation}
    \veta_{n\to i}
    =
    \nabla_{\mu_i}
    \mathbb E_{q_i}\!\left[\ell_{n\to i}\right],
    \qquad
    i\in\{x,\tau\}.
    \label{eq:surrogate-ngmp-interface-projection}
\end{equation}
Then, with $\widehat{\mu}_{n\to i}\propto
    \exp\{\veta_{n\to i}^{\top}T_i\}$ (remind \cref{eq:ng-message-function}), the $x$ projection is a Gaussian leaf
\begin{equation}
    \widehat{\mu}_{n\to x}(x)
    \;\propto\;
    \exp\!\left\{\xi_n x-\frac12\Lambda_n x^2\right\},
    \label{eq:surrogate-ngmp-gaussian-leaf}
\end{equation}
and the $\tau$ projection is a Gamma leaf
\begin{equation}
    \widehat{\mu}_{n\to\tau}(\tau)
    \;\propto\;
    \exp\!\left\{\Delta a_n\log\tau-\Delta b_n\tau\right\}.
    \label{eq:surrogate-ngmp-gamma-leaf}
\end{equation}
The blue labels in \Cref{fig:iid-mean-precision-updates}b show this dependency
for two observations. The projected $y_1\!\to\!\tau$ site changes the
$\tau$-cavity entering the second Normal factor; its projected message toward
$x$ then changes the $x$-cavity entering the first factor. The receiving
marginals shown in red set the tangent spaces of the two projections. The
message therefore depends indirectly on its own previous value through the
loopy cavity update, rather than only through a global marginal average.

The point of the local-surrogate construction is that this NGMP computation
also has a surrogate-model interpretation, exactly as the unary Poisson update
has the two equivalent views in \Cref{fig:poisson-normal-surrogate}. Once the
two projected messages in \eqref{eq:surrogate-ngmp-gaussian-leaf} and
\eqref{eq:surrogate-ngmp-gamma-leaf} have been computed, freeze their natural
parameters and insert one conjugate leaf on each receiving interface.
\Cref{fig:multi-interface-surrogate} shows this conversion for one observation
factor. The original factor remains the object used to compute the two exact
cavity log-messages, but ordinary BP is run on the auxiliary graph in which its
two projected contributions are represented separately. What distinguishes
NGMP from VMP is therefore not the final Gaussian and Gamma leaf families, but
their source: NGMP projects the exact cavity log-messages at the receiving
marginals, whereas VMP uses the mean-field tilted factors in
\eqref{eq:surrogate-vmp-tilted-factors}.

\begin{figure}[H]
    \centering

\begin{tikzpicture}[
    x=1mm,
    y=1mm,
    every node/.style={font=\small},
    paneltitle/.style={font=\bfseries\small, align=center},
    fact/.style={draw=black!70, minimum width=8mm, minimum height=7mm,
                 fill=white, inner sep=1pt},
    widefact/.style={fact, minimum width=22mm},
    eqfact/.style={draw=black!70, minimum width=5mm, minimum height=5mm,
                   inner sep=0pt, fill=white, font=\scriptsize},
    obsdot/.style={rectangle, fill=black, draw=black, inner sep=0pt,
                   minimum width=1.7mm, minimum height=1.7mm},
    edge/.style={semithick, black!75},
    edgelabel/.style={font=\scriptsize, fill=white, inner sep=0.7pt}
]

\begin{scope}[shift={(0,0)}]
    \node[paneltitle] at (30,31) {(a) Original multi-interface factor};

    \node[fact] (xf) at (14,9) {$f$};
    \node[eqfact] (xeq) at (30,9) {$=$};
    \node[fact] (xg) at (46,9) {$g$};
    \draw[edge] (xf) -- node[above,edgelabel] {$x$} (xeq);
    \draw[edge] (xeq) -- (xg);

    \node[widefact] (normal) at (30,-2.5)
        {$\mathcal N(y\mid x,\tau^{-1})$};
    \draw[edge] (xeq) -- (normal);
    \node[obsdot] (ydot) at (45,-2.5) {};
    \node[right=1.5pt of ydot,font=\scriptsize] {$y$};
    \draw[edge] (normal) -- (ydot);

    \node[fact] (th) at (14,-14) {$h$};
    \node[eqfact] (teq) at (30,-14) {$=$};
    \node[fact] (tk) at (46,-14) {$k$};
    \draw[edge] (th) -- node[above,edgelabel] {$\tau$} (teq);
    \draw[edge] (teq) -- (tk);
    \draw[edge] (normal) -- (teq);
\end{scope}

\begin{scope}[shift={(88,0)}]
    \node[paneltitle] at (30,31) {(b) Per-interface surrogate leaves};

    \node[fact] (sxf) at (14,6) {$f$};
    \node[eqfact] (sxeq) at (30,6) {$=$};
    \node[fact] (sxg) at (46,6) {$g$};
    \draw[edge] (sxf) -- node[above,edgelabel] {$x$} (sxeq);
    \draw[edge] (sxeq) -- (sxg);

    \node[fact] (xsite) at (30,18) {$\mathcal N^{\star}$};
    \draw[edge] (xsite) -- (sxeq);
    \node[obsdot] (sytilde) at (30,26) {};
    \node[right=1.5pt of sytilde,font=\scriptsize] {$\widetilde y^{\star}$};
    \draw[edge] (sytilde) -- (xsite);
    \node[obsdot] (slambda) at (40,18) {};
    \node[right=1.5pt of slambda,font=\scriptsize] {$\Lambda^{\star}$};
    \draw[edge] (xsite) -- (slambda);

    \node[fact] (sth) at (14,-7) {$h$};
    \node[eqfact] (steq) at (30,-7) {$=$};
    \node[fact] (stk) at (46,-7) {$k$};
    \draw[edge] (sth) -- node[above,edgelabel] {$\tau$} (steq);
    \draw[edge] (steq) -- (stk);

    \node[fact] (tsite) at (30,-19) {$\mathcal G^{\star}$};
    \draw[edge] (steq) -- (tsite);
    \node[obsdot] (sa) at (40,-19) {};
    \node[right=1.5pt of sa,font=\scriptsize] {$a^{\star}$};
    \draw[edge] (tsite) -- (sa);
    \node[obsdot] (sb) at (30,-27) {};
    \node[right=1.5pt of sb,font=\scriptsize] {$b^{\star}$};
    \draw[edge] (tsite) -- (sb);
\end{scope}

\end{tikzpicture}
    \caption{Surrogate-model interpretation of the multi-interface NGMP
    computation in \Cref{fig:iid-mean-precision-updates}b.
    \textbf{(a)} The original observation factor
    $\mathcal N(y\mid x,\tau^{-1})$ touches the $x$ and $\tau$ equality
    constraints and produces one projected message for each interface.
    \textbf{(b)} With those message parameters frozen, the factor's
    contributions are represented by a Gaussian leaf $\mathcal N^{\star}$ on
    $x$ and a Gamma leaf $\mathcal G^{\star}$ on $\tau$; their observed ports
    denote the projected natural parameters. Neighboring factors $f,g,h,k$
    keep their ordinary BP rules. This is the direct multi-interface analogue
    of \Cref{fig:poisson-normal-surrogate}a, except that one original factor
    yields one surrogate leaf per constrained interface.}
    \label{fig:multi-interface-surrogate}
\end{figure}

This gives an outer fixed-point map. Let $\vlambda$ collect the natural
parameters of the constrained edge marginals, and let
$\widehat{\mathcal G}(\vlambda)$ denote the conjugate surrogate graph obtained
by projecting all required exact log-messages at those marginals. During the
inner BP sweep, this surrogate graph is frozen. Running ordinary BP on
$\widehat{\mathcal G}(\vlambda)$ returns new constrained-edge natural
parameters, which define
\begin{equation}
    \Phi(\vlambda)
    :=
    \text{natural parameters returned by BP on }
    \widehat{\mathcal G}(\vlambda).
    \label{eq:surrogate-fixed-point-map}
\end{equation}
For a single constrained degree-two edge, $\Phi_i(\vlambda)$ is the sum of the
two projected incoming message parameters in \eqref{eq:ng-main-theorem}; for a
larger graph, $\Phi$ stacks these local edge updates after the conjugate BP
sweep has propagated their consequences through the surrogate model. The NGMP
iteration is therefore
\begin{equation}
    \vlambda^{(t+1)}=\Phi(\vlambda^{(t)}),
    \label{eq:surrogate-undamped-fixed-point-iteration}
\end{equation}
possibly with damping or momentum (\Cref{app:momentum-damping} explains how these stabilize the iteration in practice). A fixed point
$\vlambda=\Phi(\vlambda)$ is exactly the surrogate-implementation form of the
projected stationarity equations in \Cref{thm:ngmp-edge}.

This is the new contribution relative to the classical surrogate story. NGMP
does not claim that the original Normal precision factor has become conjugate,
nor does it require a global Normal--Gamma joint approximation. It preserves the
edge-local constrained-Bethe stationarity of \Cref{thm:ngmp-edge}: each
constrained edge chooses its own family, receives its own projected message, and
then participates in ordinary BP on the current conjugate surrogate graph.

\section{Related Work}\label{sec:related-work}
\paragraph{Variational message passing and local constraints.}
VMP realizes coordinate-ascent variational inference as local expected-log
factor updates \citep{winn_variational_2005}; structured VMP retains selected
clusters, as in \eqref{eq:vmp-svmp-update}
\citep{dauwels_variational_2007,senoz_variational_2021}. Factor-graph
fragments and Extended VMP change how these local quantities are evaluated
but retain this information flow
\citep{wand_fast_2017,akbayrak_extended_2021}. When the model is conjugate,
the tilted message of \eqref{eq:vmp-svmp-update} already lies in the edge
family and no projection is needed; this is the setting in which we use the
plain name VMP.

\paragraph{Non-conjugate VMP.}
\citet{knowles_nonconjugate_2011} were the first to bring the Fisher
information of the receiving marginal into VMP. When the tilted message
$\widetilde\mu_{a\to i}=\exp\{\widetilde\ell_{a\to i}\}$ leaves the edge
family, non-conjugate VMP (NCVMP) sends the message whose natural parameter
is the Fisher-metric gradient of the expected log-factor at the current
marginal, $\fisher_i(\vlambda_i)^{-1}\nabla_{\lambda_i}\mathbb
E_{q_{\lambda_i}}[\widetilde\ell_{a\to i}]$. In the notation of
\eqref{eq:ng-projection}, this is $\tangentproj{i}[\widetilde\ell_{a\to i}]$:
one tangent projection of the tilted VMP log-message, applied once per
update. NCVMP is therefore a hybrid. Its information flow is that of VMP,
because out-of-cluster variables enter only through their marginals in
$\widetilde f_a$, and its projection point is the receiving marginal
$q_{\lambda_i}$, as in NGMP. Its use of natural gradients is limited to this
single step.

\paragraph{Conjugate-computation variational inference.}
\citet{khan_conjugatecomputation_2017} generalized this step to CVI, which
applies mirror descent to a global ELBO over a fixed-form, possibly
mean-field-factorized exponential family. The duality in
\eqref{eq:vi-ng-equals-mu-grad} and the stationary condition
\eqref{eq:vi-khan-stationary} turn the mean-gradient of an expected
non-conjugate term into an auxiliary conjugate natural parameter. Khan and Lin
call these additive contributions to a global or mean-field coordinate
``messages.'' Their derivation does not, however, start from an FFG local
polytope or derive a distinct outgoing message for every interface of a
multi-interface factor. Later work extends the same conjugate-computation
principle \citep{khan_fast_2018a,khan_bayesian_2023,lin2018variational}.
Khan and Lin's Gaussian pseudo-observation/Kalman construction also supplies
the computation used by the Poisson model in
\eqref{eq:interp-poisson-chain}; our claim concerns its factor-to-edge
semantics.

\paragraph{FFG-local CVI marginal optimization.}
Akbayrak et al. make CVI local to an FFG edge by collecting its scheduled
messages and optimizing the complete marginal
\citep{akbayrak_probabilistic_2022}. We call this schedule instantiated with
the CVI objective of \citet{khan_conjugatecomputation_2017} \emph{Projective
VMP} (PVMP). Given the VMP or structured-VMP messages on edge $i$, write their
normalized product as
$m_i(z_i):=\prod_{a\in\mathcal V(i)}m_{a\to i}(z_i)$. Then
\begin{equation}
    q_i^*
    =
    \arg\min_{q_i\in\mathcal Q_i}
    \mathbb D_{\text{KL}}\!\left[\,q_i\,\middle\|\, m_i\,\right].
    \label{eq:projective-vmp-marginal}
\end{equation}
The optimizer $q_i^*$ returns as the next VMP marginal without EP-like
division by an opposing message. PVMP runs the natural-gradient iteration on
the edge to convergence with the other marginals fixed, whereas NCVMP takes
one step of it and moves on; the two share stationary points on the edge but
differ in cost per update, which is the budget-matched comparison of
\Cref{app:edge-uncertainty-fe}. Both take VMP messages as input and add a
projection at the receiving marginal, so both are hybrids in the above sense.
Neither identifies one factor's natural-gradient contribution with its
outgoing message. Q-conjugacy makes this objective analytic for the
Poisson–Gaussian model in \eqref{eq:interp-poisson-chain}
\citep{lukashchuk_qconjugate_2024}.

\paragraph{NGMP recovers NCVMP and PVMP under mean-field constraints.}
NGMP accepts factorization constraints in addition to the edge form
constraint. If a mean-field or structured-VMP constraint is imposed on factor
$a$, the log-message $\ell_{a\to i}$ in \eqref{eq:ng-message-eta} is the
tilted log-message $\widetilde\ell_{a\to i}$ induced by
\eqref{eq:vmp-svmp-update}, and one projection step of NGMP is exactly the
NCVMP update of \citet{knowles_nonconjugate_2011}, with the same numerical
result. Iterating that projection on the edge marginal to convergence instead
of sending the one-step message gives PVMP. Without the factorization
constraint, NGMP projects the exact BP log-message, so the factor's other
neighbors enter through their messages rather than their marginals; this is
the generalization made in \Cref{thm:ngmp-edge}. We follow this
distinction in naming the baselines of \Cref{sec:edge-uncertainty} and
\Cref{sec:experiments}: VMP when the model is conjugate and no projection is
needed, NCVMP when a single projection step is taken, and PVMP when the
projection is iterated to convergence.

\paragraph{Expectation propagation and local divergence projections.}
EP removes a site to form a cavity, restores the factor to form a tilted
distribution, performs an inclusive-KL moment projection, and divides by the
cavity to recover the site \citep{minka_expectation_2001}. Minka's divergence
framework relates such local objectives more broadly
\citep{minka_divergence_2005}. Thus, EP projects a tilted marginal, NCVMP
takes one projection step of the tilted VMP message, PVMP optimizes a
complete edge marginal, and NGMP projects a factor-to-edge BP log-message.

\paragraph{Gaussian surrogate inference.}
Gaussian surrogate likelihoods are well established: variational Gaussian
sites yield a posterior precision equal to the prior precision plus local
site precisions \citep{10.1162/neco.2008.08-07-592}, consistent with Gaussian
Markov structure \citep{rue2005gaussian}. Such surrogates recover classical
filtering and smoothing computations
\citep{doi:10.1137/0804035,mangion2011online}. The pseudo-observation and
Kalman computation are therefore not contributions; the derived
factor-to-edge semantics is.

\paragraph{From global gradients to factor-to-edge messages.}
The formal distinction is already visible in the preceding equations. Khan's
condition \eqref{eq:vi-khan-stationary} uses the global log-model; BP instead
defines one functional message per factor interface in
\eqref{eq:ng-bp-logmsg}. \Cref{thm:ngmp-edge} proves that its projection
\eqref{eq:ng-message-eta} is the corresponding factor-to-edge natural
parameter. Hence, an arbitrary-degree factor yields a separate message at each
interface, with its own family and receiving marginal, without postulating a
global Gaussian, Normal--Gamma, or other joint family. This per-interface
statement is stronger than a global additive-gradient decomposition, a
one-step NCVMP message, or a complete-marginal PVMP update.


\section{Comparing VMP to NGMP: When Does NGMP Matter?}
\label{sec:edge-uncertainty}

VMP and NGMP address the same local approximation problem: both keep the graph
and chosen marginal families fixed, but they differ in what information enters
each update. VMP computes the update from the expected log-factor under the
current neighboring beliefs, whereas NGMP projects the exact cavity
log-message. The ablations below hold everything else fixed and ask when this
difference matters. We name the baselines as in \Cref{sec:related-work}. In
the first ablation, the model is conjugate, so the baseline is plain VMP with
closed-form updates. In the two non-conjugate ablations, the tilted VMP
messages leave the edge families, so the baseline is PVMP, which iterates the
projective update of \eqref{eq:projective-vmp-marginal} on every edge to
convergence \citep{akbayrak_probabilistic_2022}; the single-step variant of
this projection is NCVMP \citep{knowles_nonconjugate_2011} and serves as the
budget-matched control in \Cref{app:edge-uncertainty-fe}. The computational
complexity favors NGMP: it takes a single natural-gradient step per edge
update, whereas PVMP runs an inner optimization scheme on every edge until
convergence and may therefore evaluate up to a hundred gradients per update.

The relevant distinction is edge uncertainty and how often an approximate
update is composed. When the neighboring beliefs concentrate, variational
expectations become evaluations, and the competing updates agree. When
uncertainty persists, the local discrepancy can accumulate either while
smoothing information along a latent-state chain or while filtering shared
parameters through successive data batches. We isolate these three regimes
below: vanishing edge uncertainty, error accumulation along a Poisson
smoothing chain, and error accumulation along an online heteroscedastic
filtering chain. The broader comparison on real data is deferred to
\Cref{sec:experiments}.
Every result is averaged over 20 instances or masks, with scripts provided in
the \texttt{when\_ngmp\_helps} directory of the
\codelink{accompanying code repository}.

\paragraph{Vanishing edge uncertainty.}
We first return to the normal mean–precision model
of \eqref{eq:surrogate-normal-precision-iid}, with
$x\sim\mathcal N(0,25)$, $\tau\sim\mathrm{Gamma}(2,1)$, and $N$ observations.
For VMP, this model is conjugate, so its coordinate updates are in closed-form;
NGMP's tangent projection has no closed form here and is computed
numerically with an unscented approximation
(\Cref{app:tangent-projection-numerics}). Both constrained edges receive
every observation, so their uncertainty vanishes with $N$.
\Cref{fig:edge-uncertainty-normal-kl} compares VMP and NGMP
with numerically exact marginals, obtained by integrating out $x$ and applying
one-dimensional quadrature over $\tau$. For the precision marginal,
$\mathbb D_{\mathrm{KL}}[p_{\mathrm{exact}}\|q]$ decays approximately as
$N^{-2}$ for VMP and $N^{-3}$ for NGMP: the accuracy ratio grows from
$1.7\times$ at $N=4$ to ${\approx}700\times$ at $N=512$, where matching
NGMP's divergence would require VMP to observe ${\approx}15{,}000$
samples. The mean marginals are already nearly identical. The difference
is therefore an uncertainty correction: VMP omits the determinant term
(\Cref{sec:surrogate-vmp}) in the exact message toward $\tau$, whereas NGMP
retains its tangent component. In this isolated model, the difference has no
practical consequence: at $N=512$, the entire remaining VMP divergence amounts to a
$0.2\%$ deficit of the posterior precision variance, too small to affect
any prediction. The remaining ablations show when it starts to matter: first
when the error is repeated while smoothing through the unobserved stretches
of a chain, and then when state-dependent approximate batch updates are
composed during filtering.

\begin{figure}[t!]
    \centering
    \begin{minipage}{0.48\linewidth}
    \centering
    \includegraphics[width=\linewidth]{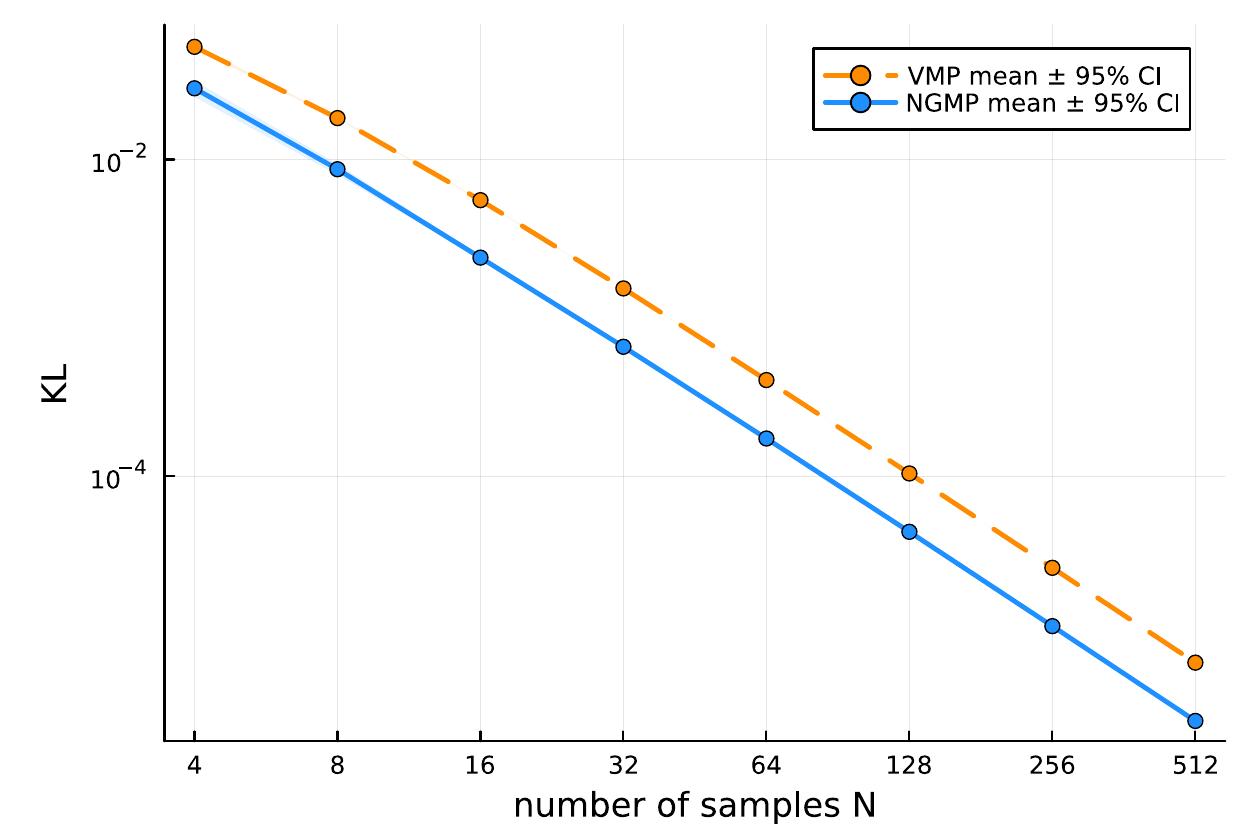}\\[-2pt]
    \small (a) Posterior over $x$
    \end{minipage}
    \hfill
    \begin{minipage}{0.48\linewidth}
    \centering
    \includegraphics[width=\linewidth]{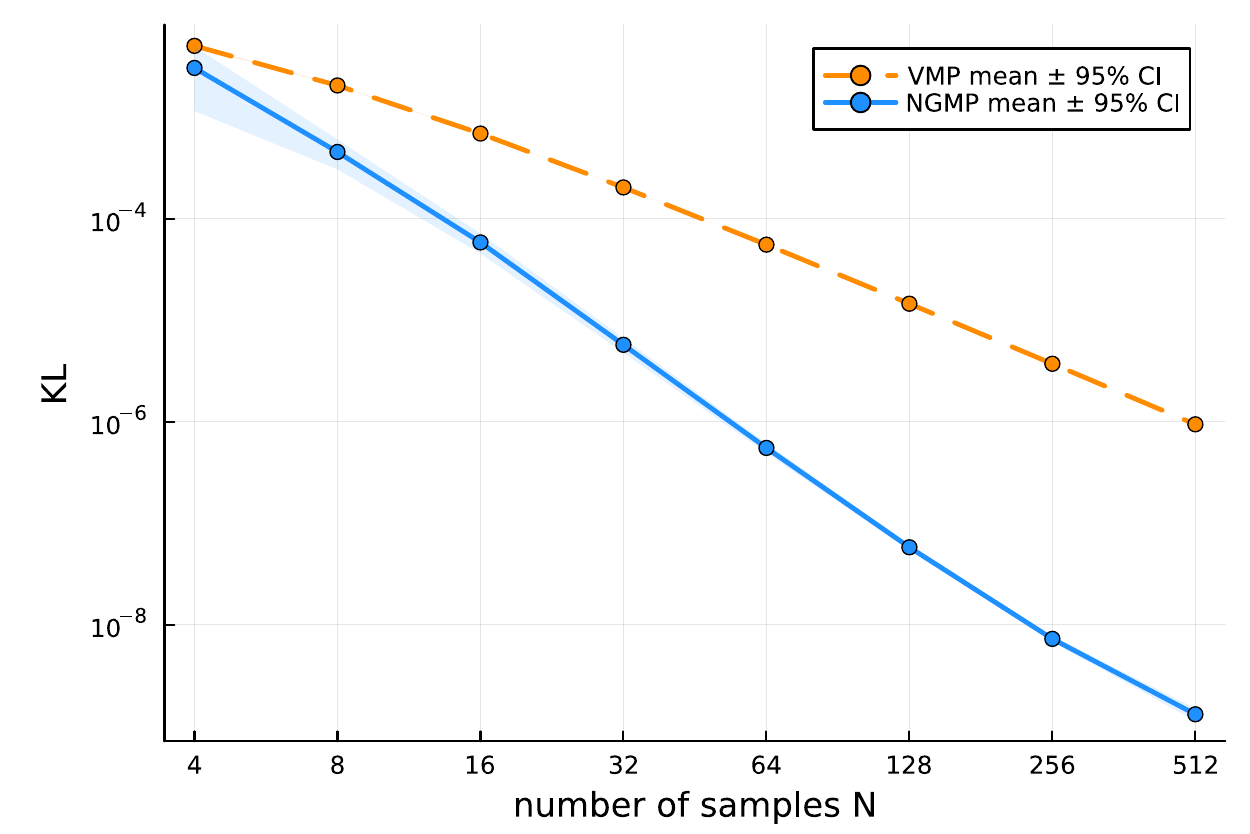}\\[-2pt]
    \small (b) Posterior over $\tau$
    \end{minipage}
    \caption{Marginal inclusive divergence in the isolated Normal
    mean--precision model: (a)
    $\mathbb D_{\mathrm{KL}}[p(x\mid y)\|q(x)]$ and (b)
    $\mathbb D_{\mathrm{KL}}[p(\tau\mid y)\|q(\tau)]$ (mean $\pm$ 95\% CI
    over 20 instances; each instance draws a fresh true $(x,\tau)$ from the
    priors and generates its own stream of samples). For $\tau$, VMP decays
    as $N^{-2}$ and NGMP as $N^{-3}$; for $x$, NGMP is a roughly constant
    $2.3\times$ closer.}
    \label{fig:edge-uncertainty-normal-kl}
\end{figure}

\paragraph{Error accumulation along a Poisson smoothing chain.}
Next, we use the Poisson state-space model of
\eqref{eq:interp-poisson-chain} on the monthly sunspot series \citep{sidc},
with $\sigma^2=0.1$. Unlike in the first ablation, every state edge of the
chain is non-conjugate, so the baseline is PVMP: it computes its marginals
with the projective update of \eqref{eq:projective-vmp-marginal}
(configuration details in \Cref{app:edge-uncertainty-fe}). Both inference
methods are nevertheless closed-form in this model: NGMP's projection is a
single analytic natural-gradient step per edge, and the gradients iterated by
PVMP's inner optimizer are likewise available analytically
(\Cref{app:tangent-projection-numerics}). At $5$--$20\%$ holdout, most missing months
neighbor observed ones and the methods tie within their confidence intervals
(\Cref{tab:edge-uncertainty-poisson}). At $50\%$, long unobserved stretches
keep the state edges uncertain and the methods separate.

\Cref{fig:edge-uncertainty-poisson} identifies the failure mode. PVMP repeats
the mean-field transition update through each missing month, pinning its
posterior variance near $\sigma^2/2=0.05$ independently of distance from
the data. NGMP instead propagates the cavity uncertainty, so the variance
grows from $0.09$ beside an observation to $0.27$ in the deepest gaps.
Accordingly, PVMP's held-out negative log-likelihood rises from $4.8$ at
distance one to $33.5$ at distances five to eight, whereas NGMP remains
near $5.7$. Both methods reach their Bethe plateaus within the 20-sweep budget
--- NGMP by sweep three to four --- so the gap is a property of the fixed
points, not of early stopping (traces in \Cref{fig:app-poisson-fe} of
\Cref{app:edge-uncertainty-fe}).

\begin{figure}[t!]
    \centering
    \includegraphics[width=0.49\linewidth]{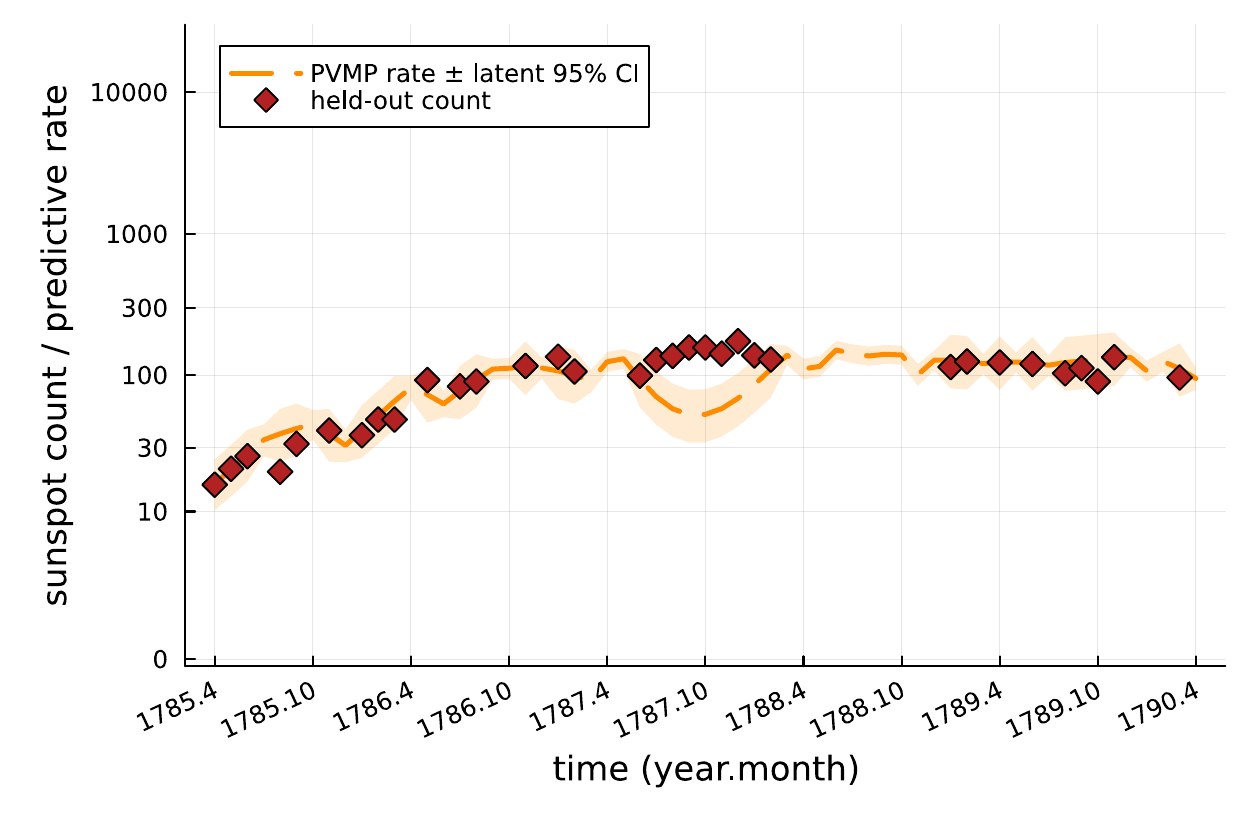}
    \hfill
    \includegraphics[width=0.49\linewidth]{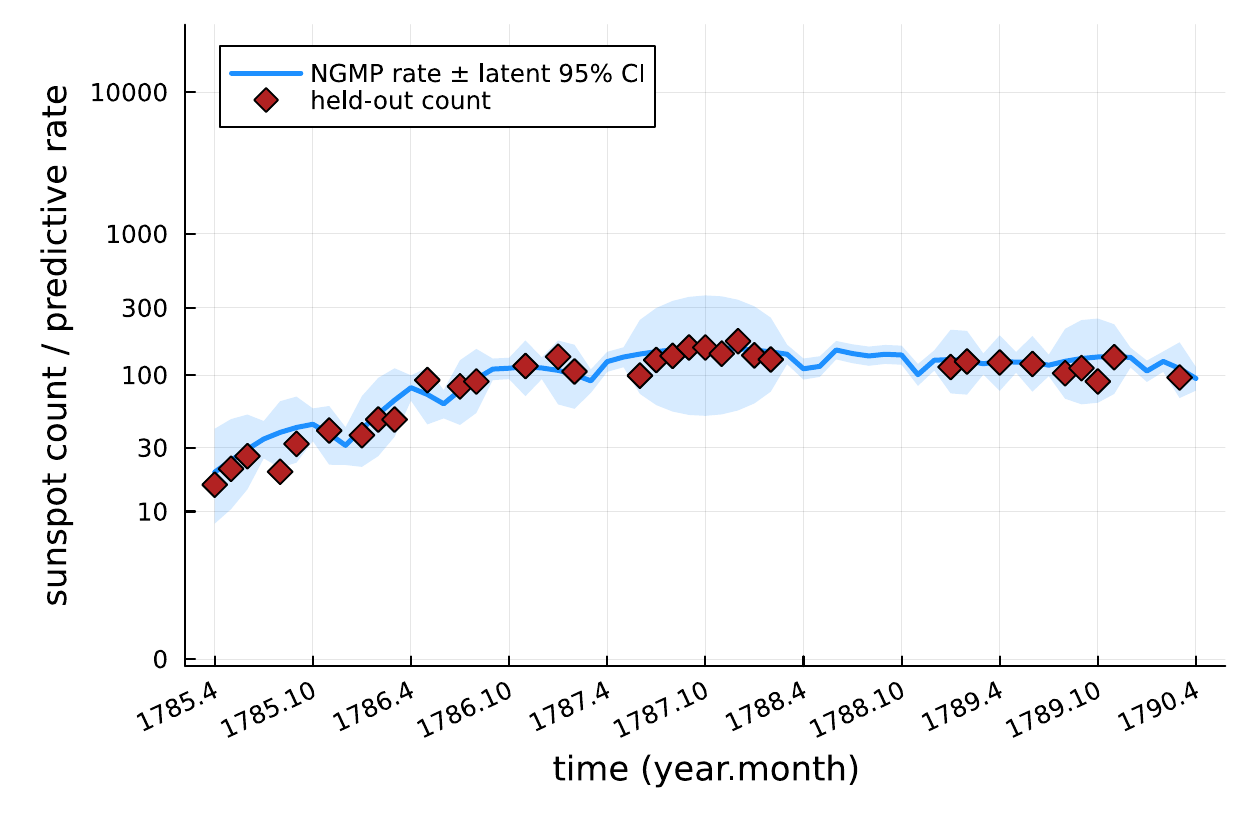}\\[4pt]
    \includegraphics[width=0.44\linewidth]{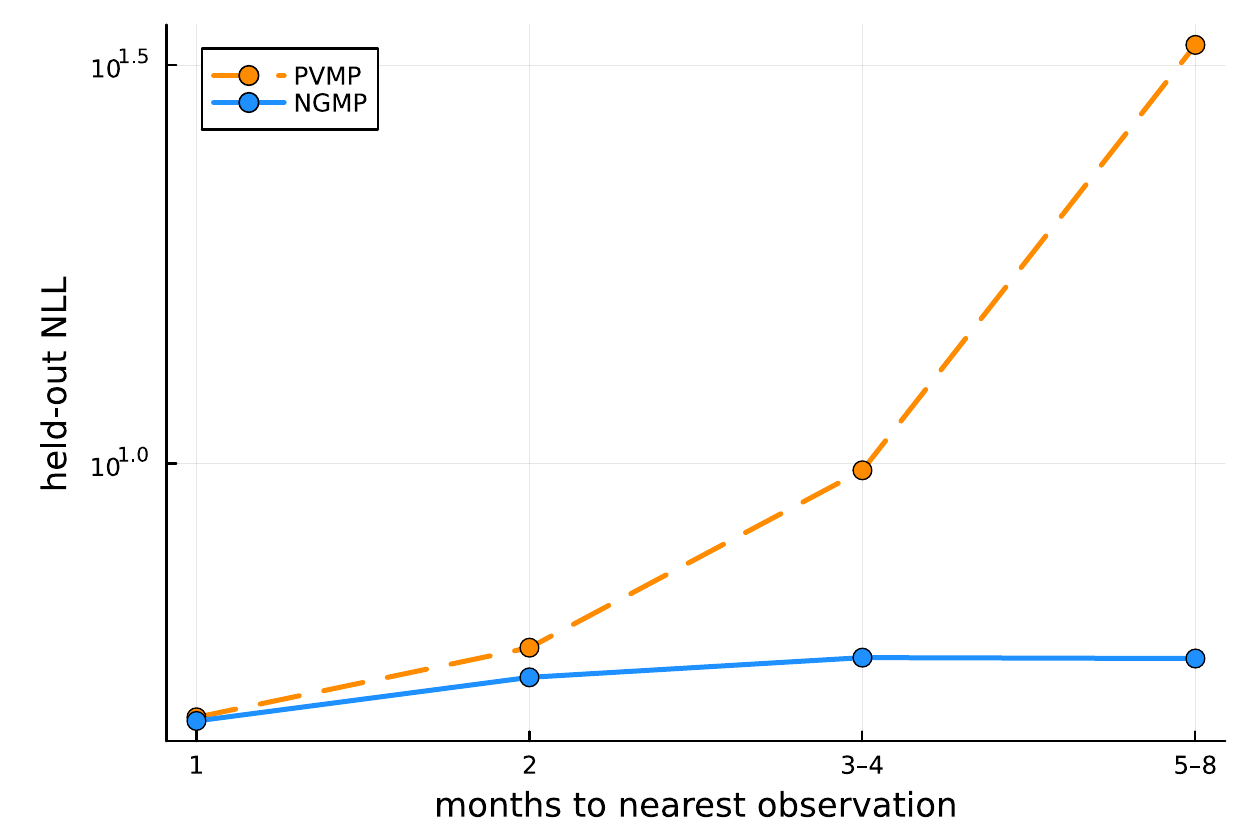}
    \hfill
    \includegraphics[width=0.44\linewidth]{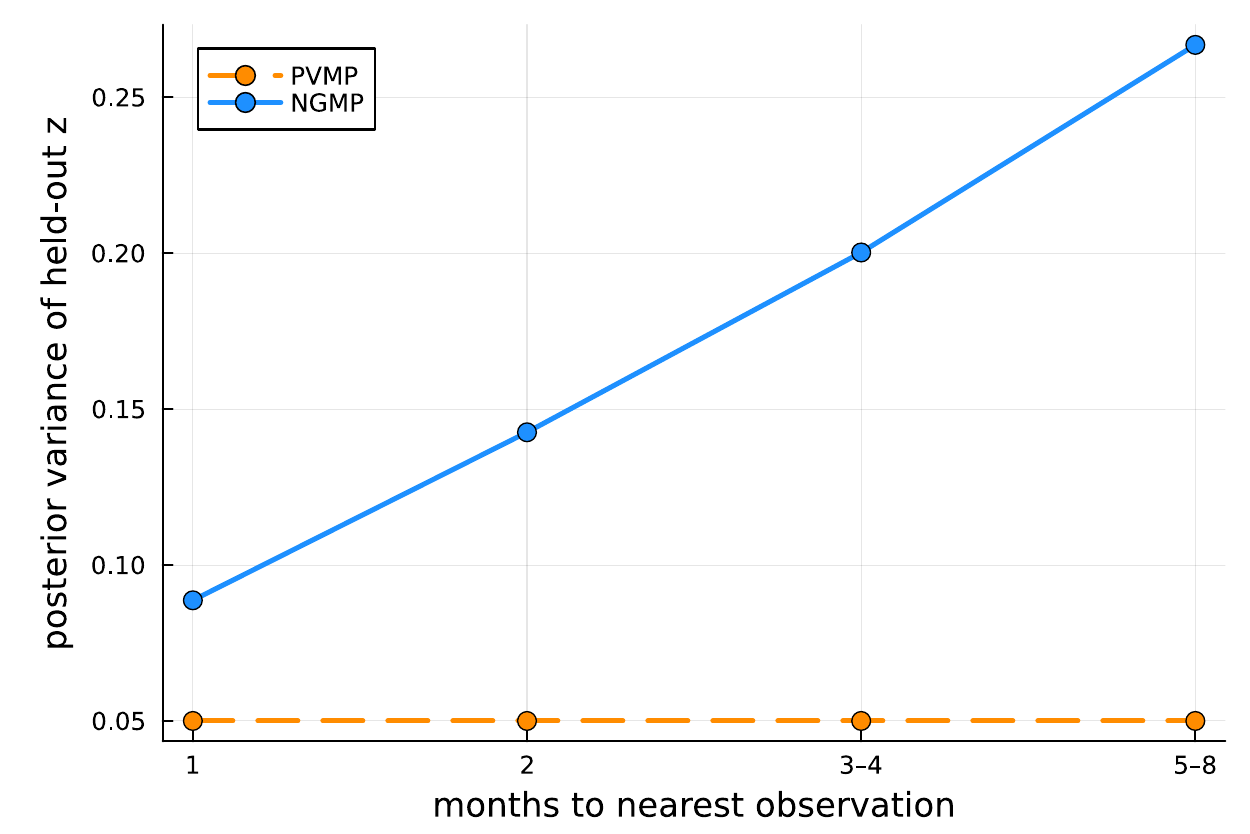}
    \caption{Sunspot results at $50\%$ holdout. Top: a five-year window
    around the deepest held-out stretch of a representative mask, one
    method per panel on identical axes; lines are the plug-in predictive
    rates $\exp(\mathbb E[z_k] + \operatorname{Var}(z_k)/2)$, shaded bands
    push the pointwise $95\%$ credible intervals of $z_k$ through the
    exponential link, and diamonds are held-out counts. PVMP's band keeps a
    fixed width inside the gap and its rate drifts confidently away from
    the held-out counts, while NGMP's band widens toward the gap center
    and keeps covering them. Bottom: held-out negative log-likelihood
    (log scale) and posterior state variance against distance to the
    nearest observation over all 20 masks.}
    \label{fig:edge-uncertainty-poisson}
\end{figure}

\begin{table}[t!]
    \centering
    \caption{Sunspot state-space model: held-out predictive metrics
    (mean $\pm$ 95\% CI over 20 random masks). PVMP and NGMP agree for short
    gaps and separate when the $50\%$ mask creates long uncertain stretches.
    The budget-matched NCVMP control is reported in
    \Cref{tab:app-budget-poisson} of \Cref{app:edge-uncertainty-fe}.}
    \label{tab:edge-uncertainty-poisson}
    \begin{tabular}{rcccc}
\toprule
 & \multicolumn{2}{c}{NLL} & \multicolumn{2}{c}{RMSE} \\
\cmidrule(lr){2-3} \cmidrule(lr){4-5}
Held out & PVMP & NGMP & PVMP & NGMP \\
\midrule
5\% & 4.534 $\pm$ 0.087 & 4.532 $\pm$ 0.085 & 13.918 $\pm$ 0.385 & 13.935 $\pm$ 0.378 \\
10\% & 4.569 $\pm$ 0.058 & 4.566 $\pm$ 0.058 & 13.972 $\pm$ 0.260 & 13.995 $\pm$ 0.260 \\
20\% & 4.672 $\pm$ 0.039 & 4.661 $\pm$ 0.037 & 14.429 $\pm$ 0.214 & 14.446 $\pm$ 0.208 \\
50\% & 5.411 $\pm$ 0.162 & 4.930 $\pm$ 0.035 & 16.839 $\pm$ 0.471 & 15.789 $\pm$ 0.162 \\
\bottomrule
\end{tabular}

\end{table}

\FloatBarrier

\paragraph{Error accumulation along a heteroscedastic filtering chain.}
Finally, consider heteroscedastic regression: the observation noise is not 
constant but depends on the input location, so the model must learn two
functions from the same data: the mean and the noise level. The oscillatory mean is $-(x + \tfrac12)\sin(3\pi x)$, whose true
noise standard deviation is $0.45\,\lvert x+0.5\rvert$, so the data are
nearly noise-free around $x=-0.5$ and increasingly noisy toward the ends of
the input range. The model uses separate random
Fourier feature maps $\phi(\cdot)$ for the mean and $\psi(\cdot)$ for the
log-precision,
\begin{subequations}\label{eq:edge-uncertainty-heteroscedastic}
\begin{align}
    s_o
    &\sim\mathcal N\!\left(
        \psi(x_o)^\top w,c_{\mathrm{top}}^{-1}
    \right), \\
    y_o
    &\sim\mathcal N\!\left(
        \phi(x_o)^\top v,e^{-s_o}
    \right).
\end{align}
\end{subequations}
The latent score $s_o$ is the log-precision of observation $o$: the
observation noise variance at input $x_o$ is $e^{-s_o}$, and $s_o$ depends
on the input through $\psi(x_o)^\top w$, so the first line of
\eqref{eq:edge-uncertainty-heteroscedastic} is the learned noise map over
input space.
Both conditional means are dot products between a fixed feature vector and a
weight vector;
\Cref{fig:edge-uncertainty-filtering-chain,fig:heteroscedastic-hierarchy-ffg}
draw each such conditional as a \texttt{softdot} node\footnote{The
\texttt{softdot} (\emph{soft dot product}) factor is
\begin{equation}
    f_{*}\!\left(z\mid b,\phi,\tau\right)
    =\mathcal N\!\left(z\mid b^{\top}\phi,\,\tau^{-1}\right),
    \label{eq:softdot-factor}
\end{equation}
a Gaussian over $z$ whose mean is the dot product of the coefficients $b$
and the features $\phi$ and whose precision $\tau$ arrives on a separate
edge; as $\tau\to\infty$ the factor degenerates to a deterministic dot
product, hence the name. \Cref{fig:edge-uncertainty-filtering-chain}(b)
opens the node graphically.}, reusing the notation of
\citet{lukashchuk_composing_2026}.

The 400 observations arrive in ten batches of 40, and the model is refit
after each batch. With $\theta=(w,v)$ and $q_b(\theta)=q_b(w)q_b(v)$
denoting the Gaussian weight marginals after batch $b$, each filtering
step multiplies the previous posterior by the approximate message of the
new batch,
\begin{equation}
    q_b(\theta)
    \propto
    q_{b-1}(\theta)\,
    \widehat\mu_{\mathcal B_b}(\theta;q_{b-1}),
    \qquad b=1,\ldots,10.
    \label{eq:edge-uncertainty-filtering}
\end{equation}
If each batch message were an exact Gaussian likelihood, this recursion
would be ordinary Kalman-style filtering, and the ten sequential updates
would return the same posterior as one fit to all 400 observations at
once. Approximate messages break this equivalence: writing
$\mathcal U_{\mathcal B}$ for the update induced by a batch $\mathcal B$,
in general
$\mathcal U_{\mathcal B_{10}}\circ\cdots\circ\mathcal U_{\mathcal B_1}
\neq\mathcal U_{\bigcup_{b=1}^{10}\mathcal B_b}$,
because each $\widehat\mu_{\mathcal B_b}$ depends on the posterior it is
applied to. The experiment measures how far each method falls short of
this ideal: the batching penalty in \Cref{tab:edge-uncertainty-streaming}
is the increase in held-out NLL of the sequential fit over the full-batch
fit, and a method with a small penalty can be trusted to filter online.
\Cref{fig:edge-uncertainty-filtering-chain} shows the corresponding factor
graph and online schedule.

\begin{table}[H]
    \centering
    \small
    \caption{Sequential heteroscedastic model over 20 paired seeds
    (mean $\pm$ 95\% CI). The same 400 observations are fitted jointly or
    filtered in ten batches; the batching penalty is the resulting increase
    in held-out NLL.}
    \label{tab:edge-uncertainty-streaming}
    \begin{tabular}{lccccc}
    \toprule
    & \multicolumn{2}{c}{full batch} & \multicolumn{2}{c}{sequential} & \\
    \cmidrule(lr){2-3}\cmidrule(lr){4-5}
    & NLL & RMSE & NLL & RMSE & batching penalty \\
    \midrule
    PVMP & $0.276 \pm 0.031$ & $0.588 \pm 0.037$ &
        $0.512 \pm 0.044$ & $0.583 \pm 0.039$ & $0.235 \pm 0.030$ \\
    NGMP & $0.235 \pm 0.031$ & $0.591 \pm 0.038$ &
        $0.244 \pm 0.031$ & $0.607 \pm 0.042$ & $0.009 \pm 0.010$ \\
    \bottomrule
    \end{tabular}
\end{table}

\begin{figure}[t!]
    \centering
    \input{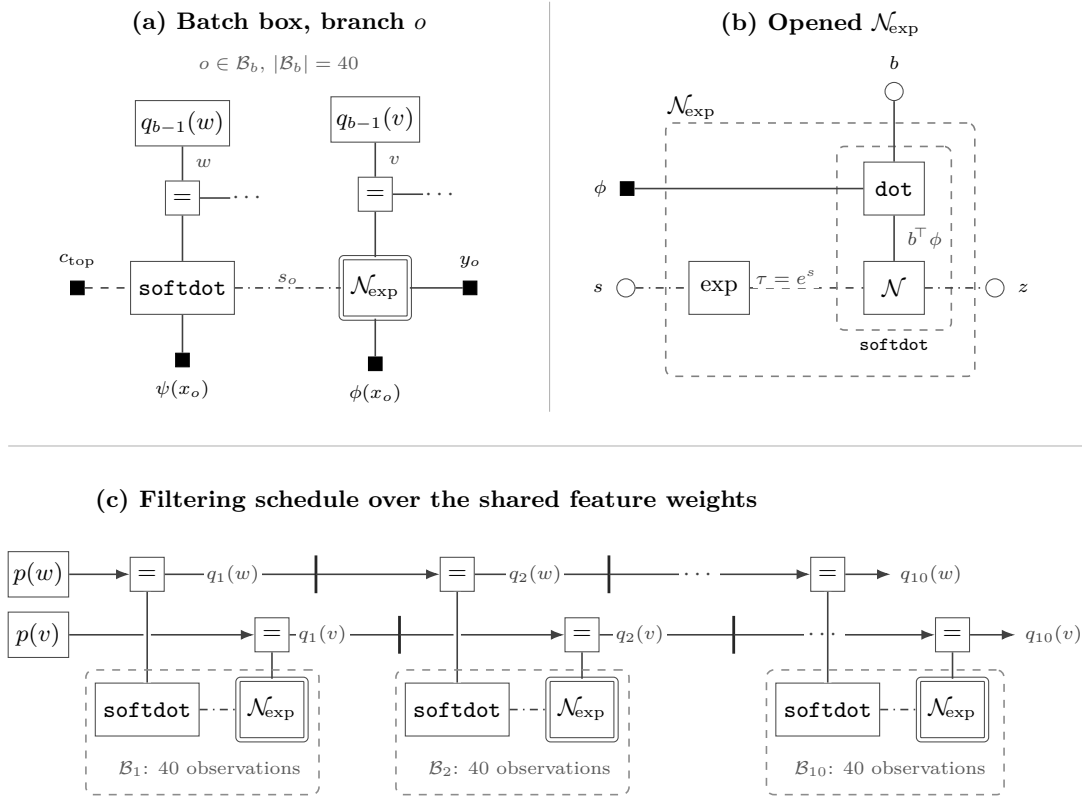}
    \par\vspace{5pt}
    \caption{Factor graph and online schedule for the sequential
    heteroscedastic model. Panel (a) shows branch $o$ of the batch box; each
    \texttt{softdot} node is the soft dot-product factor
    \eqref{eq:softdot-factor}. The upper \texttt{softdot} is the
    log-precision path of \eqref{eq:edge-uncertainty-heteroscedastic}, with
    $b=w$, $\phi=\psi(x_o)$, $\tau=c_{\mathrm{top}}$, and $z=s_o$; the
    observation path below it is the same factor with the uncertain
    precision $\tau=e^{s_o}$. The lateral stubs with ellipses at the
    equality nodes continue to the other observations
    $o'\in\mathcal B_b\setminus\{o\}$, so the same $w$ and $v$ are shared by
    all 40 branches. Panel (b) opens the double-bordered $\nexpsym$
    composite and, inside it, the \texttt{softdot} box; the inner box is
    \eqref{eq:softdot-factor} drawn graphically. The deterministic
    \texttt{dot} node forms the dot product $b^{\top}\phi$, the Gaussian
    $\mathcal N$ receives it as its mean, and the precision of that Gaussian
    is the exponentiated incoming score $\tau=e^s$. In panel (a) the ports
    are $b=v$, $\phi=\phi(x_o)$, $s=s_o$, and $z=y_o$. Panel (c) repeats
    the compact batch box for
    $\mathcal B_1,\mathcal B_2,\ldots,\mathcal B_{10}$ beneath the two
    shared-weight chains. Batch 1 starts from the model priors $p(w)p(v)$;
    after batch $b$, the Gaussian marginals $q_b(w)q_b(v)$ are frozen and
    copied left-to-right as the next priors. A perpendicular stop bar marks
    that no message from a later batch is sent back into an earlier batch;
    the bar specifies the online schedule and is not an additional model
    factor. With exact state-independent Gaussian messages, this
    posterior-as-prior multiplication is ordinary associative Kalman-style
    filtering. Solid edges carry weights, features, and observations;
    dashed edges carry precisions; dash-dotted edges carry the latent
    log-precision scores.}
    \label{fig:edge-uncertainty-filtering-chain}
\end{figure}

We use 128 Mat\'ern-$3/2$ random Fourier features plus an intercept for the
mean path and 32 RBF random Fourier features plus an intercept for the
log-precision path. The mean-path coefficients have a prior standard deviation
$2.0$, while the log-precision RFF coefficients retain a standard deviation
$1.6$. Each of the 400
likelihood factors has an uncertain local log-precision $s_o$, but all sites
contribute to the same shared noise weights $w$. Both methods share the graph
and marginal families.
The observation factor of \eqref{eq:edge-uncertainty-heteroscedastic} is non-conjugate with
respect to its log-precision edge, so the baseline is again PVMP: it computes
its expected-log-factor sites there with the projective update of
\eqref{eq:projective-vmp-marginal}, whereas NGMP projects the exact cavity
messages at the receiving marginals, with the required expectations computed
by respective tangent projections (\Cref{app:tangent-projection-numerics}).
Both methods run under the same $240$-sweep budget; Bethe traces and
final-sweep residual changes are given in \Cref{app:edge-uncertainty-fe}.

With all 400 observations fitted together, the held-out RMSE is effectively tied,
while NGMP has a small NLL advantage (\Cref{tab:edge-uncertainty-streaming}).
When the same observations are processed in ten batches of 40, the RMSE remains
similar, but PVMP's NLL rises by $0.235\pm0.030$ nats. NGMP's
$0.009\pm0.010$-nat penalty is consistent with batching invariance. Thus, the
predictive distributions, not merely an internal covariance diagnostic,
establish PVMP's batch-partition dependence in this model.

The mechanism is visible in the posterior over the noise weights $w$. After
the sequential fit, $\log\det\Sigma_w$ is $-72.34$ under PVMP and $-8.01$
under NGMP; across the $33$ coordinates, PVMP's geometric-mean posterior
variance per coordinate is therefore roughly $7.0\times$ smaller, because it
has accumulated the local under-dispersion of its $400$ sites into a shared
overconfident posterior. The collapse happens early: after the first
$40$-point batch, PVMP is already more concentrated
($\log\det\Sigma_w=-27.75$) than NGMP becomes with all 400 points, and a
prior that certain cannot be corrected by later batches
(\Cref{fig:edge-uncertainty-collapse} in \Cref{app:edge-uncertainty-fe}).
NGMP instead concentrates as evidence accumulates and ends near its
full-batch level ($-6.12$ versus $-8.01$).

\Cref{fig:edge-uncertainty-streaming} shows the two sequential fits. Both
methods resolve the oscillatory mean, but NGMP's sequential predictive
density stays close to its full-batch fit while PVMP's NLL degrades. The
learned variance functions in \Cref{fig:edge-uncertainty-variance} localize
the difference: batching distorts PVMP's posterior-mean variance into a sharp
peak near $x=2$ while its pointwise credible interval remains narrow, whereas
NGMP retains substantially more uncertainty about the variance, especially
where data are sparse. The separately plotted mean-weight contribution
distinguishes this epistemic component from the learned aleatoric variance.
The full-batch fits add no visible information at this scale and are deferred
to \Cref{fig:app-streaming-predictive,fig:app-streaming-variance} in
\Cref{app:edge-uncertainty-fe}.

\begin{figure}[t!]
    \centering
    \begin{subfigure}[t]{0.49\linewidth}
        \centering
        \includegraphics[width=\linewidth]{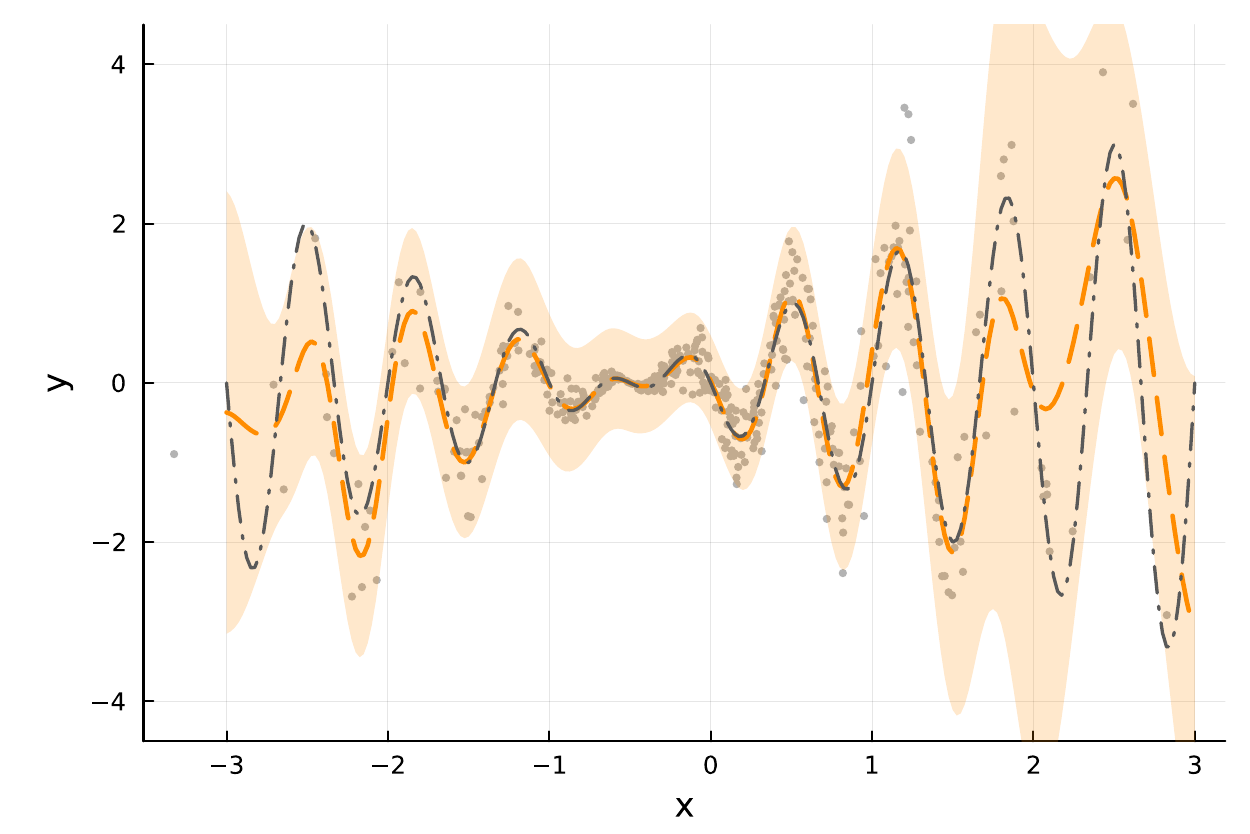}
        \caption{PVMP, sequential batches}
        \label{fig:edge-uncertainty-streaming-vmp-sequential}
    \end{subfigure}
    \hfill
    \begin{subfigure}[t]{0.49\linewidth}
        \centering
        \includegraphics[width=\linewidth]{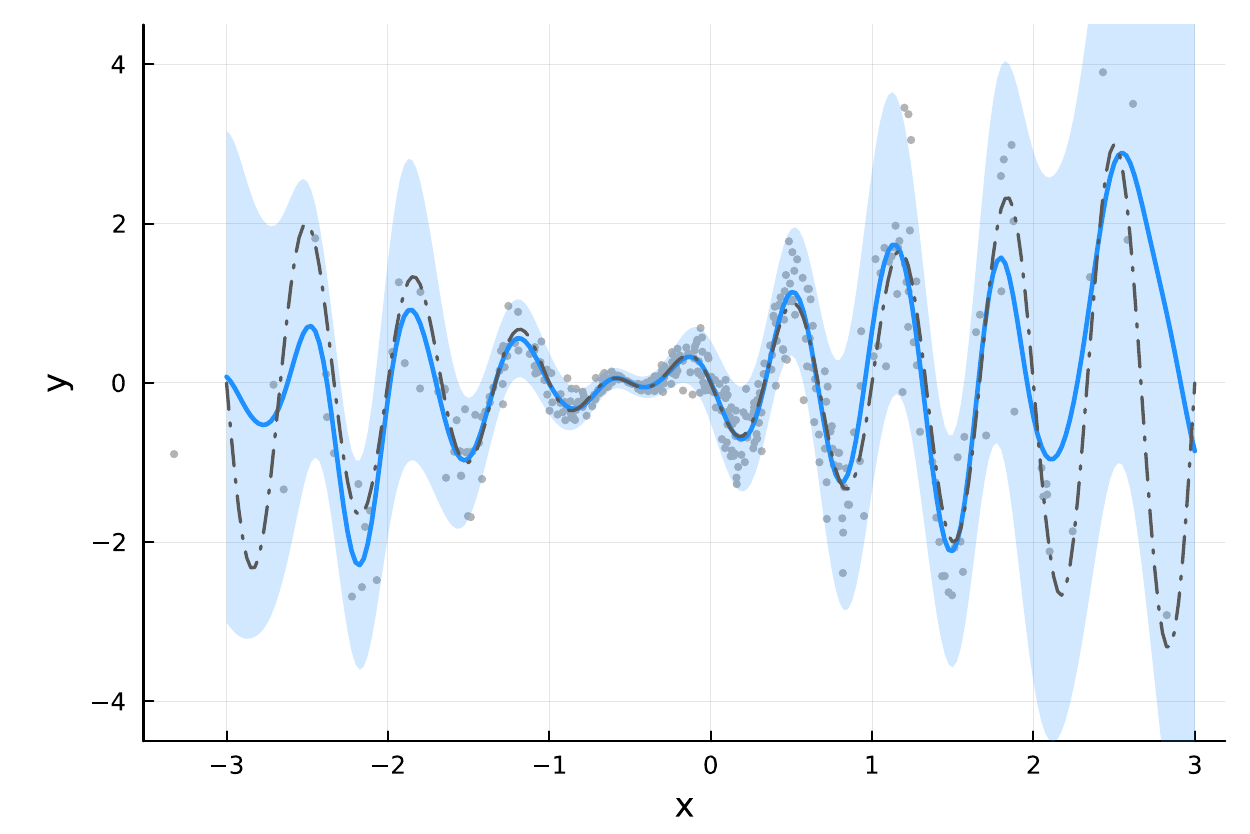}
        \caption{NGMP, sequential batches}
        \label{fig:edge-uncertainty-streaming-ngmp-sequential}
    \end{subfigure}
    \caption{Posterior predictive bands after the ten-step filtering chain of
    \Cref{fig:edge-uncertainty-filtering-chain}, on one representative seed
    with identical axes. The orange dashed curve is PVMP and the blue solid
    curve is NGMP; the matching shaded regions are pointwise 95\%
    posterior-predictive bands. Gray dash-dotted curves show the true mean,
    and gray points are the 400 training observations. Both methods resolve
    the oscillatory mean; under posterior reuse, PVMP's noise-weight covariance
    collapses and its NLL rises, whereas NGMP remains close to its full-batch
    NLL. The full-batch fits of both methods are visually indistinguishable
    from panel (b); \Cref{fig:app-streaming-predictive} in
    \Cref{app:edge-uncertainty-fe} shows all four fits.}
    \label{fig:edge-uncertainty-streaming}
\end{figure}

\begin{figure}[t!]
    \centering
    \begin{subfigure}[t]{0.32\linewidth}
        \centering
        \includegraphics[width=\linewidth]{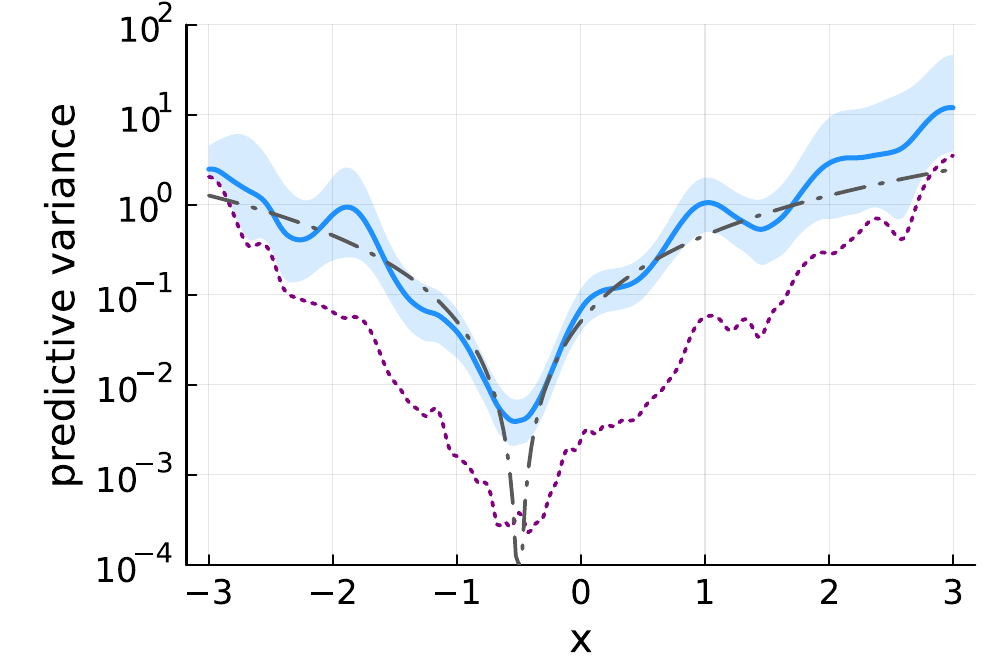}
        \caption{NGMP, sequential batches}
        \label{fig:edge-uncertainty-variance-ngmp-sequential}
    \end{subfigure}
    \hfill
    \begin{subfigure}[t]{0.32\linewidth}
        \centering
        \includegraphics[width=\linewidth]{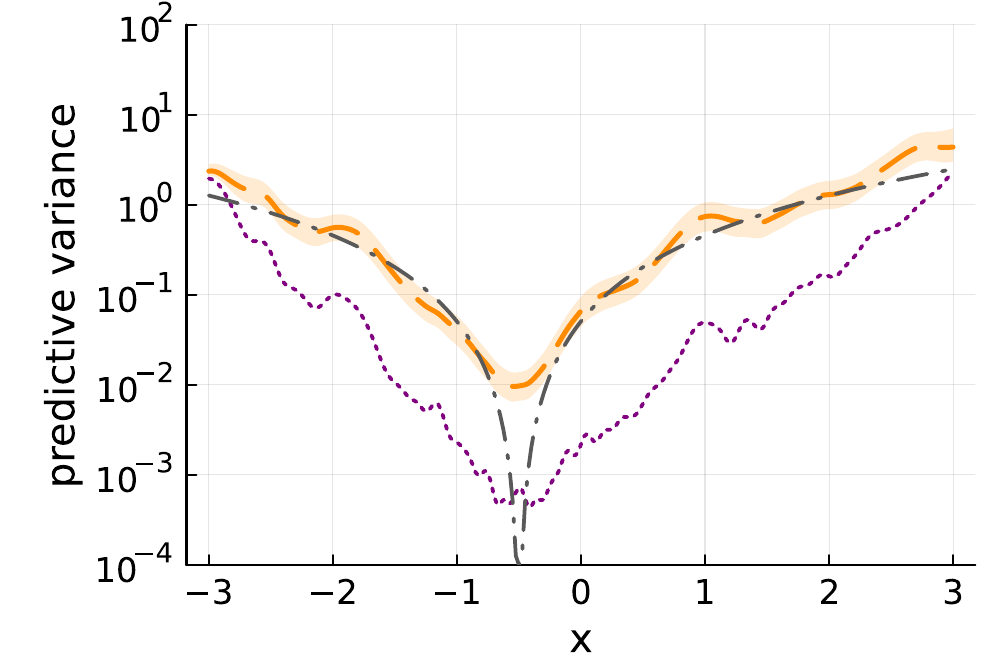}
        \caption{PVMP, full batch}
        \label{fig:edge-uncertainty-variance-vmp-full}
    \end{subfigure}
    \hfill
    \begin{subfigure}[t]{0.32\linewidth}
        \centering
        \includegraphics[width=\linewidth]{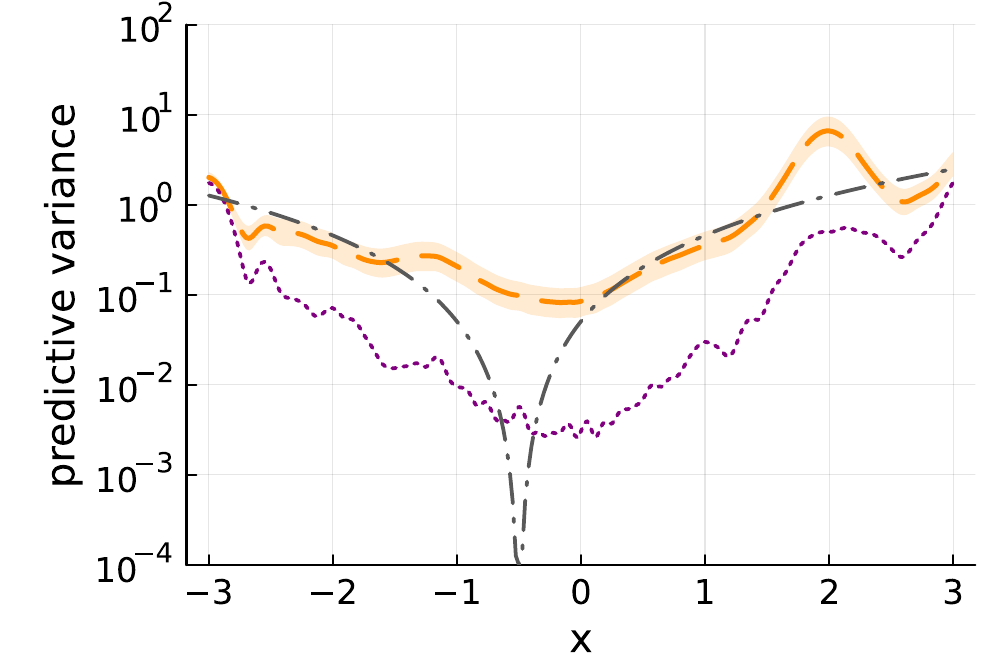}
        \caption{PVMP, sequential batches}
        \label{fig:edge-uncertainty-variance-vmp-sequential}
    \end{subfigure}
    \caption{Posterior predictive variance on the same representative seed as
    \Cref{fig:edge-uncertainty-streaming}. The blue solid curve in panel (a)
    is NGMP and the orange long-dashed curves in panels (b)--(c) are PVMP;
    panel (b) uses all observations at once, panels (a) and (c) use the ten
    sequential batches. Each curve is the posterior mean
    total predictive variance,
    $\bar S(x)=V_{\mathrm{epi}}(x)+V_{\mathrm{alea}}(x)$, where
    $V_{\mathrm{epi}}(x)=\phi(x)^\top\Sigma_v\phi(x)$ is epistemic variance
    from the uncertain mean weights and
    $V_{\mathrm{alea}}(x)=\mathbb E_q[e^{-s(x)}]$ is learned aleatoric
    variance. Shading gives a pointwise 95\% posterior credible interval for
    the total. The purple dotted curve shows $V_{\mathrm{epi}}(x)$ alone, so
    its vertical gap to the orange or blue total is $V_{\mathrm{alea}}(x)$.
    The gray dash-dotted curve is the benchmark's true aleatoric variance
    $[0.45(x+0.5)]^2$. These pointwise
    credible intervals describe posterior
    uncertainty for one fit and are distinct from the across-seed confidence
    intervals in \Cref{tab:edge-uncertainty-streaming}. All panels use
    identical logarithmic axes. NGMP's full-batch fit is visually
    indistinguishable from panel (a); \Cref{fig:app-streaming-variance} in
    \Cref{app:edge-uncertainty-fe} shows all four fits.}
    \label{fig:edge-uncertainty-variance}
\end{figure}

\FloatBarrier

\paragraph{Practical rule.}
Across these ablations, NGMP is indistinguishable from VMP and PVMP when the
relevant edge uncertainty is reducible, and it improves calibration or
prediction when uncertain updates are repeatedly composed. Hence, when both
updates are available for the same constrained interface, NGMP is the safer
default. However, the undamped NGMP fixed-point iteration can
oscillate and, in practice, may not converge, so every NGMP run in this paper
damps its updates in natural coordinates, and the deeper hierarchies of
\Cref{sec:experiments} additionally use vector-transport momentum
(\Cref{app:momentum-damping}). Damping rescales the same single
natural-gradient step, and momentum reuses the previous one, so neither adds
gradient evaluations. VMP and PVMP remain adequate when neighboring
beliefs are already concentrated; persistent uncertainty during smoothing or
state-dependent approximate messages reused during filtering are signals that
their local error may accumulate. The filtering result is model-specific:
exact Gaussian message multiplication remains associative, and NGMP is
empirically almost invariant to batching here. Closing PVMP's cost gap by
granting the projective update of \eqref{eq:projective-vmp-marginal} a
single inner step per edge, which is NCVMP, does not help: this control
underfits, still collapses along the same concentration path, and yields
uniformly worse results at the same runtime, so when PVMP is used, its inner
projections should be run to convergence
(\Cref{tab:app-budget-streaming,fig:edge-uncertainty-collapse} in
\Cref{app:edge-uncertainty-fe}).

\section{Experiments}\label{sec:experiments}
This section evaluates NGMP in two larger modeling settings. The point of
both is that probabilistic hierarchies assembled from the factors of
\Cref{sec:edge-uncertainty} learn non-linear dependencies well. The first
setting is supervised regression on six UCI data sets, where the depth-three
heteroscedastic hierarchy of \Cref{fig:heteroscedastic-hierarchy-ffg} is
compared with several variational Bayesian learning frameworks. This is a
framework-level comparison rather than another NGMP--PVMP ablation: direct
comparisons with PVMP are given in \Cref{sec:edge-uncertainty}, while applying
PVMP's inner manifold optimization at every non-conjugate edge is
computationally prohibitive for these data sets and the depth-three hierarchy.
Dataset dimensions and an exact model-capacity accounting are reported in
\Cref{tab:app-uci-dataset-sizes,tab:app-uci-model-capacity}. The second composes frozen pretrained
forecasters into a probabilistic ensemble, following the setup of
\citet{lukashchuk_composing_2026}; the Bayesian model wraps existing
predictors without retraining them. Homoscedastic baseline variants never win
either comparison, so the main text reports the heteroscedastic variants and
defers the homoscedastic ones to \Cref{app:modelling-details}.

\subsection{Regression}\label{sec:experiments-regression}

We evaluate whether the local natural-gradient construction remains useful in a
larger supervised-learning setting on six UCI regression data sets: Concrete,
Energy, Boston Housing, Power Plant, Wine Quality Red, and Yacht.  We use the
depth-three instance of the heteroscedastic hierarchy in
\Cref{fig:heteroscedastic-hierarchy-ffg}, which generalizes the sequential
heteroscedastic model of \Cref{sec:edge-uncertainty}. Its first
linear-Gaussian path models the predictive mean, and two uncertain precision
layers produce input-dependent observation noise. Every layer has its own
fixed feature realization and length scale.

\begin{figure}[t!]
    \centering
\input{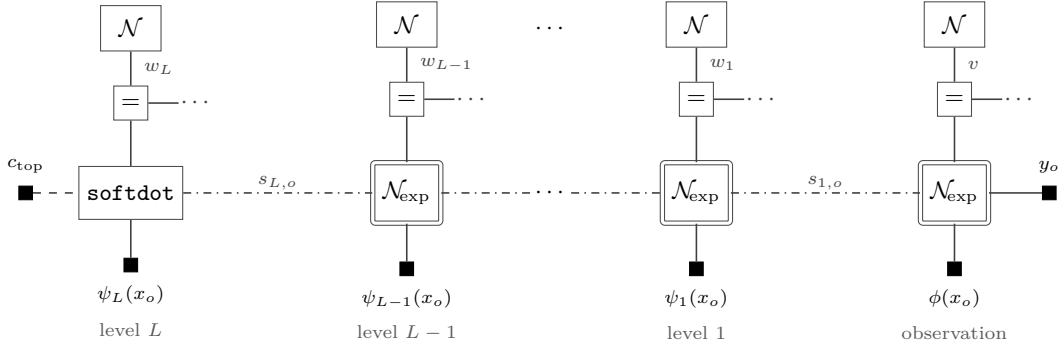}
\begin{tikzpicture}[ffg]

    \node[softfact] (lvlL) {\texttt{softdot}};
    \node[composite] (lvlLm) [right=25mm of lvlL] {$\nexpsym$};
    \node (middots) [right=12mm of lvlLm, inner sep=1pt] {$\cdots$};
    \node[composite] (lvl1) [right=12mm of middots] {$\nexpsym$};
    \node[composite] (lvl0) [right=25mm of lvl1] {$\nexpsym$};

    \draw[scoreedge] (lvlL) -- node[above, edgelabel] {$s_{L,o}$} (lvlLm);
    \draw[scoreedge] (lvlLm) -- (middots);
    \draw[scoreedge] (middots) -- (lvl1);
    \draw[scoreedge] (lvl1) -- node[above, edgelabel] {$s_{1,o}$} (lvl0);

    \node[clamped] (ctop) [left=6mm of lvlL] {};
    \node[portlabel, above=0.5mm of ctop] {$c_{\mathrm{top}}$};
    \draw[precisionedge] (ctop) -- (lvlL);

    \node[clamped] (yo) [right=7mm of lvl0] {};
    \node[portlabel, above=0.5mm of yo] {$y_o$};
    \draw[edge] (lvl0) -- (yo);

    \node[clamped] (featL) [below=5mm of lvlL] {};
    \node[portlabel] (featLlab) [below=0.5mm of featL] {$\psi_L(x_o)$};
    \draw[edge] (featL) -- (lvlL);
    \node[sublabel] at ($(featL)+(0,-8.5mm)$) {level $L$};

    \node[clamped] (featLm) [below=5mm of lvlLm] {};
    \node[portlabel] (featLmlab) [below=0.5mm of featLm] {$\psi_{L-1}(x_o)$};
    \draw[edge] (featLm) -- (lvlLm);
    \node[sublabel] at ($(featLm)+(0,-8.5mm)$) {level $L-1$};

    \node[clamped] (feat1) [below=5mm of lvl1] {};
    \node[portlabel] (feat1lab) [below=0.5mm of feat1] {$\psi_1(x_o)$};
    \draw[edge] (feat1) -- (lvl1);
    \node[sublabel] at ($(feat1)+(0,-8.5mm)$) {level $1$};

    \node[clamped] (feat0) [below=5mm of lvl0] {};
    \node[portlabel] (feat0lab) [below=0.5mm of feat0] {$\phi(x_o)$};
    \draw[edge] (feat0) -- (lvl0);
    \node[sublabel] at ($(feat0)+(0,-8.5mm)$) {observation};

    \node[eqfact] (eqL) [above=6mm of lvlL] {$=$};
    \node[prior] (priL) [above=4.5mm of eqL] {$\mathcal N$};
    \draw[edge] (priL) -- node[right=0.5mm, portlabel] {$w_L$} (eqL);
    \draw[edge] (eqL) -- (lvlL);
    \draw[edge] (eqL.east) -- ++(4mm,0) node[right, inner sep=0.5pt] {$\cdots$};

    \node[eqfact] (eqLm) [above=6mm of lvlLm] {$=$};
    \node[prior] (priLm) [above=4.5mm of eqLm] {$\mathcal N$};
    \draw[edge] (priLm) -- node[right=0.5mm, portlabel] {$w_{L-1}$} (eqLm);
    \draw[edge] (eqLm) -- (lvlLm);
    \draw[edge] (eqLm.east) -- ++(4mm,0) node[right, inner sep=0.5pt] {$\cdots$};

    \node at (middots |- priL) {$\cdots$};

    \node[eqfact] (eq1) [above=6mm of lvl1] {$=$};
    \node[prior] (pri1) [above=4.5mm of eq1] {$\mathcal N$};
    \draw[edge] (pri1) -- node[right=0.5mm, portlabel] {$w_1$} (eq1);
    \draw[edge] (eq1) -- (lvl1);
    \draw[edge] (eq1.east) -- ++(4mm,0) node[right, inner sep=0.5pt] {$\cdots$};

    \node[eqfact] (eq0) [above=6mm of lvl0] {$=$};
    \node[prior] (pri0) [above=4.5mm of eq0] {$\mathcal N$};
    \draw[edge] (pri0) -- node[right=0.5mm, portlabel] {$v$} (eq0);
    \draw[edge] (eq0) -- (lvl0);
    \draw[edge] (eq0.east) -- ++(4mm,0) node[right, inner sep=0.5pt] {$\cdots$};

\end{tikzpicture}
    \caption{One observation branch of the arbitrary-depth heteroscedastic
    hierarchy used in the regression experiment. Every level has its own
    Gaussian weights $w_\ell$, shared across observation branches through
    the equality nodes, and its own clamped feature map
    $\psi_\ell(x_o)$. The top-level \texttt{softdot} is the soft
    dot-product factor \eqref{eq:softdot-factor}. Each level turns its
    incoming score into a precision and emits the score $s_{\ell,o}$ read by
    the level below, and the observation level turns the last score into the
    noise precision of $y_o$. The double-bordered $\nexpsym$ composites are
    opened in \Cref{fig:edge-uncertainty-filtering-chain}(b). The experiments in this section use
    the depth-three instance. Edge styles are as in
    \Cref{fig:edge-uncertainty-filtering-chain}: solid for weights,
    features, and observations; dashed for precisions; dash-dotted for
    scores.}
    \label{fig:heteroscedastic-hierarchy-ffg}
\end{figure}

Each of the feature maps ($\phi$ for the mean layer, $\psi_\ell$ for the
precision layers) contains 1{,}000 random Fourier features, sampled once per
train--test split and then held fixed; posterior uncertainty is therefore over
the linear weights and latent precisions, not over the random features. The
Gaussian weight blocks admit closed-form updates, the non-conjugate
exponential links on the precision paths are updated with NGMP, and the
sweeps are run until convergence. Feature construction, initialization, and
optimization settings are given in \Cref{app:modelling-details}. Prediction
integrates the Gaussian weight posteriors: the predictive variance is the sum
of the mean-weight epistemic variance and the expected aleatoric variance
propagated through the two precision layers.

\paragraph{Protocol.}
We use 20 deterministic repeated-holdout splits with 90\% of each data set for
training and 10\% for testing.  Features and targets are standardized using
training statistics only.  All methods use the same saved split indices, so
their results are paired at the level of both splits and held-out examples.
The data sets range from 308 to 9{,}568 observations and from 4 to 13 input
features; the per-data-set counts are given in
\Cref{tab:app-uci-dataset-sizes}.
We compare NGMP with Bayes by Backprop (BBB)
\citep{blundell_weight_2015}, diagonal and full-covariance deterministic
variational inference (dDVI and DVI) \citep{wu_deterministic_2019}, Bayesian
predictive coding (BPC) \citep{tschantz_bayesian_2025}, and IVON
\citep{shen_variational_2024}.  BBB, dDVI, and DVI exist in a heteroscedastic
variant, shown in \Cref{tab:uci-nll}, and a homoscedastic variant with fixed
observation noise, marked by the prefix ho (hoBBB, hodDVI, hoDVI); BPC and
IVON are available only in homoscedastic form (hoBPC, hoIVON).
\Cref{app:modelling-details} explains the difference between the two variants
and reports the full tables.  The data sets come from
the UCI Machine Learning Repository \citep{asuncion_uci_2007}, and the
repeated-holdout convention follows the one established by \citet{hernandez-lobato_probabilistic_2015}.  The negative log likelihood (NLL)
includes the Jacobian required to
return from standardized targets to the original target units; RMSE is also
reported in original units.  \Cref{tab:uci-nll} gives means and approximate
95\% confidence-interval half-widths, $1.96$ times the standard error across
the 20 paired splits; the RMSE comparison is \Cref{tab:uci-rmse} in
\Cref{app:modelling-details}.

\paragraph{Results.}
NGMP has the lowest mean NLL on Concrete, Energy, Power,
and Yacht. On Wine, its interval includes the best point estimate, while on
Boston, the two DVI variants have the strongest NLL.  The RMSE comparison
(\Cref{tab:uci-rmse}) is more mixed: NGMP is best on Energy and Power,
is confidence-interval compatible with the best results on Yacht and Wine, and
does not improve on DVI for Concrete or Boston.  This distinction is expected
for a probabilistic model: RMSE assesses only the predictive mean, whereas NLL
also rewards the quality of the learned heteroscedastic
uncertainty.  In particular, the large NLL gains on Energy and Yacht,
together with competitive RMSE, indicate that the precision hierarchy improves
the predictive distribution rather than merely shifting its mean.

\begin{table*}[ht]
\centering
\caption{Test negative log likelihood (NLL) in original target units; lower is
better. Values are means \(\pm\) approximate 95\% confidence-interval
half-widths over 20 paired splits. Bold values include the best point estimate
within their confidence interval. BBB, dDVI, and DVI are the heteroscedastic
variants; IVON is homoscedastic and uses $K=20$ posterior network draws for
its predictive mixture. The remaining homoscedastic variants are in
\Cref{tab:app-uci-nll}.}
\label{tab:uci-nll}
\resizebox{\textwidth}{!}{%
\begin{tabular}{lrrrrrr}
\toprule
Method & Concrete & Energy & Boston & Power & Wine & Yacht \\
\midrule
BBB & $3.3328 \pm 0.0215$ & $2.4229 \pm 0.0350$ & $2.6555 \pm 0.0385$ & $2.7785 \pm 0.0143$ & $\bm{0.9438 \pm 0.0232}$ & $2.6565 \pm 0.0451$ \\
dDVI & $3.0477 \pm 0.0432$ & $\bm{1.1920 \pm 0.4625}$ & $\bm{2.4646 \pm 0.0921}$ & $2.8238 \pm 0.0183$ & $\bm{0.9440 \pm 0.0314}$ & $0.4641 \pm 0.1000$ \\
DVI & $3.0406 \pm 0.0438$ & $\bm{1.0094 \pm 0.2973}$ & $\bm{2.4489 \pm 0.0917}$ & $2.8249 \pm 0.0200$ & $\bm{0.9448 \pm 0.0307}$ & $0.4543 \pm 0.1096$ \\
IVON ($K=20$) & $3.7966 \pm 0.0082$ & $3.2693 \pm 0.0412$ & $3.3284 \pm 0.0855$ & $3.7839 \pm 0.0013$ & $1.0430 \pm 0.0156$ & $3.6584 \pm 0.0101$ \\
NGMP & $\bm{2.9845 \pm 0.0433}$ & $\bm{0.9783 \pm 0.0655}$ & $2.6720 \pm 0.0752$ & $\bm{2.7610 \pm 0.0226}$ & $\bm{0.9571 \pm 0.0293}$ & $\bm{0.2782 \pm 0.1437}$ \\
\bottomrule
\end{tabular}%
}
\end{table*}

\subsection{Ensemble Forecasting}\label{sec:experiments-ensemble-forecasting}

We follow the ensemble-forecasting protocol of
\citet{lukashchuk_composing_2026} and replace only its PVMP inference with
NGMP. Interested readers can find details of the ETTh tasks, frozen expert
bank, and training and evaluation splits in \Cref{app:etth-details}.

The controlled edge-uncertainty experiments in \cref{sec:edge-uncertainty}
predict a larger separation in NLL than in point error when
uncertain reliability updates are composed.  We test that prediction against
four neural gates and a precision-gated probabilistic ensemble.  The neural
baselines are an affine gate and a two-layer ReLU gate trained with Adam \citep{kingma_adam_2015}, and
an affine gate and a two-layer ReLU gate trained with IVON
\citep{shen_variational_2024}. Here $K=1000$ means that the IVON posterior
predictive is approximated with 1{,}000 gate-weight draws. Each draw produces
one Gaussian ensemble component; the reported NLL uses the equally weighted
mixture density, and the point forecast averages the component means. The
precise component construction is given in \Cref{app:etth-details}.
The Adam and IVON rows are independently trained neural benchmarks; no causal
optimizer effect is inferred from their contrast.  For readability, the main
table retains the stronger IVON neural baselines; all Adam results remain in
the appendix.

The gating model is the precision-gated ensemble (PGE) of
\citet{lukashchuk_composing_2026}: a depth-one instance of the precision
hierarchy in which the frozen experts form the lower level and their
predictions enter a precision-weighted mixture, while one softdot with an
exponential link maps the shared 65-dimensional context to each expert's
precision.  The PVMP row reports the numbers from that paper,
computed on the identical test split; the NGMP row runs this paper's inference
on the same model.

The affine IVON gate is the closest neural comparison to the precision-gated
model: both consume the same 65-dimensional context and frozen seven-expert
bank, and both map the context to seven reliability scores.  Their scale
semantics differ.  IVON optimizes the softmax-weighted expert squared error; NGMP, on
the other hand, performs inference in an explicit probabilistic precision model whose priors and likelihood determine both relative weights and absolute predictive scale.

\begin{table*}[t]
\centering
\caption{ETTh forecasting in standardized OT units. Entries are test-set point estimates $\pm$ approximate 95\% normal confidence-interval half-widths over test predictions; RMSE intervals use the delta method. Lower is better. Among the methods shown, the lowest point estimate is bold; the second-lowest is also bold when its paired 95\% normal confidence interval relative to the lowest, computed on identical test predictions, includes zero. For IVON, $K=1000$ is the number of gate-weight samples drawn from the fitted variational posterior to approximate the posterior predictive. The predictive-mixture construction and complete experimental setup are described in \Cref{app:etth-details}; the full results, including Adam-trained gates, are reported in \Cref{tab:etth-precision-gated-full}.}
\label{tab:etth-precision-gated}
\scriptsize
\resizebox{\textwidth}{!}{%
\begin{tabular}{lrrrr}
\toprule
\multicolumn{5}{c}{ETTh1} \\
\cmidrule(lr){1-5}
Method & 96 & 192 & 336 & 720 \\
\midrule
\multicolumn{5}{l}{\textit{RMSE}} \\
MoE -- IVON, affine gate ($K{=}1000$) & \(0.3832 \pm 0.0090\) & \(0.3695 \pm 0.0091\) & \(0.3672 \pm 0.0084\) & \(0.5428 \pm 0.0127\) \\
MoE -- IVON, ReLU gate ($K{=}1000$) & \(0.3832 \pm 0.0090\) & \(0.3697 \pm 0.0091\) & \(0.3677 \pm 0.0084\) & \(0.5334 \pm 0.0125\) \\
\addlinespace[2pt]
PGE -- PVMP \citep{lukashchuk_composing_2026} & \(0.3583 \pm 0.0083\) & \(0.3386 \pm 0.0085\) & \(\mathbf{0.3105 \pm 0.0073}\) & \(0.3347 \pm 0.0076\) \\
PGE -- NGMP (this work) & \(\mathbf{0.3554 \pm 0.0084}\) & \(\mathbf{0.3337 \pm 0.0086}\) & \(\mathbf{0.3115 \pm 0.0074}\) & \(\mathbf{0.3300 \pm 0.0080}\) \\
\midrule
\multicolumn{5}{l}{\textit{NLL}} \\
MoE -- IVON, affine gate ($K{=}1000$) & \(10.2777 \pm 0.7366\) & \(3.5515\!\times\!10^{4} \pm 1.1067\!\times\!10^{4}\) & \(429.0098 \pm 143.4073\) & \(168.4529 \pm 26.4000\) \\
MoE -- IVON, ReLU gate ($K{=}1000$) & \(2216.8905 \pm 676.2776\) & \(2.2344\!\times\!10^{8} \pm 9.2280\!\times\!10^{7}\) & \(58.6475 \pm 9.6980\) & \(94.2422 \pm 16.2638\) \\
\addlinespace[2pt]
PGE -- PVMP \citep{lukashchuk_composing_2026} & \(0.4120 \pm 0.0173\) & \(0.3701 \pm 0.0171\) & \(0.3141 \pm 0.0135\) & \(0.3763 \pm 0.0140\) \\
PGE -- NGMP (this work) & \(\mathbf{0.3888 \pm 0.0210}\) & \(\mathbf{0.3378 \pm 0.0200}\) & \(\mathbf{0.2877 \pm 0.0161}\) & \(\mathbf{0.3571 \pm 0.0153}\) \\
\bottomrule
\end{tabular}%
}
\medskip
\resizebox{\textwidth}{!}{%
\begin{tabular}{lrrrr}
\toprule
\multicolumn{5}{c}{ETTh2} \\
\cmidrule(lr){1-5}
Method & 96 & 192 & 336 & 720 \\
\midrule
\multicolumn{5}{l}{\textit{RMSE}} \\
MoE -- IVON, affine gate ($K{=}1000$) & \(0.5821 \pm 0.0134\) & \(\mathbf{0.5198 \pm 0.0129}\) & \(0.5869 \pm 0.0142\) & \(0.8249 \pm 0.0181\) \\
MoE -- IVON, ReLU gate ($K{=}1000$) & \(0.5821 \pm 0.0134\) & \(0.5655 \pm 0.0139\) & \(0.5874 \pm 0.0142\) & \(0.6202 \pm 0.0129\) \\
\addlinespace[2pt]
PGE -- PVMP \citep{lukashchuk_composing_2026} & \(0.5882 \pm 0.0120\) & \(0.5798 \pm 0.0127\) & \(0.5940 \pm 0.0128\) & \(\mathbf{0.5669 \pm 0.0131}\) \\
PGE -- NGMP (this work) & \(\mathbf{0.5631 \pm 0.0124}\) & \(0.5408 \pm 0.0128\) & \(\mathbf{0.5463 \pm 0.0125}\) & \(0.6101 \pm 0.0143\) \\
\midrule
\multicolumn{5}{l}{\textit{NLL}} \\
MoE -- IVON, affine gate ($K{=}1000$) & \(5.9904\!\times\!10^{21} \pm 6.4398\!\times\!10^{21}\) & \(304.6987 \pm 52.2098\) & \(3.7880\!\times\!10^{5} \pm 1.0683\!\times\!10^{5}\) & \(3.2209\!\times\!10^{5} \pm 7.4793\!\times\!10^{4}\) \\
MoE -- IVON, ReLU gate ($K{=}1000$) & \(1.9494\!\times\!10^{18} \pm 1.3836\!\times\!10^{18}\) & \(1.1309\!\times\!10^{5} \pm 6.8836\!\times\!10^{4}\) & \(2.0843\!\times\!10^{6} \pm 5.8045\!\times\!10^{5}\) & \(1.1263\!\times\!10^{6} \pm 2.1811\!\times\!10^{5}\) \\
\addlinespace[2pt]
PGE -- PVMP \citep{lukashchuk_composing_2026} & \(\mathbf{0.9342 \pm 0.0306}\) & \(0.9237 \pm 0.0328\) & \(0.9612 \pm 0.0334\) & \(\mathbf{0.8699 \pm 0.0299}\) \\
PGE -- NGMP (this work) & \(\mathbf{0.9397 \pm 0.0386}\) & \(\mathbf{0.8602 \pm 0.0365}\) & \(\mathbf{0.8702 \pm 0.0348}\) & \(0.9774 \pm 0.0357\) \\
\bottomrule
\end{tabular}%
}
\end{table*}

\paragraph{Results.}
Across the benchmark, RMSE alone often makes the neural gates appear
competitive, whereas their NLL is substantially worse than that of both
precision-gated models. This larger separation in NLL than in RMSE supports
the prediction from the controlled experiments. Against the PVMP results of
\citet{lukashchuk_composing_2026}, NGMP has lower NLL at all
four ETTh1 horizons and at ETTh2 horizons 192 and 336.  It has lower RMSE at
ETTh1 horizons 96, 192, and 720 and at ETTh2 horizons 96, 192, and 336.  Normal
confidence intervals are shown with every point estimate.  Full results,
including the Adam gates omitted from the compact main table, are reported in
\cref{tab:etth-precision-gated-full}.

\begin{figure*}[t]
    \centering
    \includegraphics[width=0.89\textwidth]{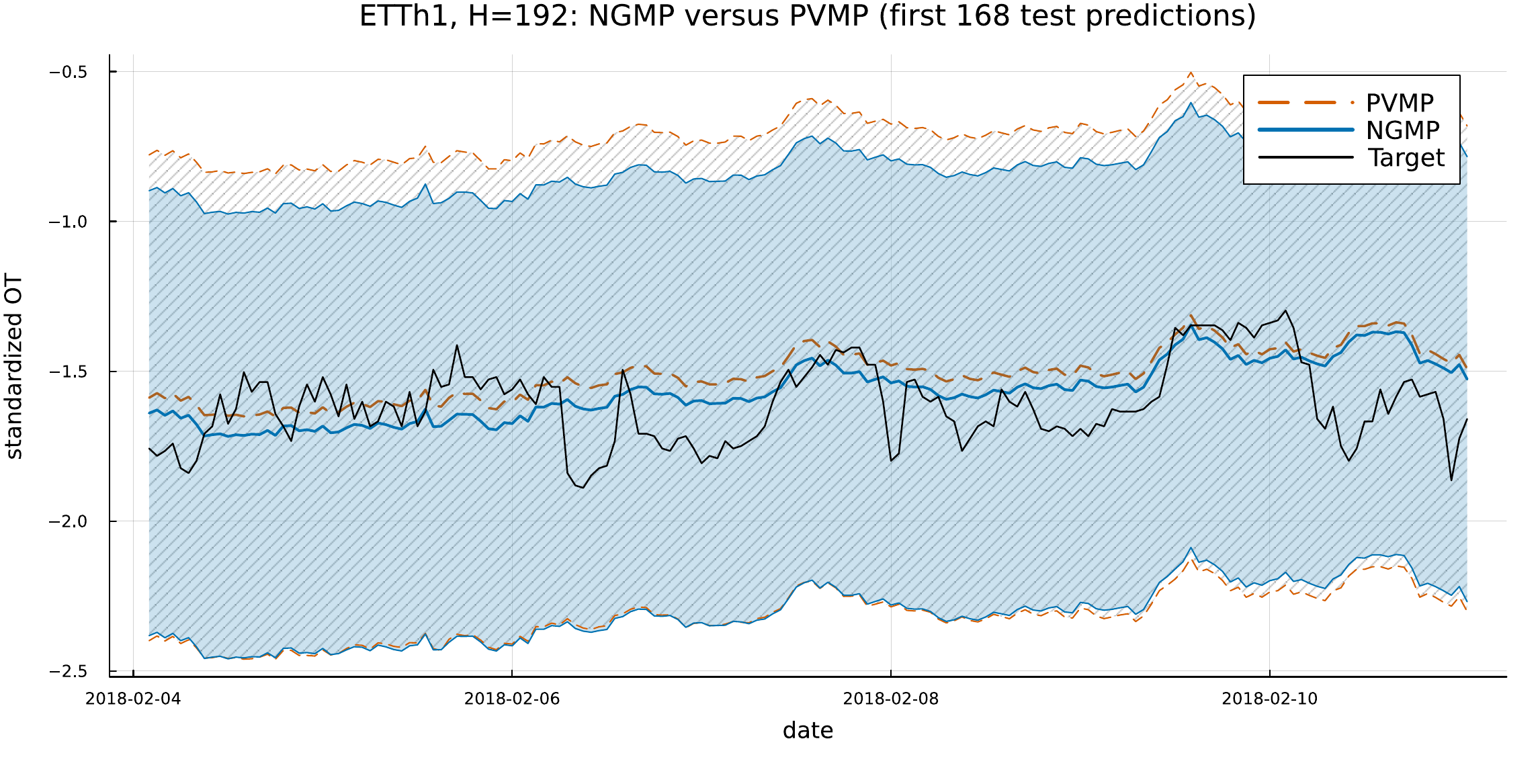}
    \input{figures/etth1_h192_ngmp_vmp_caption}
    \label{fig:etth1-h192-ngmp-vmp}
\end{figure*}

\section{Discussion and Conclusion}
This paper establishes an edge-local counterpart of the natural-gradient
stationarity condition for fixed-form variational inference. Each finite
message is the Fisher projection of an exact BP log-message at its receiving
marginal, which keeps BP, structured VMP, EP, and NGMP in one local
variational language while preserving their different information flows. The
experiments locate the practical consequence: when uncertainty persists on
the inputs to a non-conjugate factor and approximate updates are repeatedly
composed, retaining its representable component improves calibration and
prediction. The ETTh study exhibits the same effect in a larger model, where
competitive point forecasts alone conceal a collapse of predictive scale.

\paragraph{Compiling factor graphs to surrogate programs.}
The surrogate interpretation also suggests a systems-level research
direction. \Cref{fig:poisson-normal-surrogate} shows two representations of
the same projected computation, not two approximations whose posteriors happen
to agree: projected messages on the original graph and conjugate leaves on the
surrogate graph implement the same update. \Cref{fig:multi-interface-surrogate}
extends this identity to one surrogate leaf per constrained interface, and
\Cref{eq:surrogate-fixed-point-map,eq:surrogate-undamped-fixed-point-iteration}
collect the resulting computation in the fixed-point map. \texttt{RxInfer.jl}
combines the expressive \texttt{GraphPPL.jl} model representation with \texttt{ReactiveMP.jl}, where
messages and marginals are reactive streams whose subscribers trigger local
updates \citep{nuijten_graphppljl_2024,bagaev_reactive_2023}.
This execution model is valuable for streaming and selective recomputation,
but it also leaves subscription dispatch and graph traversal in the runtime.
A complementary \texttt{GraphPPL.jl} backend could instead traverse the factor graph once,
generate the required per-interface projection kernels and conjugate surrogate
leaves, fix the inner BP schedule, and emit a typed numerical implementation of
the message passing. This lowering would not change the asymptotic arithmetic complexity or
the NGMP fixed point; it would move graph orchestration out of the hot loop and
expose independent projections and batched linear algebra to compilation by,
for example,
\href{https://github.com/EnzymeAD/Reactant.jl}{\texttt{Reactant.jl}}, whose
MLIR tracing system is being extended with control-flow-aware Julia static
analysis \citep{lounes_controlflow_2026}, or JAX
\citep{bradbury_jax_2018}. Specialized libraries such as Dynamax already expose
pure JAX inference kernels for state-space model classes
\citep{linderman_dynamax_2025}; the opportunity here is to retain GraphPPL's
general modeling interface while obtaining a similarly static execution path.
We performed this lowering by hand for the models in this paper. Automating it
would remove substantial implementation work and turn the surrogate view into
a practical compiler target.

\paragraph{Continuous-state active-inference message passing.}
A second prospect concerns planning. The AIF-MP construction of
\citet{nuijten_what_2026}, building on the EFE-as-variational-inference
formulation of \citet{nuijten_expected_2026}, uses channel-reparameterized
observation and dynamics factors but restricts its implementation to discrete
state spaces with exact factor evaluations. For continuous states, the exact
messages from these modified factors will generally not remain in a tractable
finite family. NGMP suggests a principled relaxation: retain the
AIF-MP factors and their entropy corrections, impose exponential-family form
constraints on the continuous state edges, and project each factor-to-edge
log-message according to \eqref{eq:ng-message-eta}. The required local
expectations could be analytic or approximated with the sigma-point and
quadrature strategies of \Cref{app:tangent-projection-numerics}. Unlike a
mean-field factor update, this projection preserves the portion of cavity
uncertainty visible to the chosen state family, which is precisely the
information needed when planning depends on uncertainty about future states.
Deriving the full channel-augmented constrained-Bethe scheme and testing its
convergence and control performance remain future work; the present result
identifies NGMP as a concrete route from discrete exact AIF-MP to approximate
continuous-state inference.

Together, these prospects suggest a separation between modeling and execution:
using factor graphs as the expressive modeling interface and surrogate programs
as their compiled numerical representation. NGMP supplies the mathematical
translation between the two while retaining as much of the uncertainty as the variational family representation allows, which is needed for
downstream problems such as continuous-state planning. Automating this
translation and establishing when its continuous-state projections are
sufficiently accurate for control are the fruitful next steps.

\section*{Acknowledgements}

We gratefully acknowledge financial support by the Dutch Ministry of Economic Affairs (PPS funding), by the Dutch Research Council (NWO) and by hearing aid manufacturer GN Hearing, under contracts \\ TKI-HTSM/21.0161/2112P09 (project: Auto-AR) and KICH3.LTP.20.006 (Project: ROBUST). This work was additionally funded by the Eindhoven Artificial Intelligence Systems Institute (EAISI), under their Exploratory Multidisciplinary AI Research Program (EMDAIR), as part of the Embodied AI for Continuous Human-Like Learning Project (EMBODeAI), and we gratefully acknowledge their support.

We are grateful for insightful discussions with Bert de Vries and other BIASlab members. 

\bibliography{references}
\bibliographystyle{tmlr}

\appendix
\section{Proof of Theorem~\ref{thm:ngmp-edge}}
\label{app:ngmp-proof}
\begin{proof}[Proof of \cref{thm:ngmp-edge}]
We suppress additive constants, since all message multipliers are defined only up to normalization gauge.

First consider the factor-side variation. The edge form constraint changes the edge variable but does not constrain the adjacent factor beliefs. Therefore the same variation with respect to $q_a(\vz_a)$ as in the Bethe Lagrangian gives
\begin{equation}
    q_a^{*}(\vz_a)
    \;\propto\;
    f_a(\vz_a)
    \prod_{j\in\mathcal E(a)}
    \exp\!\big(\lambda_{aj}(z_j)\big).
    \label{eq:ng-qa-stationarity}
\end{equation}

Now consider the edge-side variation. The only terms depending on $\vlambda_i$ are collected in \eqref{eq:ng-local-lagrangian}. The exponential-family identities are
\begin{equation}
    \nabla_{\lambda_i}A_i^{*}\!\big(\vmu_i(\vlambda_i)\big)
    =
    \fisher_i(\vlambda_i)\vlambda_i,
    \qquad
    \nabla_{\lambda_i}
    \mathbb E_{q_{\lambda_i}}[r]
    =
    \operatorname{Cov}_{q_{\lambda_i}}\!\big[T_i,r\big].
    \label{eq:ng-two-identities}
\end{equation}
Applying \eqref{eq:ng-two-identities} to \eqref{eq:ng-local-lagrangian} gives
\[
    0
    =
    -\fisher_i(\vlambda_i)\vlambda_i
    +
    \sum_{a\in\mathcal V(i)}
    \operatorname{Cov}_{q_{\lambda_i}}
    \big[T_i,\lambda_{ai}\big],
\]
which is exactly the edge stationarity condition \eqref{eq:ng-lambda-stationary}. Multiplying by $\fisher_i(\vlambda_i)^{-1}$ gives the projected form
\[
    \vlambda_i
    =
    \sum_{a\in\mathcal V(i)}
    \tangentproj{i}\!\big[\lambda_{ai}\big].
\]

It remains to close the multipliers through the marginalization constraints. Substitute \eqref{eq:ng-qa-stationarity} into the constraint
\[
    q_{\lambda_i}(z_i)
    =
    \int q_a(\vz_a)\,d\vz_{a\setminus i}.
\]
For each $a\in\mathcal V(i)$ this gives
\[
    q_{\lambda_i}(z_i)
    \propto
    \exp\!\big(\lambda_{ai}(z_i)\big)
    \int
    f_a(\vz_a)
    \prod_{j\in\mathcal E(a)\setminus i}
    \exp\!\big(\lambda_{aj}(z_j)\big)
    d\vz_{a\setminus i}.
\]
Using the definition \eqref{eq:ng-bp-logmsg}, the same identity in log-coordinates is
\begin{equation}
    \log q_{\lambda_i}(z_i)
    =
    \lambda_{ai}(z_i)
    +
    \ell_{a\to i}(z_i)
    +
    \mathrm{const}.
    \label{eq:ng-marg-pointwise}
\end{equation}

Apply the projection $\tangentproj{i}$ to both sides of \eqref{eq:ng-marg-pointwise}. Constants vanish under covariance with $T_i$, and $\tangentproj{i}[\log q_{\lambda_i}]=\vlambda_i$ because $\log q_{\lambda_i}$ is affine in $T_i$. Therefore
\begin{equation}
    \vlambda_i
    =
    \tangentproj{i}\!\big[\lambda_{ai}\big]
    +
    \tangentproj{i}\!\big[\ell_{a\to i}\big]
    =
    \tangentproj{i}\!\big[\lambda_{ai}\big]
    +
    \eta_{a\to i}.
    \label{eq:ng-edge-closure}
\end{equation}

Finally use that $\mathcal V(i)=\{b,c\}$. From \eqref{eq:ng-edge-closure},
\[
    \tangentproj{i}\!\big[\lambda_{bi}\big]
    =
    \vlambda_i-\eta_{b\to i},
    \qquad
    \tangentproj{i}\!\big[\lambda_{ci}\big]
    =
    \vlambda_i-\eta_{c\to i}.
\] 
Substituting these two identities into the projected form of \eqref{eq:ng-lambda-stationary} gives
\[
    \vlambda_i
    =
    \big(\vlambda_i-\eta_{b\to i}\big)
    +
    \big(\vlambda_i-\eta_{c\to i}\big),
\]
and hence
\[
    \vlambda_i
    =
    \eta_{b\to i}
    +
    \eta_{c\to i}.
\]
This is \eqref{eq:ng-main-theorem}.
\end{proof}

\section{Supplementary Derivation for Information-Geometric Message Passing}
\label{app:info-gem-message-passing}
This appendix gives a deliberately small example of inductive inference. The
derivations below are not meant as a new algorithmic contribution; the same
stationary equations are obtained from the conjugate-computation view of
\citet{khan_conjugatecomputation_2017}. The purpose here is narrower and more
pedagogical. We use a graph with one latent variable and two unary factors, so
there is no factorization approximation hiding in the notation. In this setting, the global conjugate-computation view and the local constrained-Bethe view
coincide completely, which makes it a useful place to clarify language.

The main point is to clarify which object is meant by a \emph{message}
once a marginal is constrained to a finite exponential family. The exact
Lagrange multipliers of a constrained free-energy problem are functional
objects, whereas the messages used by conjugate message-passing algorithms are
finite exponential-family surrogates. This appendix shows how the latter can be
computed from the former without claiming that the underlying non-conjugate
factor has become conjugate.

We follow the same variational logic as the constrained free-energy derivations
of \citet{caticha_relative_2004,caticha_entropic_2011} and the constrained
Bethe-free-energy construction of \citet{senoz_variational_2021}. A
model induces a positive unnormalized density; a free-energy functional ranks
candidate posterior beliefs; and constraints restrict which beliefs are
admissible. The calculus of variations then tells us what the stationary condition requires.
Here, the graph has only one latent variable, so all of the bookkeeping is
visible on the page.

\subsection{Exact Constrained Model}
\label{app:exact-multipliers-projected-messages}

Let the latent variable $z$ be connected to two unary factors,
\begin{equation}
    f(z) = f_{\mathrm c}(z)\,f_{\mathrm n}(z),
    \label{eq:exact-multipliers-product-model}
\end{equation}
where $f_{\mathrm c}$ is conjugate to the chosen family $\mathcal E$, and
$f_{\mathrm n}$ is not. We constrain the marginal belief on $z$ to
$q_\lambda\in\mathcal E$, with
\begin{equation}
    q_\lambda(z)
    =
    \exp\{\vlambda^{\top}T(z)-A(\vlambda)\},
    \qquad
    \vmu=\mathbb E_{q_\lambda}[T(z)].
    \label{eq:exact-multipliers-family}
\end{equation}
Here and below, densities are written with respect to the carrier reference
measure $\mathrm{d}\nu(z)=h(z)\,\mathrm{d}z$. With respect to the Lebesgue measure,
the density in \eqref{eq:exact-multipliers-family} is multiplied by $h(z)$.
The two factor beliefs $q_{\mathrm c}$ and $q_{\mathrm n}$ are left
unrestricted. For this product graph, the local Bethe objective is
\begin{equation}
    \mathcal F(q_{\mathrm c},q_{\mathrm n},q_\lambda)
    =
    \sum_{a\in\{\mathrm c,\mathrm n\}}
    \int q_a(z)\log\frac{q_a(z)}{f_a(z)}\,\mathrm{d}\nu(z)
    -
    \int q_\lambda(z)\log q_\lambda(z)\,\mathrm{d}\nu(z) .
    \label{eq:exact-multipliers-bethe-objective}
\end{equation}
The local-polytope constraints are
\begin{equation}
    q_{\mathrm c}(z)=q_\lambda(z),
    \qquad
    q_{\mathrm n}(z)=q_\lambda(z).
    \label{eq:exact-multipliers-product-constraints}
\end{equation}
Introducing functional multipliers $\rho_{\mathrm c}$ and $\rho_{\mathrm n}$
for these constraints gives the Lagrangian
\begin{align}
    \mathcal L
    &=
    \mathcal F(q_{\mathrm c},q_{\mathrm n},q_\lambda)
    +
    \sum_{a\in\{\mathrm c,\mathrm n\}}
    \int \rho_a(z)\{q_\lambda(z)-q_a(z)\}\,\mathrm{d}\nu(z)
    \nonumber\\
    &\quad
    +
    \sum_{a\in\{\mathrm c,\mathrm n\}}
    \tau_a\left(\int q_a(z)\,\mathrm{d}\nu(z)-1\right)
    +
    \tau\left(\int q_\lambda(z)\,\mathrm{d}\nu(z)-1\right).
    \label{eq:exact-multipliers-lagrangian}
\end{align}
The last normalization term is redundant when $q_\lambda$ is parameterized as
a normalized density, but keeping it visible makes the constrained variational
problem explicit. Taking the variational derivative with respect to $q_a$
gives
\begin{equation}
    \frac{\delta\mathcal L}{\delta q_a(z)}
    =
    \log\frac{q_a(z)}{f_a(z)}+1-\rho_a(z)+\tau_a.
    \label{eq:exact-multipliers-factor-derivative}
\end{equation}
Setting \eqref{eq:exact-multipliers-factor-derivative} to zero yields
\begin{equation}
    q_a(z)
    \;\propto\;
    f_a(z)\exp\{\rho_a(z)\},
    \qquad a\in\{\mathrm c,\mathrm n\}.
    \label{eq:exact-multipliers-factor-stationarity}
\end{equation}
Therefore, the exact marginalization constraints can be satisfied by
\begin{equation}
    \rho_a(z)
    =
    \log q_\lambda(z)-\log f_a(z)+\mathrm{const}.
    \label{eq:exact-multipliers-ratio}
\end{equation}
This equation is a useful sanity check. Exact marginalization is not the
obstruction: the functional multipliers can always absorb whatever is needed to
make $q_a=q_\lambda$. The obstruction is finite representation. If
$f_a$ is non-conjugate, then \eqref{eq:exact-multipliers-ratio} contains the
negative non-conjugate residual of $\log f_a$. Thus, the exact Lagrange
multiplier is generally not an exponential-family message.

The stationarity of the constrained marginal itself adds only a projected
condition. Variation with respect to the finite parameter $\vlambda$, holding
the multipliers fixed, gives
\begin{equation}
    0
    =
    -\fisher(\vlambda)\vlambda
    +
    \operatorname{Cov}_{q_\lambda}\!\big[T,\rho_{\mathrm c}\big]
    +
    \operatorname{Cov}_{q_\lambda}\!\big[T,\rho_{\mathrm n}\big].
    \label{eq:exact-multipliers-lambda-derivative}
\end{equation}
Premultiplying by $\fisher(\vlambda)^{-1}$ gives
\begin{equation}
    \vlambda
    =
    \Pi^{\lambda}[\rho_{\mathrm c}]
    +
    \Pi^{\lambda}[\rho_{\mathrm n}],
    \label{eq:exact-multipliers-projected-rho}
\end{equation}
where $\Pi^{\lambda}[r]=\nabla_{\mu}\mathbb E_{q_\lambda}[r(z)]$. Thus
, the stationary condition identifies only the components of the multipliers visible through
the sufficient statistics $T$. The orthogonal functional components remain
free to enforce the exact constraints. Substituting
\eqref{eq:exact-multipliers-ratio} into
\eqref{eq:exact-multipliers-projected-rho} and using
$\Pi^{\lambda}[\log q_\lambda]=\vlambda$ yields
\begin{equation}
    \vlambda
    =
    \Pi^{\lambda}[\log f_{\mathrm c}]
    +
    \Pi^{\lambda}[\log f_{\mathrm n}].
    \label{eq:exact-multipliers-exact-projected-stationarity}
\end{equation}
This is a stationary condition on the constrained marginal; it is not the
density equality $q_\lambda(z)\propto f_{\mathrm c}(z)f_{\mathrm n}(z)$.

This also fixes the terminology. The model factor $f_a$ is not a Lagrange
multiplier, and the exact multiplier $\rho_a$ is not the finite message passed
by the algorithm. The ordinary exact BP log-message from a unary factor is
\begin{equation}
    \ell_a(z)=\log f_a(z)+\mathrm{const}.
    \label{eq:exact-multipliers-unary-bp-message}
\end{equation}
Consequently, \eqref{eq:exact-multipliers-exact-projected-stationarity} is the
unary-factor specialization of \eqref{eq:ng-message-eta} and
\eqref{eq:ng-main-theorem}: each mean-coordinate gradient is the natural
parameter of its outgoing message.
For a non-conjugate factor, this exact log-message is itself non-conjugate. The
exact constrained Lagrange problem remains well-defined because the multiplier
in \eqref{eq:exact-multipliers-ratio} can cancel the non-conjugate part inside
the unrestricted factor belief. A finite message-passing algorithm cannot
communicate that arbitrary residual unless we replace the factor with a surrogate.

\subsection{Projected Surrogate Update}
\label{app:projected-surrogate-model}

The point of the surrogate construction is narrower than that of exact inference in a
new model. For an arbitrary reference value $\vlambda_0$, the surrogate below
does not claim to have the same marginal as the original constrained problem.
Rather, it provides a finite-dimensional update map whose fixed points coincide
with the stationary points of the original constrained problem. The objective to
prove is therefore the fixed-point equivalence
\begin{equation}
    \vlambda=\Phi(\vlambda)
    \quad\Longleftrightarrow\quad
    \nabla_{\mu}\mathcal F_{\mathrm{exact}}(\vlambda)=0,
    \label{eq:surrogate-fixed-point-claim}
\end{equation}
not equality between the one-step surrogate marginal and the original
constrained optimum for every choice of $\vlambda_0$.

The finite-dimensional update is obtained by replacing the original product
graph, locally and temporarily, with a conjugate surrogate product graph.
Continue with the carrier-measure convention used above: write all densities
with respect to $\mathrm{d}\nu(z)=h(z)\,\mathrm{d}z$. Thus, the
exponential-family marginal has a log density
$\vlambda^{\top}T(z)-A(\vlambda)$ in this reference measure. Restoring
Lebesgue densities simply multiplies the final marginal by $h(z)$.

At a reference marginal $q_{\lambda_0}$, define the projected natural
contribution of factor $a$ by
\begin{equation}
    \veta_a(\vlambda_0)
    =
    \Pi^{\lambda_0}[\ell_a]
    =
    \nabla_{\mu}\,
    \mathbb E_{q_{\lambda}}[\ell_a(z)]
    \big|_{\lambda=\lambda_0}.
    \label{eq:exact-multipliers-projected-eta}
\end{equation}
The surrogate graph has the same topology as the original two-factor graph, but
its unary factors are conjugate sites
\begin{equation}
    \widehat m_a(z;\vlambda_0)
    \propto
    \exp\{\veta_a(\vlambda_0)^{\top}T(z)\},
    \qquad a\in\{\mathrm c,\mathrm n\}.
    \label{eq:exact-multipliers-projected-message}
\end{equation}
When $f_a$ is conjugate, this projection returns the usual natural-parameter
contribution of the factor. When $f_a$ is non-conjugate, it returns only the
component of $\ell_a$ visible to the sufficient statistics $T$ under the
reference marginal $q_{\lambda_0}$.

Now solve the exact constrained Bethe problem for this new surrogate graph.
Introduce unrestricted surrogate factor beliefs $\widehat q_{\mathrm c}$ and
$\widehat q_{\mathrm n}$, keep the edge belief constrained to
$q_\lambda\in\mathcal E$, and define
\begin{equation}
    \widehat{\mathcal F}_{\lambda_0}
    (\widehat q_{\mathrm c},\widehat q_{\mathrm n},q_\lambda)
    =
    \sum_{a\in\{\mathrm c,\mathrm n\}}
    \int \widehat q_a(z)
    \log\frac{\widehat q_a(z)}{\widehat m_a(z;\vlambda_0)}
    \,\mathrm{d}\nu(z)
    -
    \int q_\lambda(z)\log q_\lambda(z)\,\mathrm{d}\nu(z).
    \label{eq:surrogate-bethe-objective}
\end{equation}
The local constraints are
\begin{equation}
    \widehat q_{\mathrm c}(z)=q_\lambda(z),
    \qquad
    \widehat q_{\mathrm n}(z)=q_\lambda(z).
    \label{eq:surrogate-constraints}
\end{equation}
With functional multipliers $\widehat\rho_{\mathrm c}$ and
$\widehat\rho_{\mathrm n}$, the surrogate Lagrangian is
\begin{align}
    \widehat{\mathcal L}_{\lambda_0}
    &=
    \widehat{\mathcal F}_{\lambda_0}
    +
    \sum_{a\in\{\mathrm c,\mathrm n\}}
    \int \widehat\rho_a(z)
    \{q_\lambda(z)-\widehat q_a(z)\}\,\mathrm{d}\nu(z)
    \nonumber\\
    &\quad
    +
    \sum_{a\in\{\mathrm c,\mathrm n\}}
    \widehat\tau_a
    \left(\int \widehat q_a(z)\,\mathrm{d}\nu(z)-1\right)
    +
    \widehat\tau
    \left(\int q_\lambda(z)\,\mathrm{d}\nu(z)-1\right).
    \label{eq:surrogate-lagrangian}
\end{align}
Taking the variational derivative with respect to $\widehat q_a$ gives
\begin{equation}
    \frac{\delta\widehat{\mathcal L}_{\lambda_0}}
    {\delta \widehat q_a(z)}
    =
    \log\frac{\widehat q_a(z)}{\widehat m_a(z;\vlambda_0)}
    +1-\widehat\rho_a(z)+\widehat\tau_a.
    \label{eq:surrogate-factor-derivative}
\end{equation}
Setting \eqref{eq:surrogate-factor-derivative} to zero yields
\begin{equation}
    \widehat q_a(z)
    \;\propto\;
    \widehat m_a(z;\vlambda_0)\exp\{\widehat\rho_a(z)\}.
    \label{eq:surrogate-factor-stationarity}
\end{equation}
Enforcing the marginal constraints therefore identifies the exact surrogate
multipliers as
\begin{equation}
    \widehat\rho_a(z)
    =
    \log q_\lambda(z)-\log\widehat m_a(z;\vlambda_0)+\mathrm{const}.
    \label{eq:surrogate-multiplier-ratio}
\end{equation}
Unlike \eqref{eq:exact-multipliers-ratio}, these multipliers are finite
exponential-family objects. Indeed, using
\eqref{eq:exact-multipliers-projected-message},
\begin{equation}
    \widehat\rho_a(z)
    =
    \big(\vlambda-\veta_a(\vlambda_0)\big)^{\top}T(z)+\mathrm{const}.
    \label{eq:surrogate-multiplier-affine}
\end{equation}

It remains to take stationarity with respect to the constrained marginal
parameter $\vlambda$. Using
$\nabla_{\lambda}A^{*}(\vmu(\vlambda))=\fisher(\vlambda)\vlambda$ and
$\nabla_{\lambda}\mathbb E_{q_\lambda}[r]
=\operatorname{Cov}_{q_\lambda}[T,r]$ for fixed $r$, the derivative of
\eqref{eq:surrogate-lagrangian} with respect to $\vlambda$ is
\begin{equation}
    0
    =
    -\fisher(\vlambda)\vlambda
    +
    \sum_{a\in\{\mathrm c,\mathrm n\}}
    \operatorname{Cov}_{q_\lambda}\!\big[T,\widehat\rho_a\big].
    \label{eq:surrogate-lambda-derivative}
\end{equation}
Premultiplying by $\fisher(\vlambda)^{-1}$ gives
\begin{equation}
    \vlambda
    =
    \Pi^{\lambda}[\widehat\rho_{\mathrm c}]
    +
    \Pi^{\lambda}[\widehat\rho_{\mathrm n}].
    \label{eq:surrogate-projected-multiplier-stationarity}
\end{equation}
Substituting \eqref{eq:surrogate-multiplier-ratio}, and using
$\Pi^{\lambda}[\log q_\lambda]=\vlambda$ together with
$\Pi^{\lambda}[\log\widehat m_a(\cdot;\vlambda_0)]
=\veta_a(\vlambda_0)$, yields
\begin{equation}
    \vlambda^{\star}(\vlambda_0)
    =
    \veta_{\mathrm c}(\vlambda_0)
    +
    \veta_{\mathrm n}(\vlambda_0).
    \label{eq:surrogate-inner-stationarity}
\end{equation}
Thus the exact solution of the surrogate constrained problem has natural
parameter $\vlambda^{\star}(\vlambda_0)$ and density
\begin{equation}
    q_{\lambda^{\star}(\lambda_0)}(z)
    =
    \exp\!\left\{
        \big(\veta_{\mathrm c}(\vlambda_0)
        +\veta_{\mathrm n}(\vlambda_0)\big)^{\top}T(z)
        -
        A\big(\veta_{\mathrm c}(\vlambda_0)
        +\veta_{\mathrm n}(\vlambda_0)\big)
    \right\}
    \quad\text{w.r.t. }\nu.
    \label{eq:surrogate-exact-marginal}
\end{equation}
Equivalently, with respect to Lebesgue measure the same density is multiplied
by $h(z)$.

Equation \eqref{eq:surrogate-inner-stationarity} is the inner exact
message-passing solution on the surrogate graph. There is no additional
Lagrange multiplier enforcing $\vlambda_0=\vlambda$ in
\eqref{eq:surrogate-lagrangian}. The parameter $\vlambda_0$ defines the
surrogate factors and is held fixed while the inner constrained problem is
solved. In other words, the inner problem defines an update map
\begin{equation}
    \Phi(\vlambda_0)
    \coloneqq
    \vlambda^{\star}(\vlambda_0)
    =
    \veta_{\mathrm c}(\vlambda_0)
    +
    \veta_{\mathrm n}(\vlambda_0).
    \label{eq:surrogate-update-map}
\end{equation}
The projected-message fixed point is the self-consistency condition that the
reference marginal used to build the surrogate is also the marginal returned by
that surrogate, that is, $\Phi(\vlambda)=\vlambda$:
\begin{equation}
    \vlambda
    =
    \veta_{\mathrm c}(\vlambda)
    +
    \veta_{\mathrm n}(\vlambda)
    =
    \Pi^{\lambda}[\log f_{\mathrm c}]
    +
    \Pi^{\lambda}[\log f_{\mathrm n}].
    \label{eq:exact-multipliers-product-projected-stationarity}
\end{equation}
One could introduce a separate copy of the reference parameter and a constraint
forcing that copy to equal $\vlambda$, but this would either be redundant or
 change the problem. It is redundant if the surrogate sites are already
frozen before the inner optimization. It changes the stationarity equations if
one differentiates through the dependence
$\veta_a(\vlambda_0)=\Pi^{\lambda_0}[\ell_a]$ because then extra
$\partial\veta_a/\partial\vlambda_0$ terms enter. The message-passing algorithm
uses the first interpretation: build a local surrogate at the current marginal,
solve the surrogate problem exactly, and seek a fixed point of the resulting
map.

This fixed-point condition is exactly the first-order condition of the original
constrained problem, not a new exactness claim for every intermediate
surrogate. Indeed, after enforcing the exact constraints
$q_{\mathrm c}=q_{\mathrm n}=q_\lambda$, the original constrained objective
reduces, up to constants, to
\begin{equation}
    \mathcal F_{\mathrm{exact}}(\vlambda)
    =
    A^{*}(\vmu(\vlambda))
    -
    \mathbb E_{q_\lambda}
    [\log f_{\mathrm c}(z)+\log f_{\mathrm n}(z)] .
    \label{eq:surrogate-exact-reduced-objective}
\end{equation}
Its derivative in mean coordinates is
\begin{equation}
    \nabla_{\mu}\mathcal F_{\mathrm{exact}}
    =
    \vlambda
    -
    \Pi^{\lambda}[\log f_{\mathrm c}]
    -
    \Pi^{\lambda}[\log f_{\mathrm n}]
    =
    \vlambda-\Phi(\vlambda).
    \label{eq:surrogate-exact-gradient-map}
\end{equation}
Thus, the outer iteration
$\vlambda^{t+1}=\Phi(\vlambda^t)$ is a fixed-point method, equivalently a
unit natural-gradient step, for solving the exact constrained stationarity
condition. A converged fixed point is therefore a stationary marginal of the
original constrained problem. This does not imply that the intermediate
surrogate marginals are exact marginals of the original model, nor that the
fixed point is the global minimizer when the reduced objective is non-convex.

At such a fixed point, the surrogate multipliers are explicit
\begin{equation}
    \widehat\rho_{\mathrm c}(z)
    =
    \veta_{\mathrm n}(\vlambda)^{\top}T(z)+\mathrm{const},
    \qquad
    \widehat\rho_{\mathrm n}(z)
    =
    \veta_{\mathrm c}(\vlambda)^{\top}T(z)+\mathrm{const}.
    \label{eq:surrogate-fixed-point-multipliers}
\end{equation}
This is the payoff of the surrogate construction. At a fixed point, the
surrogate graph has the same constrained marginal as the stationary point of
the original constrained problem, but its half-edge multipliers are now
finite-dimensional exponential-family objects. Equivalently, the exponentiated
multipliers are computable messages in the chosen sufficient statistics rather
than arbitrary functional objects. This matters in message passing because a
local solution is not enough: a factor must also propagate a message to the
rest of the graph. The exact constrained problem can satisfy marginalization by
using the functional ratios in \eqref{eq:exact-multipliers-ratio}, but those
ratios are not generally messages that a finite-dimensional algorithm can pass.
The surrogate replaces that functional object with a conjugate site whose exact
Lagrangian multipliers lie in the same finite family as the constrained
marginal. In Gaussian-chain examples with non-conjugate observation factors,
this is the mechanism that turns a non-Gaussian likelihood contribution into a
Normal site message while preserving the fixed-point marginal condition.

Thus \eqref{eq:surrogate-fixed-point-multipliers} is the exact
marginalization condition for the surrogate graph, not a claim that the
original non-conjugate factor has become conjugate. The exact Lagrange
multipliers of the original constrained problem remain the functional ratios in
\eqref{eq:exact-multipliers-ratio}. The finite message in
\eqref{eq:exact-multipliers-projected-message} is a different object: it is the
conjugate surrogate factor whose constrained Bethe Lagrangian has
finite-dimensional multipliers and an exact exponential-family marginal.

This is also the convention used by conjugate-computation methods. In that
literature, the word ``message'' refers to the conjugate surrogate obtained
from a mean-parameter or natural-gradient projection
\citep{khan_conjugatecomputation_2017,lin2018variational}. The point of the
local information-geometric derivation is to obtain the same surrogate message
from the edge-local stationary condition while keeping separate the exact functional
multipliers that solve the constrained Lagrangian in the background.

\subsection{Computing the Tangent Projection in Practice}
\label{app:tangent-projection-numerics}

The projection
$\Pi^{\lambda}[\ell]=\nabla_{\mu}\mathbb E_{q_\lambda}[\ell(z)]$ of
\eqref{eq:exact-multipliers-projected-eta} reduces, via
$\nabla_{\lambda}\mathbb E_{q_\lambda}[\ell]
=\operatorname{Cov}_{q_\lambda}[T,\ell]$ and preconditioning by
$\fisher(\vlambda)^{-1}$, to covariances of the sufficient statistics with
the log-message under the current marginal. Whether the update is
closed-form therefore depends only on whether these expectations are
available analytically for the factor and family at hand. The ablations of
\Cref{sec:edge-uncertainty} cover the three cases that occur in practice.

\paragraph{Closed form: Poisson state-space model.}
Toward a Gaussian state edge, the observation log-factor is
$\ell(z)=yz-e^{z}$ up to constants, and Gaussian expectations of $z$,
$z^{2}$, and $e^{z}$—hence all required covariances—are analytic.
NGMP's projection is a single closed-form natural-gradient step per edge,
and the PVMP baseline iterates the same closed-form gradients
inside its inner manifold optimizer, so both inference methods are
closed-form in this model. The budget-matched NCVMP control of
\Cref{app:edge-uncertainty-fe} uses the same closed-form gradient but takes
a single inner step.

\paragraph{Unscented approximation: Normal mean--precision model.}
The exact cavity log-message toward the precision edge is not a linear
combination of the Gamma sufficient statistics $(\log\tau,\tau)$, and its
Gamma expectations have no closed form. The implementation evaluates the
tangent projection with a deterministic unscented (sigma-point)
approximation of the required expectations under the current marginal.
VMP, by contrast, needs no projection in this model: the mean-field
coordinate updates are conjugate and closed-form.

\paragraph{Quadrature: sequential heteroskedastic model.}
The cavity messages through the likelihood factor of
\eqref{eq:edge-uncertainty-heteroscedastic} have no closed-form expectations under the
joint $(\mu,s)$ cluster marginal; the implementation evaluates the tangent
projection by one-dimensional numerical quadrature with $32$ nodes over
the log-precision marginal, while the PVMP baseline again
iterates closed-form gradients of its mean-field sites, and the NCVMP
control again takes a single such step.

In every case, the projection consumes only expectations under the current
marginal, so the accuracy of the numerical variants is governed by the
chosen sigma-point or quadrature rule rather than by sampling noise.

\section{Damping and Momentum}
\label{app:momentum-damping}
The surrogate construction of \Cref{sec:surrogate-models} defines the outer fixed-point map
$\Phi(\vlambda)$ in \eqref{eq:surrogate-fixed-point-map}: build projected
surrogate leaves at the current constrained marginals, freeze the surrogate
graph, and run the conjugate BP sweep.

In mean coordinates, the exact reduced constrained-Bethe objective has the same
stationarity form as the global exponential-family free energy:
\begin{equation}
    \nabla_{\mu} F_{\mathrm{exact}}(\vlambda)
    =
    \vlambda-\Phi(\vlambda).
    \label{eq:momentum-fixed-point-gradient}
\end{equation}
Thus, the undamped update $\vlambda^{(t+1)}=\Phi(\vlambda^{(t)})$ is a unit
natural-gradient or fixed-point step. Damping changes only the step length in
natural coordinates,
\begin{equation}
    \vlambda^{(t+1)}
    =
    (1-\alpha)\vlambda^{(t)}
    +\alpha\,\Phi(\vlambda^{(t)}),
    \qquad
    0<\alpha\leq 1.
    \label{eq:momentum-natural-damping}
\end{equation}
Equivalently, for an individual surrogate message with natural parameter
$\veta^{\star}$ computed by projection, the damped message parameter is
\begin{equation}
    \veta^{(t)}
    =
    (1-\alpha)\veta^{(t-1)}
    +\alpha\,\veta^{\star}.
    \label{eq:interp-natural-damping}
\end{equation}
For the Gaussian Poisson leaf, this means averaging the canonical coordinates
$(\xi,\tau)$ before converting back to pseudo-observation parameters
$(\widetilde y,\Lambda)$. In message-function form
\begin{equation}
    \widehat{\mu}^{(t)}
    \;\propto\;
    \left(\widehat{\mu}^{(t-1)}\right)^{1-\alpha}
    \left(\widehat{\mu}^{\star}\right)^{\alpha}.
    \label{eq:interp-message-geometric-damping}
\end{equation}

Whether damping is needed depends on the data. On the sunspot masks of
\Cref{sec:edge-uncertainty}, the undamped iteration ($\alpha=1$), natural
damping with $\alpha=0.25$, and the damped heavy-ball setting used for the
sunspot results ($\alpha=0.5$, $\beta=0.2$; \eqref{eq:momentum-heavy-ball}
below) all converge to the same fixed point by iterations 5, 14, and 8
respectively: the counts are moderate, and the initialization
$\log(y_k+1)$ is close to the fixed point, so damping only costs iterations.
The fixed-point map stops contracting when the latent log-rate makes long
excursions to extreme values; for example, a stretch of zero counts at a very
negative log-rate, where the natural-gradient message reacts strongly to the
marginal variance. \Cref{fig:poisson-bethe-damping-progress} uses data simulated
from the Poisson state-space model \eqref{eq:interp-poisson-chain} itself
($z_0=0$, $\sigma^2=0.1$ as in the sunspot experiments, every count
observed): the latent random walk makes larger excursions and produces
longer stretches of zero counts as the chain length $N$ grows. At $N=100$
, every setting converges. From
$N=250$ on, the undamped map settles into a period-two oscillation on a
growing share of seeds (3, 9, and 9 of 20 for $N=250$, $500$, $1000$) and at
$N=1000$ diverges on three of them; damping with $\alpha=0.25$ converges on
all 80 fits at the price of more iterations; the heavy-ball setting is the
fastest whenever it converges, but its momentum step can produce a message
with negative precision, which happens on 3 of 20 seeds at $N=500$ and 9 of
20 at $N=1000$ (see \Cref{app:vector-transport-momentum}). These are the
conditions under which damping, and in harder cases momentum, are necessary,
and why they are part of the practical outer solver rather than a cosmetic
post-processing step.

\begin{figure}[tbh]
    \centering
    \includegraphics[width=0.42\linewidth]{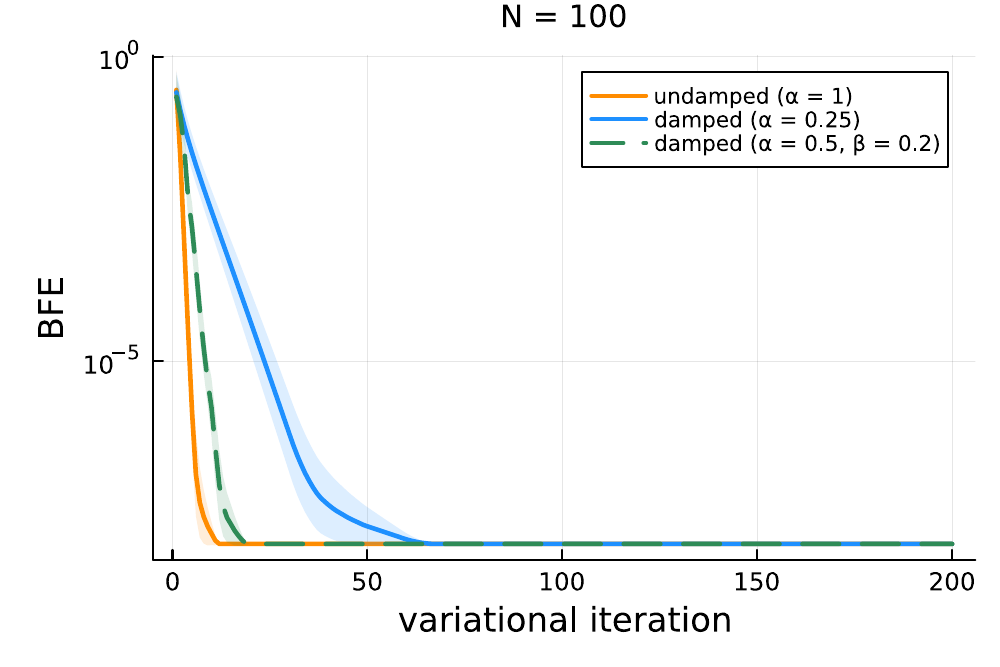}
    \hspace{2pt}
    \includegraphics[width=0.42\linewidth]{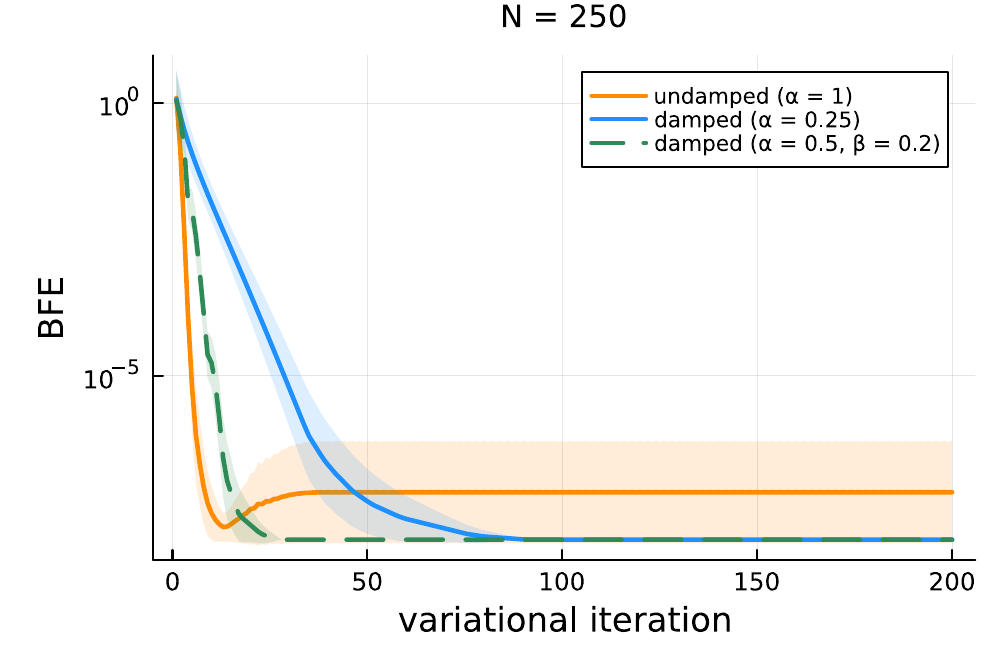}\\[2pt]
    \includegraphics[width=0.42\linewidth]{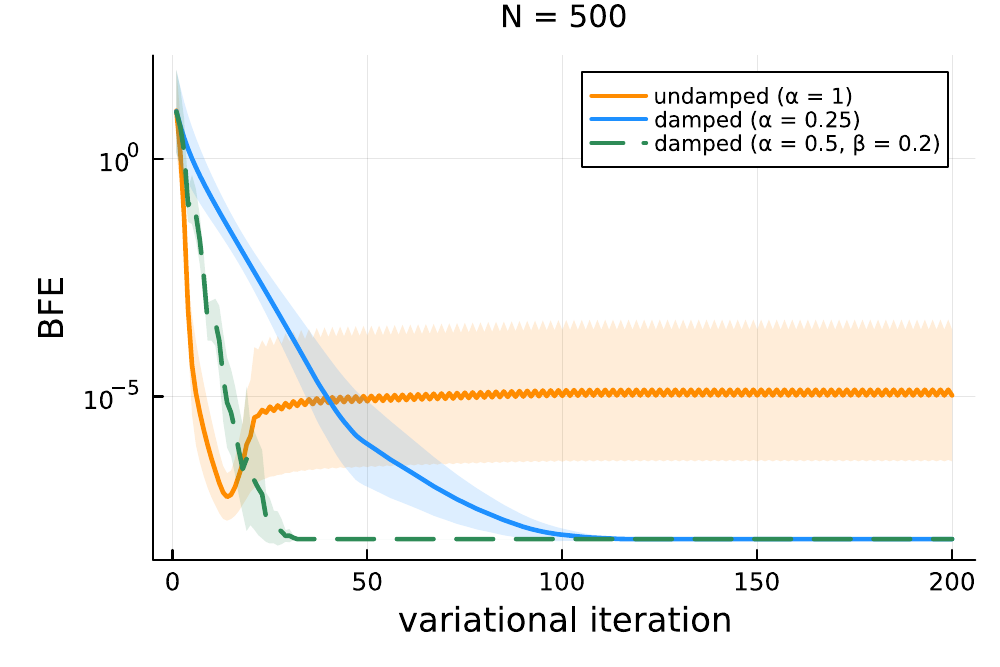}
    \hspace{2pt}
    \includegraphics[width=0.42\linewidth]{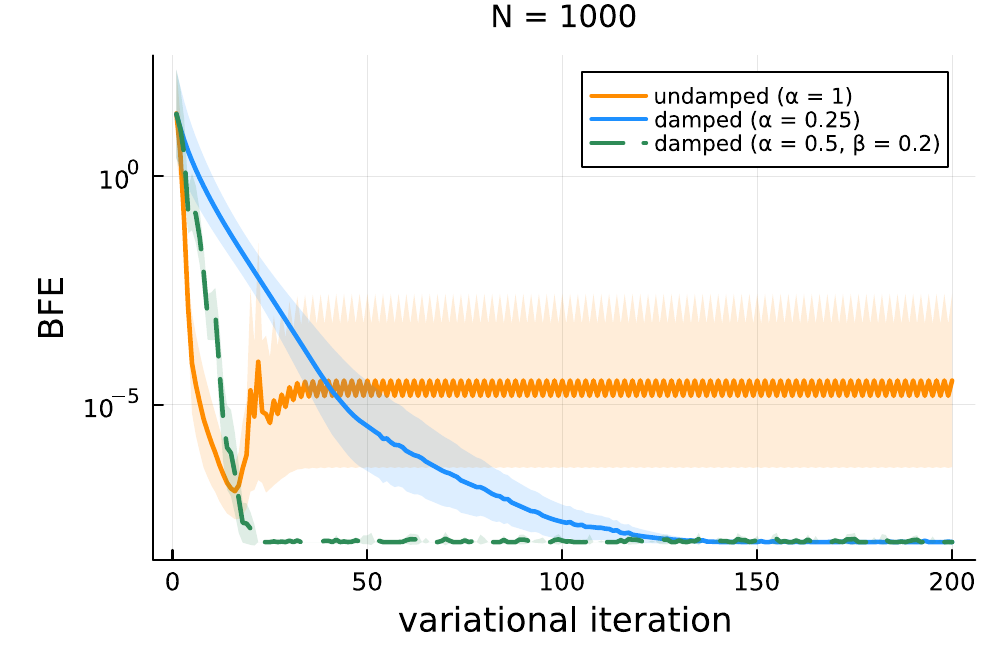}
    \caption{Convergence of the NGMP fixed-point iteration on the Poisson
    state-space model \eqref{eq:interp-poisson-chain} for series simulated
    from the same model ($\sigma^2=0.1$, all counts observed) of length
    $N=100$, $250$ (top row), $500$, and $1000$ (bottom row), 20 seeds each,
    200 variational iterations. BFE on the vertical axis is the Bethe free
    energy of the model per observation minus the value reached at
    convergence by the $\alpha=0.25$ setting on the same seed, so every seed
    converges to zero (log scale); curves are the geometric mean across seeds
    with a 95\% confidence band. Orange: undamped
    fixed-point iteration ($\alpha=1$). Blue: messages damped in natural
    coordinates with $\alpha=0.25$. Green dashed: the damped heavy-ball
    setting used for the sunspot experiments ($\alpha=0.5$, $\beta=0.2$); its
    momentum step left the natural domain and the fit stopped on 3 of 20
    seeds at $N=500$ and 9 of 20 at $N=1000$ (the undamped iteration failed
    on one seed at $N=1000$), and the curves average over the seeds that
    reached each iteration. Longer chains make larger excursions of the latent log-rate; the
    undamped map then stops converging while the damped map still does.}
    \label{fig:poisson-bethe-damping-progress}
\end{figure}

Momentum can be applied to the same natural-parameter residual
$\Phi(\vlambda)-\vlambda=-\nabla_{\mu}F_{\mathrm{exact}}(\vlambda)$. A simple
heavy-ball form is
\begin{equation}
    \begin{aligned}
        \bm{v}^{(t+1)}
        &=
        \beta\bm{v}^{(t)}
        +\alpha\big(\Phi(\vlambda^{(t)})-\vlambda^{(t)}\big),\\
        \vlambda^{(t+1)}
        &=
        \vlambda^{(t)}+\bm{v}^{(t+1)},
    \end{aligned}
    \qquad
    0\leq\beta<1.
    \label{eq:momentum-heavy-ball}
\end{equation}
Other accelerated fixed-point variants differ in how they choose the point at
which $\Phi$ is evaluated, but the coordinate principle is the same: damping and
momentum operate on natural/message parameters, not on derived
pseudo-observation parameters. When the iteration converges, the residual
$\Phi(\vlambda)-\vlambda$ is zero, so damping and momentum change the
convergence path but not the stationary equations of
\Cref{thm:ngmp-edge}.

\subsection{Vector-transport momentum}
\label{app:vector-transport-momentum}

The heavy-ball rule \eqref{eq:momentum-heavy-ball} adds the previous momentum
vector to a residual computed at a different iterate, as if both lived in the
same vector space. An exponential family is a smooth manifold on which the
Fisher information defines a Riemannian metric. Therefore, the Euclidean rule
generalizes directly: the momentum vector $\bm{v}^{(t)}$ is a tangent vector at
the iterate where it was formed, and before it is combined with the new
residual, it is carried to the current iterate by a vector transport
$\mathcal T_{\cdot\leftarrow\cdot}$.
\begin{equation}
    \begin{aligned}
        \bm{v}^{(t+1)}
        &=
        \beta\,
        \mathcal T_{\vlambda^{(t)}\leftarrow\vlambda^{(t-1)}}
        \bm{v}^{(t)}
        +\alpha\big(\Phi(\vlambda^{(t)})-\vlambda^{(t)}\big),\\
        \vlambda^{(t+1)}
        &=
        \vlambda^{(t)}+\bm{v}^{(t+1)}.
    \end{aligned}
    \label{eq:momentum-vector-transport}
\end{equation}
A full account of retractions and vector transports is outside the scope of
this paper; we refer to \citet{absil_optimization_2008} for the theory and to
the momentum gradient rule of \texttt{Manopt.jl} \citep{bergmann_manoptjl_2022} for the
algorithmic template that \eqref{eq:momentum-vector-transport} follows.

The transport is approximated
elementwise: the previous momentum vector is rescaled by the square root of the
ratio of the diagonal Fisher metrics at the previous and current natural
parameters. The resulting step is capped at a maximum Euclidean norm; the
settings used in the regression experiments are reported in
\Cref{app:uci-details}. With $\beta>0$, the update is no longer a convex
combination of natural parameters and can leave the natural domain, for
example, a Gamma message with a non-positive rate, so on such edges, we use
damping alone ($\beta=0$).

The motivation for momentum, beyond the stability argument of
\Cref{fig:poisson-bethe-damping-progress}, comes from
\citet{tan_analytic_2025}, who shows that natural-gradient updates for
Gaussian variational approximations benefit substantially from momentum. The
same holds here: in the depth-three hierarchies of
\Cref{sec:experiments-regression}, damping alone stalls within the sweep
budget, and the vector-transport momentum update is what converges.

\section{Convergence Diagnostics and Additional Results for the Comparison Study}
\label{app:edge-uncertainty-fe}
This appendix collects the convergence diagnostics behind the fixed-point
claims of \Cref{sec:edge-uncertainty}, together with the full predictive
comparison for the sequential heteroscedastic study. In both non-conjugate ablations PVMP
evaluates its own mean-field variational objective, so its trace is a true
Bethe free energy on the fitted graph. NGMP rebuilds its local Gaussian
surrogates between outer sweeps, so its trace scores the node energies
through the current surrogate marginals: it is a convergence diagnostic
whose \emph{level} is not directly comparable to PVMP's, because the two
curves approximate different Bethe functionals. What the traces establish
is that every comparison in \Cref{sec:edge-uncertainty} probes fixed
points rather than truncated optimization.

\paragraph{Poisson state-space model.}
\Cref{fig:app-poisson-fe} shows the Bethe free energy per observed count on
the sunspot series at all four holdout fractions. Both methods reach their
plateaus well within the 20-sweep budget: NGMP by sweep three to four, PVMP
shortly after (at $50\%$ holdout in roughly ten). The held-out separation
at $50\%$ in \Cref{tab:edge-uncertainty-poisson} is therefore a property
of the fixed points, not of stopping early.

One configuration detail of the projective update of
\eqref{eq:projective-vmp-marginal} matters for this model. With the default
gradient-norm bound of its inner manifold optimizer, a marginal that
drifts far into the tail becomes \emph{trapped}: the projection returns
its own input, neither damping nor additional sweeps escape, and rare
masks blow up catastrophically — mask-average held-out negative
log-likelihood up to $66$, driven by a single month assigned a rate of
${\approx}10^{4}$ where the exact tilted update at the same state returns
${\approx}\log 114$. Widening the bound from $1$ to $100$ frees the trap
and is used throughout \Cref{sec:edge-uncertainty}. NGMP's projection on
these edges is closed-form and has no inner optimizer to tune.

\begin{figure}[t!]
    \centering
    \includegraphics[width=0.42\linewidth]{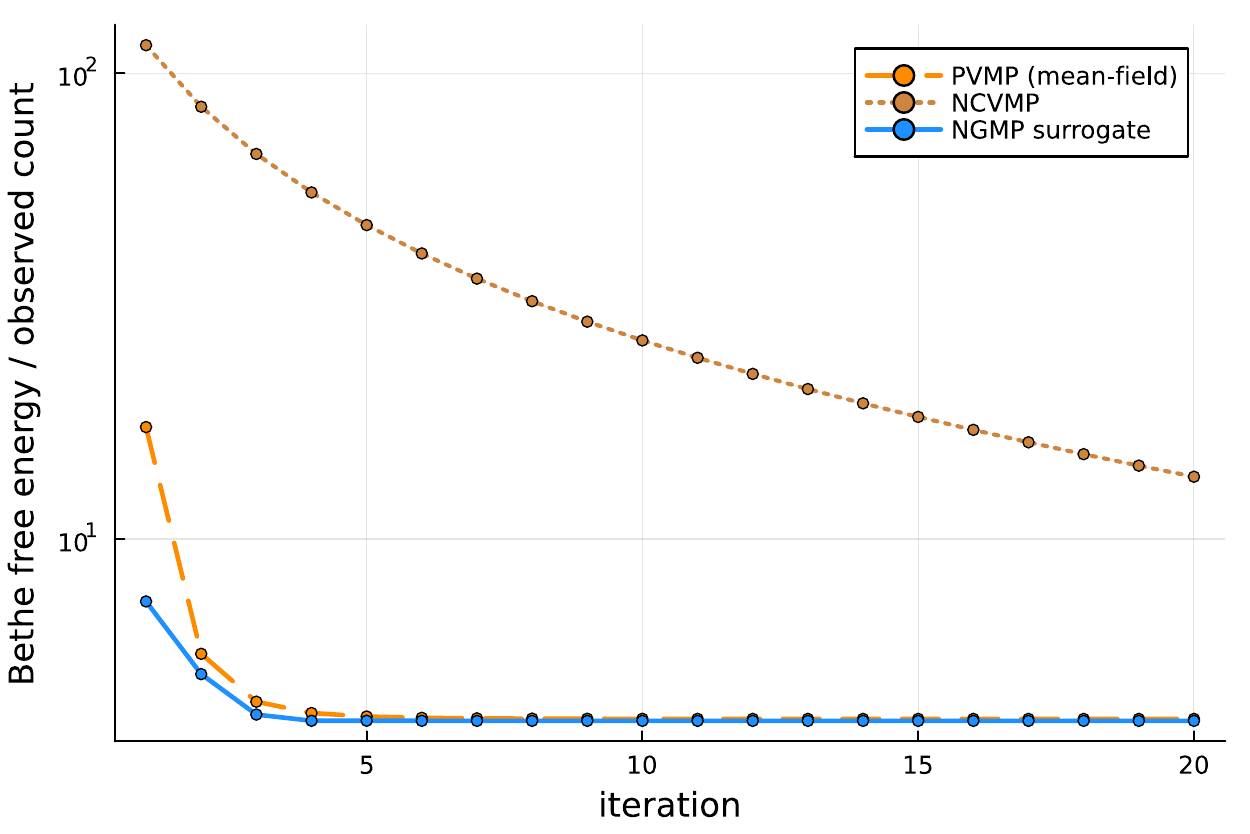}
    \hspace{2pt}
    \includegraphics[width=0.42\linewidth]{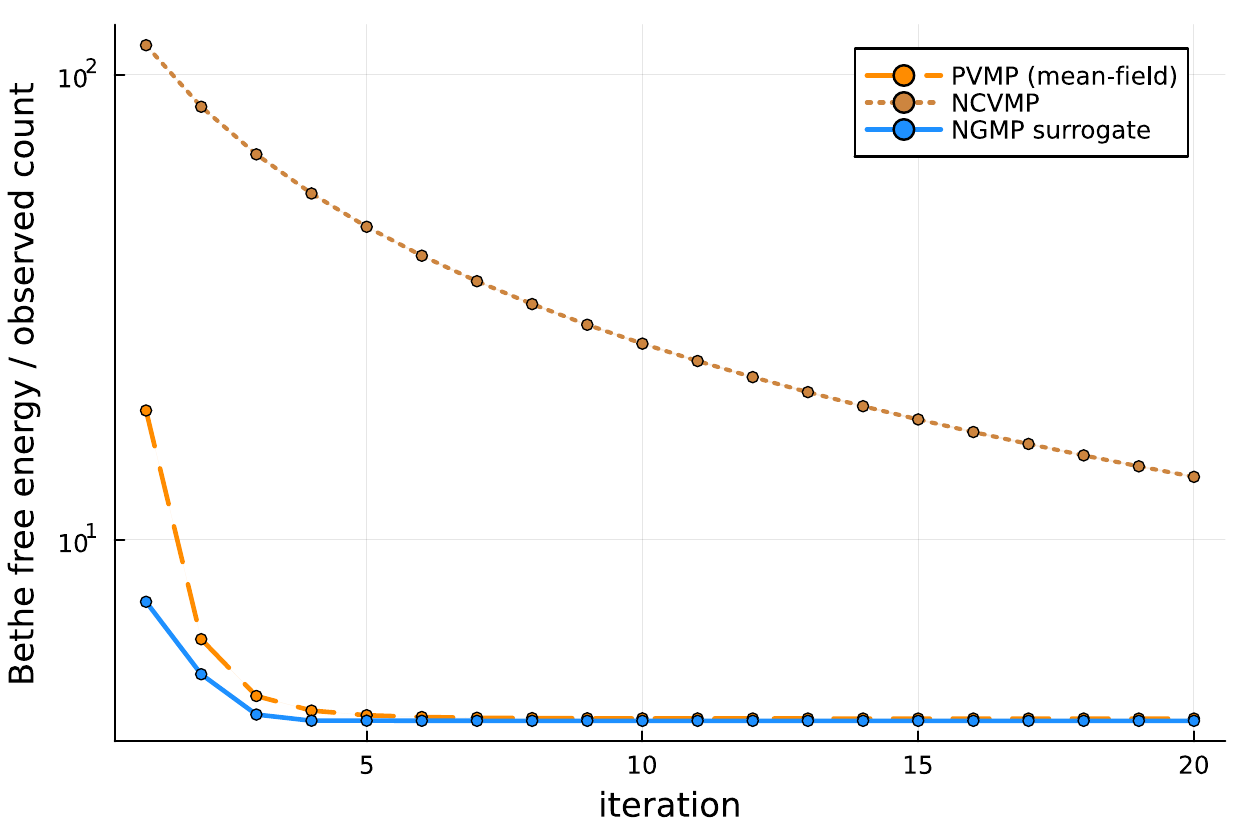}\\[2pt]
    \includegraphics[width=0.42\linewidth]{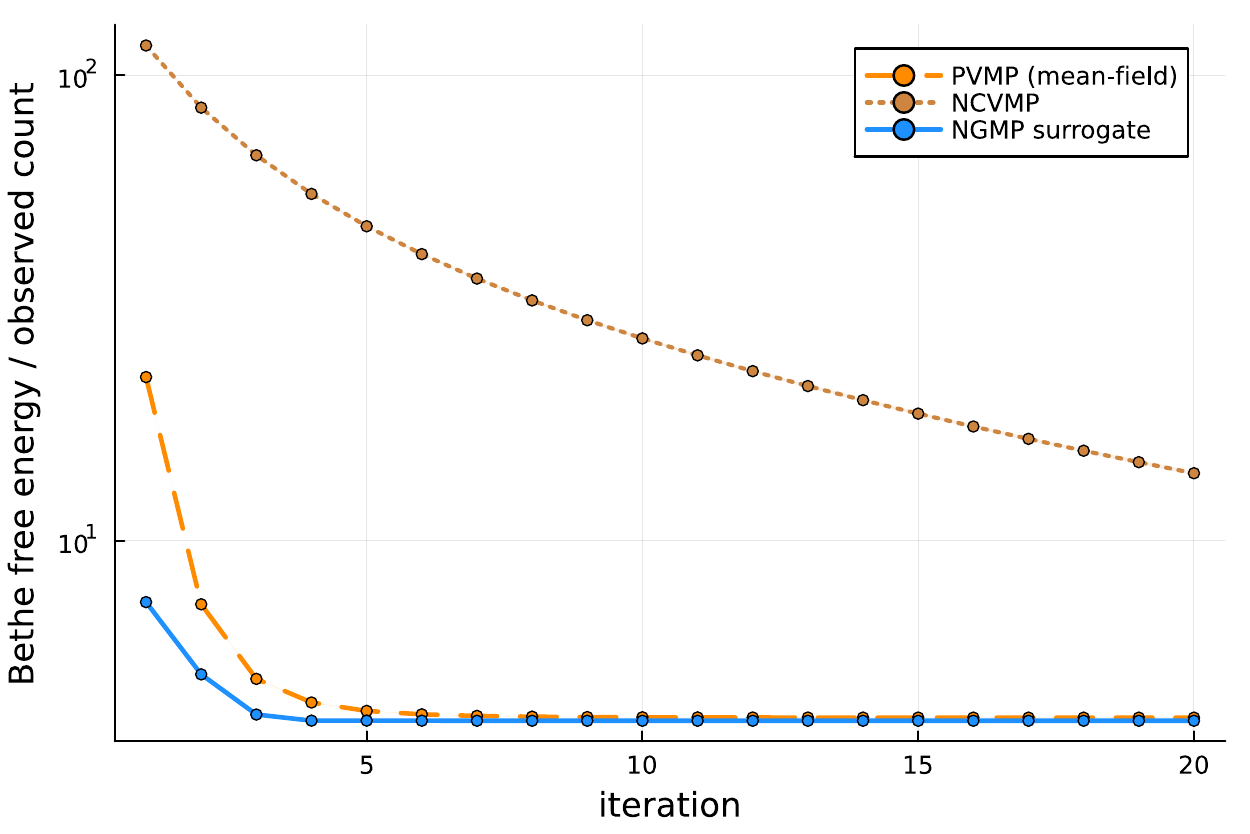}
    \hspace{2pt}
    \includegraphics[width=0.42\linewidth]{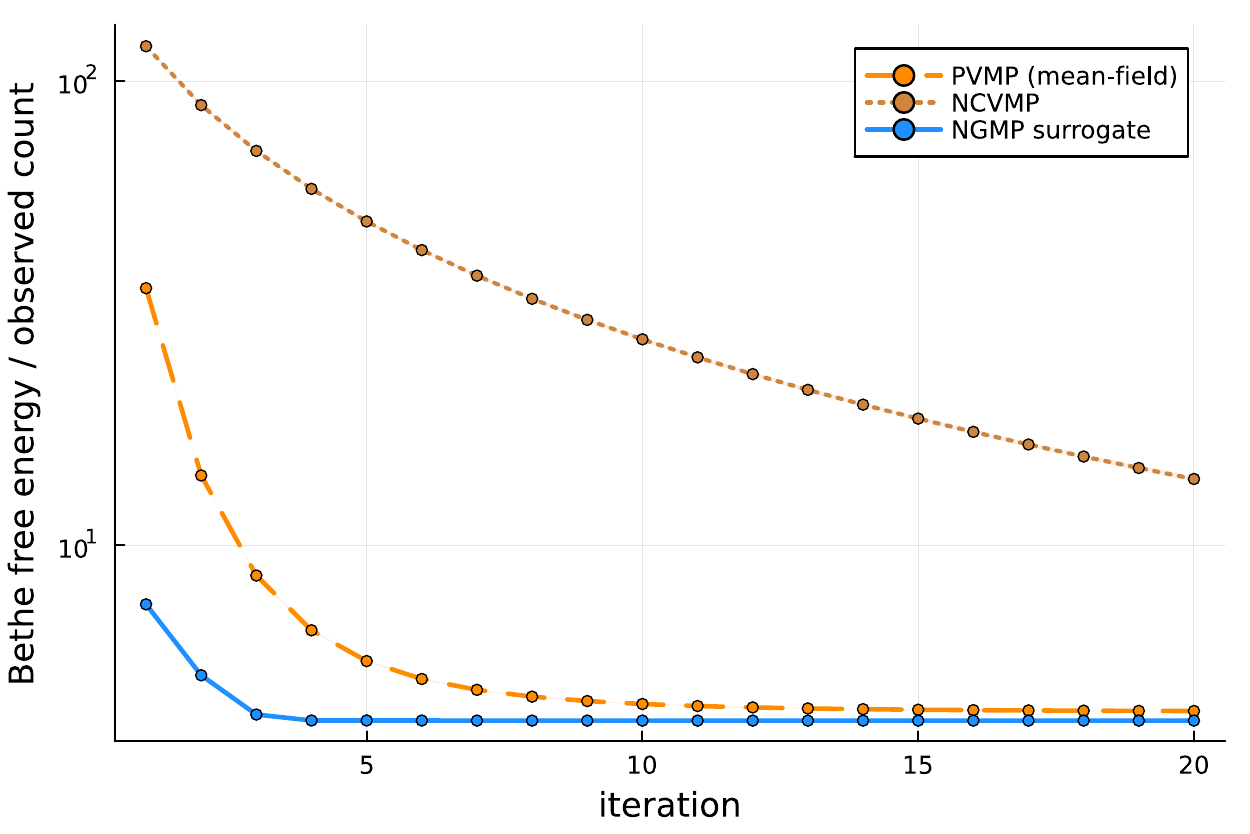}
    \caption{Bethe free energy per observed count on the sunspot series
    (mean $\pm$ 95\% CI across 20 masks, log scale) at $5\%$ and $10\%$
    holdout (top row) and $20\%$ and $50\%$ holdout (bottom row). Both
    methods plateau within the 20-sweep budget: NGMP by sweep three to four,
    PVMP shortly after. Each PVMP trace evaluates its own variational
    objective; the NGMP trace is a surrogate diagnostic, since the local
    Gaussian surrogates are rebuilt between outer sweeps. The dotted trace
    is the budget-matched NCVMP control of
    \Cref{tab:app-budget-poisson}.}
    \label{fig:app-poisson-fe}
\end{figure}

\paragraph{Sequential heteroscedastic filtering.}
\Cref{fig:app-streaming-fe} shows the corresponding traces for the
filtering study. At the end of the shared $240$-sweep budget, PVMP's Bethe
free energy decreases by $5.3\times10^{-4}$ nats per observation in the
final full-batch sweep and by $3.8\times10^{-4}$ in the first sequential
batch. The final-sweep changes in batches $2$--$10$ are below
$7\times10^{-8}$: the collapsed prior of
\Cref{fig:edge-uncertainty-collapse} leaves batches $2$--$10$ essentially
converged at initialization, so additional variational iteration cannot
change the filtering outcome. NGMP's surrogate diagnostic visibly plateaus;
its final-sweep changes range from $5.6\times10^{-6}$ to
$3.3\times10^{-5}$ nats per observation across the full and sequential
fits.

\Cref{fig:app-streaming-predictive} shows the posterior predictive bands of
all four fits on identical axes. The two full-batch fits and NGMP's
sequential fit are visually indistinguishable from one another; only PVMP's
sequential fit differs, which is why the main text shows the sequential
pair alone. \Cref{fig:app-streaming-variance} repeats the comparison for
the predictive-variance decomposition: NGMP's full-batch and sequential
fits coincide visually, so the main text omits the full-batch NGMP panel.

\begin{figure}[t!]
    \centering
    \includegraphics[width=0.48\linewidth]{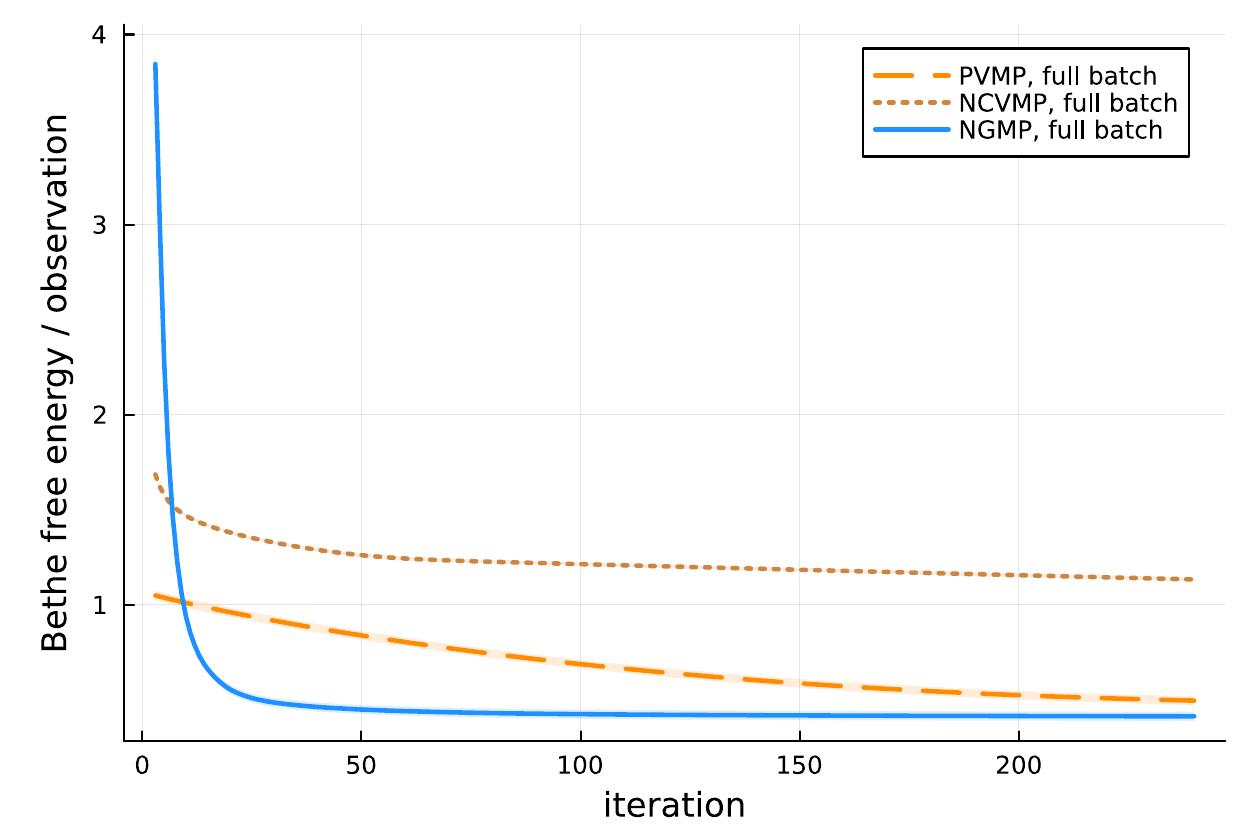}
    \hfill
    \includegraphics[width=0.48\linewidth]{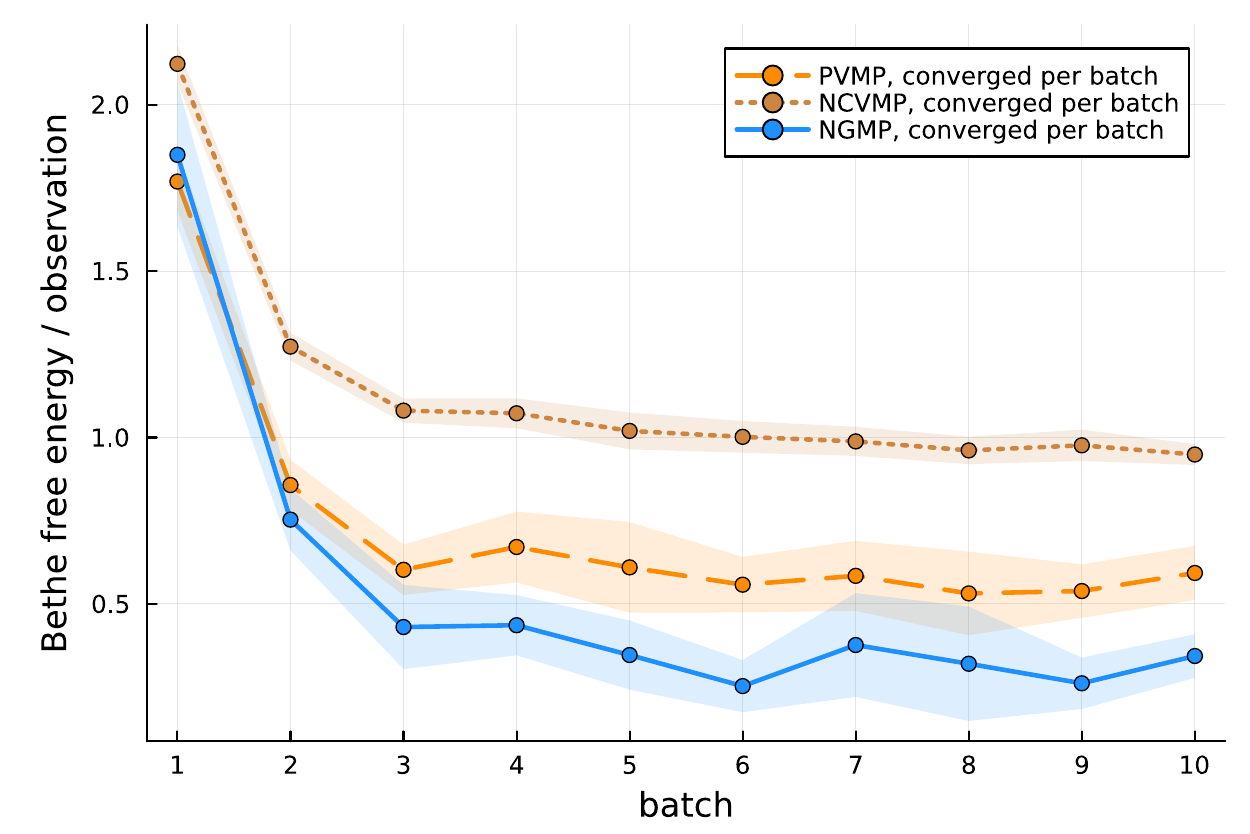}
    \caption{Bethe free energy per observation for the sequential
    heteroscedastic model (mean $\pm$ 95\% CI over 20 seeds). Left: the
    full-batch fit against sweep (first two sweeps omitted for scale). Right:
    the final-sweep value of each filtering update against batch index. PVMP
    evaluates its own mean-field variational objective; the NGMP trace
    scores the same node energies through moment-matched joint $(\mu, s)$
    cluster marginals, so the two curves approximate different Bethe
    functionals and their levels are not directly comparable. The dotted
    trace is the budget-matched NCVMP control of
    \Cref{tab:app-budget-streaming}.}
    \label{fig:app-streaming-fe}
\end{figure}

\begin{figure}[t!]
    \centering
    \includegraphics[width=0.62\linewidth]{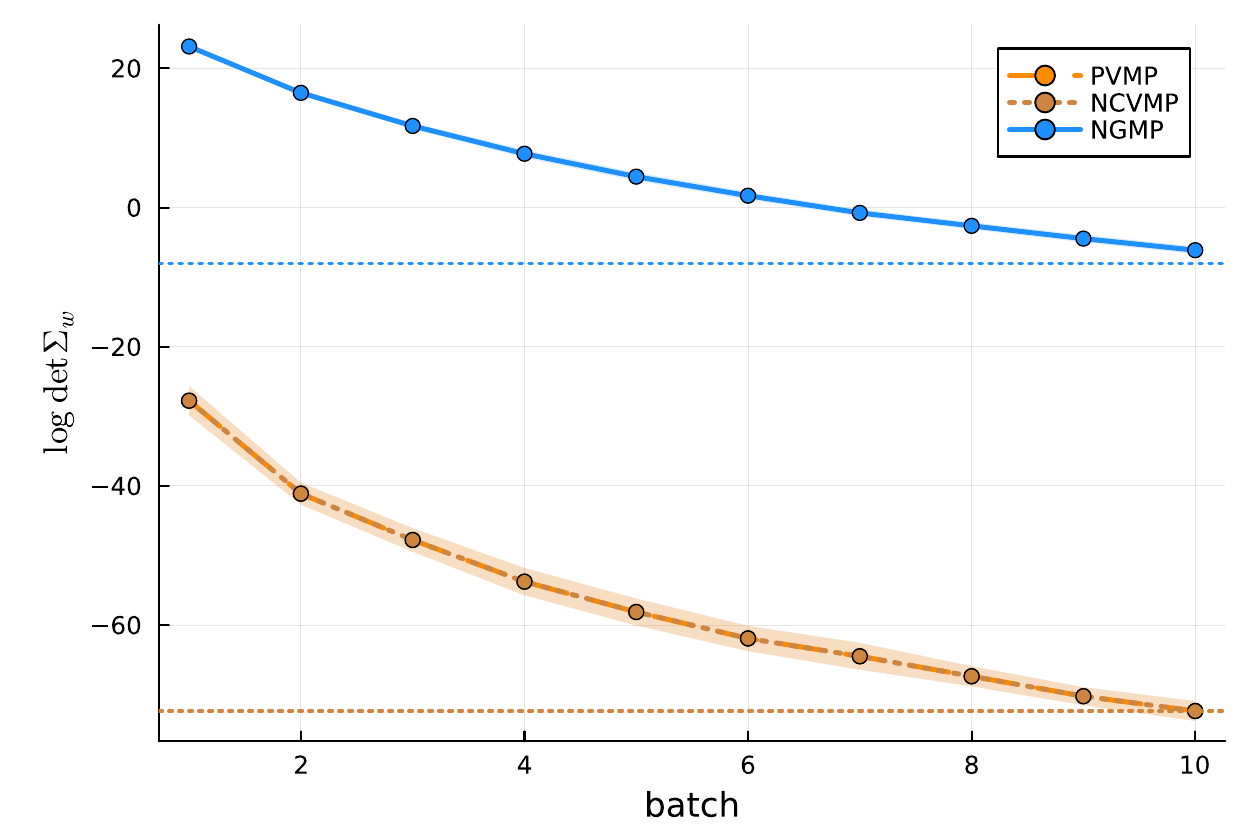}
    \caption{Noise-weight concentration $\log\det\Sigma_w$ along the
    filtering chain (mean $\pm$ 95\% CI over 20 seeds; dotted lines show
    full-batch values; lower means more certain). PVMP's sequential endpoint
    coincides with its full-batch value, and the budget-matched
    NCVMP control (\Cref{tab:app-budget-streaming}) traces
    the identical curve; NGMP concentrates as evidence accumulates and ends
    near its full-batch level.}
    \label{fig:edge-uncertainty-collapse}
\end{figure}

\begin{figure}[t!]
    \centering
    \begin{subfigure}[t]{0.49\linewidth}
        \centering
        \includegraphics[width=\linewidth]{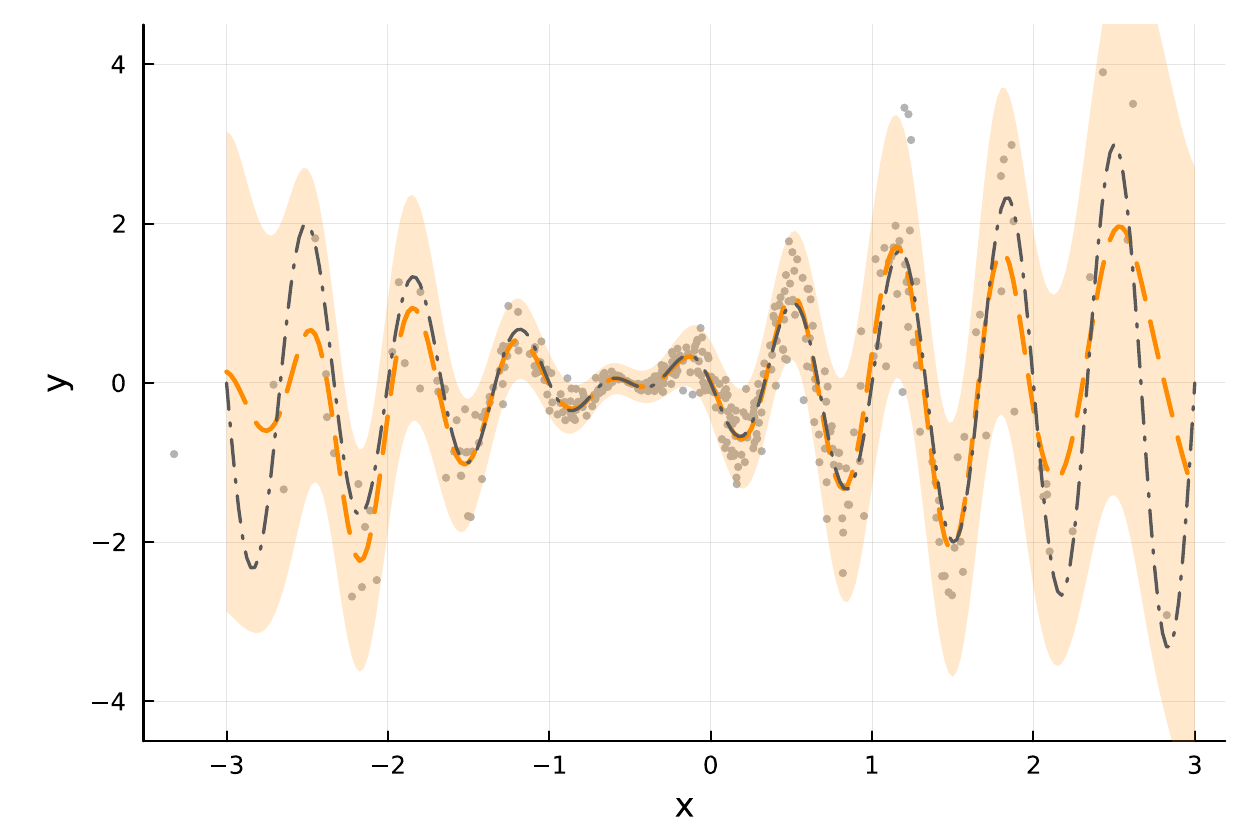}
        \caption{PVMP, full batch}
        \label{fig:app-streaming-vmp-full}
    \end{subfigure}
    \hfill
    \begin{subfigure}[t]{0.49\linewidth}
        \centering
        \includegraphics[width=\linewidth]{figures/streaming_vmp_sequential.pdf}
        \caption{PVMP, sequential batches}
        \label{fig:app-streaming-vmp-sequential}
    \end{subfigure}\\[2pt]
    \begin{subfigure}[t]{0.49\linewidth}
        \centering
        \includegraphics[width=\linewidth]{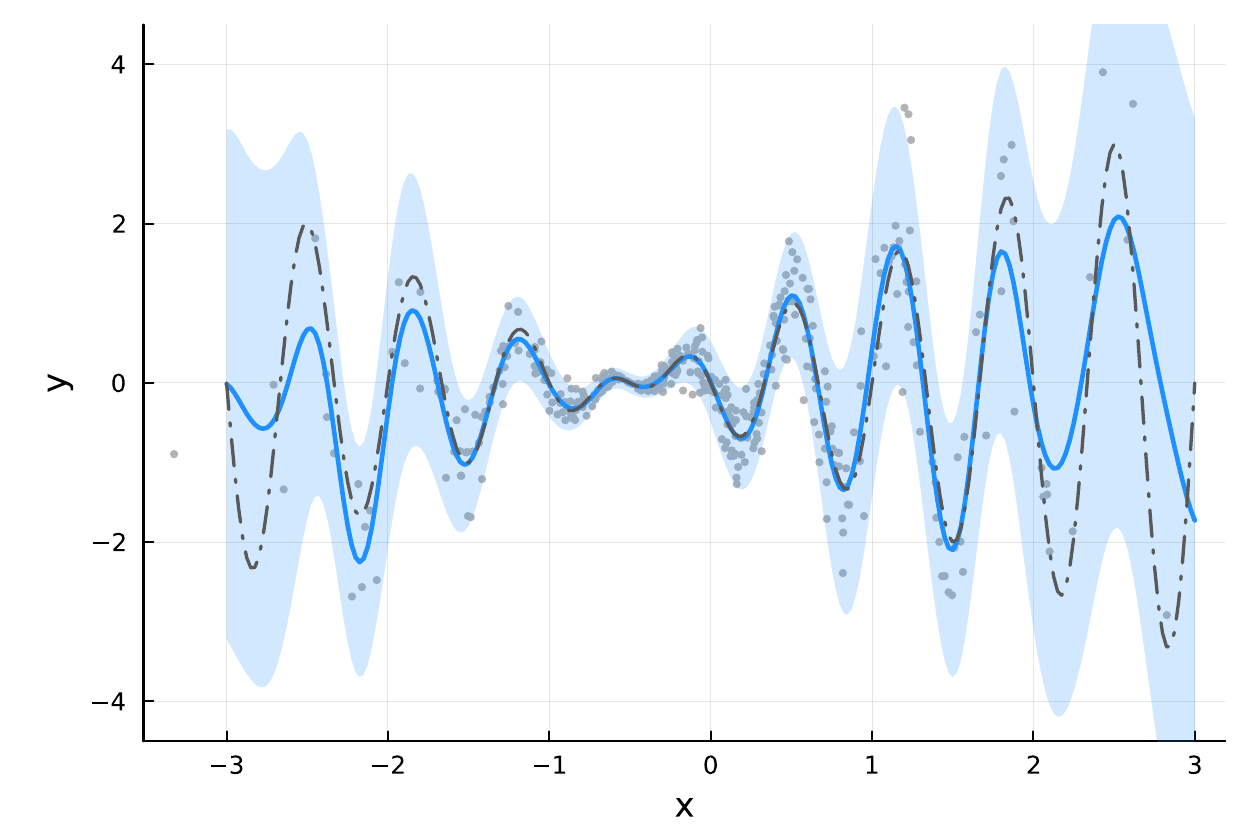}
        \caption{NGMP, full batch}
        \label{fig:app-streaming-ngmp-full}
    \end{subfigure}
    \hfill
    \begin{subfigure}[t]{0.49\linewidth}
        \centering
        \includegraphics[width=\linewidth]{figures/streaming_cavity_sequential.pdf}
        \caption{NGMP, sequential batches}
        \label{fig:app-streaming-ngmp-sequential}
    \end{subfigure}
    \caption{Posterior predictive bands of all four fits on one
    representative seed with identical axes. Orange dashed curves are PVMP
    and blue solid curves are NGMP; the matching shaded regions are
    pointwise 95\% posterior-predictive bands. Gray dash-dotted curves show
    the true mean, and gray points are the 400 training observations.
    Panels (a) and (c) use all observations at once, whereas panels (b) and
    (d) use the ten-step filtering chain of
    \Cref{fig:edge-uncertainty-filtering-chain}. Panels (a), (c), and (d)
    are visually indistinguishable; only PVMP's sequential fit in panel (b)
    collapses.}
    \label{fig:app-streaming-predictive}
\end{figure}

\begin{figure}[t!]
    \centering
    \begin{subfigure}[t]{0.49\linewidth}
        \centering
        \includegraphics[width=\linewidth]{figures/streaming_vmp_variance_full.pdf}
        \caption{PVMP, full batch}
        \label{fig:app-streaming-variance-vmp-full}
    \end{subfigure}
    \hfill
    \begin{subfigure}[t]{0.49\linewidth}
        \centering
        \includegraphics[width=\linewidth]{figures/streaming_vmp_variance_sequential.pdf}
        \caption{PVMP, sequential batches}
        \label{fig:app-streaming-variance-vmp-sequential}
    \end{subfigure}\\[2pt]
    \begin{subfigure}[t]{0.49\linewidth}
        \centering
        \includegraphics[width=\linewidth]{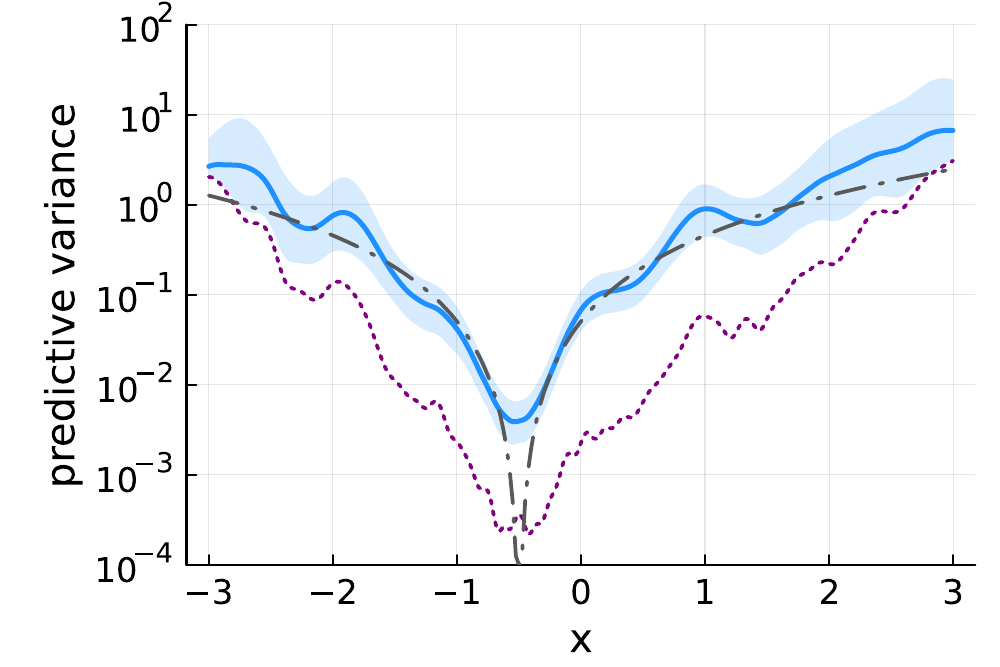}
        \caption{NGMP, full batch}
        \label{fig:app-streaming-variance-ngmp-full}
    \end{subfigure}
    \hfill
    \begin{subfigure}[t]{0.49\linewidth}
        \centering
        \includegraphics[width=\linewidth]{figures/streaming_cavity_variance_sequential.pdf}
        \caption{NGMP, sequential batches}
        \label{fig:app-streaming-variance-ngmp-sequential}
    \end{subfigure}
    \caption{Posterior predictive variance of all four fits on the same
    representative seed as \Cref{fig:app-streaming-predictive}. Panels (a)
    and (c) use all observations at once, whereas panels (b) and (d) use ten
    sequential batches. The orange long-dashed curves in panels (a)--(b) are
    PVMP and the blue solid curves in panels (c)--(d) are NGMP. Each is the
    posterior mean total predictive variance,
    $\bar S(x)=V_{\mathrm{epi}}(x)+V_{\mathrm{alea}}(x)$, where
    $V_{\mathrm{epi}}(x)=\phi(x)^\top\Sigma_v\phi(x)$ is epistemic variance
    from the uncertain mean weights and
    $V_{\mathrm{alea}}(x)=\mathbb E_q[e^{-s(x)}]$ is learned aleatoric
    variance. Shading gives a pointwise 95\% posterior credible interval for
    the total. The purple dotted curve shows $V_{\mathrm{epi}}(x)$ alone, so
    its vertical gap to the orange or blue total is $V_{\mathrm{alea}}(x)$.
    The gray dash-dotted curve is the benchmark's true aleatoric variance
    $[0.45(x+0.5)]^2$. All panels use identical logarithmic axes.}
    \label{fig:app-streaming-variance}
\end{figure}

\paragraph{Budget-matched control: NCVMP.}
The runtime of the two methods is dominated by different quantities: NGMP
takes one natural-gradient step per edge update, while PVMP's projective
update runs an inner manifold optimizer on every edge --- up to $100$
gradient evaluations per update in our configuration. The control reported
here matches the budgets from PVMP's side: same graph, constraints,
initialization, and outer sweep schedule, but every \texttt{ProjectedTo}
call is limited to a \emph{single} inner iteration, so each edge update
costs one gradient evaluation, as in NGMP. A single natural-gradient step of
the projective update at the current marginal is the non-conjugate VMP
update of \citet{knowles_nonconjugate_2011}, so this control is NCVMP
(\Cref{sec:related-work}). NCVMP degrades predictive
likelihood at the matched runtime. On the sunspot model
(\Cref{tab:app-budget-poisson}) its held-out negative
log-likelihood is $12.5$--$16.1$ against $4.5$--$5.4$ for both converged
methods, at every holdout fraction. On the sequential heteroscedastic model
(\Cref{tab:app-budget-streaming}), its full-batch NLL is $0.913$ against
$0.276$ for PVMP and $0.235$ for NGMP, while its sequential NLL is $0.916$
against $0.512$ and $0.244$, respectively;
its noise-weight posterior nevertheless collapses along converged PVMP's
exact concentration path (\Cref{fig:edge-uncertainty-collapse}), so the
overconcentration of \Cref{sec:edge-uncertainty} is not an artifact of
over-optimizing the projections, weaker per-edge optimization neither
repairs the calibration nor preserves the fit. The NCVMP Bethe traces
are the dotted curves in \Cref{fig:app-poisson-fe,fig:app-streaming-fe}.

\begin{table}[t!]
    \centering
    \caption{Sunspot state-space model with the budget-matched control
    (mean $\pm$ 95\% CI over 20 random masks): the two converged methods of
    \Cref{tab:edge-uncertainty-poisson} alongside NCVMP, that is, PVMP with
    single-step projections.}
    \label{tab:app-budget-poisson}
    \resizebox{\linewidth}{!}{%
        \begin{tabular}{rcccccc}
\toprule
 & \multicolumn{3}{c}{NLL} & \multicolumn{3}{c}{RMSE} \\
\cmidrule(lr){2-4} \cmidrule(lr){5-7}
Held out & PVMP & NCVMP & NGMP & PVMP & NCVMP & NGMP \\
\midrule
5\% & 4.534 $\pm$ 0.087 & 12.492 $\pm$ 0.546 & 4.532 $\pm$ 0.085 & 13.918 $\pm$ 0.385 & 26.441 $\pm$ 0.951 & 13.935 $\pm$ 0.378 \\
10\% & 4.569 $\pm$ 0.058 & 12.998 $\pm$ 0.453 & 4.566 $\pm$ 0.058 & 13.972 $\pm$ 0.260 & 27.137 $\pm$ 0.768 & 13.995 $\pm$ 0.260 \\
20\% & 4.672 $\pm$ 0.039 & 13.547 $\pm$ 0.410 & 4.661 $\pm$ 0.037 & 14.429 $\pm$ 0.214 & 27.616 $\pm$ 0.549 & 14.446 $\pm$ 0.208 \\
50\% & 5.411 $\pm$ 0.162 & 16.116 $\pm$ 0.391 & 4.930 $\pm$ 0.035 & 16.839 $\pm$ 0.471 & 30.729 $\pm$ 0.507 & 15.789 $\pm$ 0.162 \\
\bottomrule
\end{tabular}
    }
\end{table}

\begin{table}[t!]
    \centering
    \small
    \caption{Sequential heteroscedastic model with the budget-matched
    NCVMP control (mean $\pm$ 95\% CI over 20 paired seeds), extending
    \Cref{tab:edge-uncertainty-streaming}.}
    \label{tab:app-budget-streaming}
    \begin{tabular}{lccccc}
    \toprule
    & \multicolumn{2}{c}{full batch} & \multicolumn{2}{c}{sequential} & \\
    \cmidrule(lr){2-3}\cmidrule(lr){4-5}
    & NLL & RMSE & NLL & RMSE & batching penalty \\
    \midrule
    PVMP & $0.276 \pm 0.031$ & $0.588 \pm 0.037$ &
        $0.512 \pm 0.044$ & $0.583 \pm 0.039$ & $0.235 \pm 0.030$ \\
    NCVMP & $0.913 \pm 0.015$ & $0.577 \pm 0.035$ &
        $0.916 \pm 0.017$ & $0.576 \pm 0.035$ & $0.003 \pm 0.005$ \\
    NGMP & $0.235 \pm 0.031$ & $0.591 \pm 0.038$ &
        $0.244 \pm 0.031$ & $0.607 \pm 0.042$ & $0.009 \pm 0.010$ \\
    \bottomrule
    \end{tabular}
\end{table}

\section{Experiment details and additional results}
\label{app:modelling-details}
\subsection{Regression details}
\label{app:uci-details}

\Cref{tab:app-uci-dataset-sizes} reports the dimensions before the repeated
90/10 train--test splits. The feature count is the number of raw input
variables before standardization and random-feature construction.

\begin{table}[t]
    \centering
    \caption{UCI regression data sets used in \Cref{sec:experiments-regression}.
    The sample count is the complete data set before each repeated holdout.}
    \label{tab:app-uci-dataset-sizes}
    \begin{tabular}{lrr}
        \toprule
        Data set & Samples & Input features \\
        \midrule
        Concrete & 1{,}030 & 8 \\
        Energy & 768 & 8 \\
        Boston Housing & 506 & 13 \\
        Power Plant & 9{,}568 & 4 \\
        Wine Quality Red & 1{,}599 & 11 \\
        Yacht & 308 & 6 \\
        \bottomrule
    \end{tabular}
\end{table}

\Cref{tab:app-uci-model-capacity} separates two numerical notions of capacity.
$P$ is the number of predictive weights and biases, whereas $U$ is the number
of additional independent scalars required to represent global posterior
uncertainty beyond its mean. Both exclude fixed random features, optimizer
state, and observation-local messages. Here $d$ is the raw feature count from
\Cref{tab:app-uci-dataset-sizes}.

\begin{table*}[t]
    \centering
    \scriptsize
    \setlength{\tabcolsep}{5pt}
    \caption{Numerical model and uncertainty capacity in the UCI comparison.
    NGMP uses the hierarchy in \Cref{fig:heteroscedastic-hierarchy-ffg} with
    $L=2$: one mean-weight vector and two precision-weight vectors. Each vector
    has 1{,}000 random-feature coefficients, $d$ direct-input coefficients,
    and one intercept. All ranges are exact over the six data sets.}
    \label{tab:app-uci-model-capacity}
    \begin{tabular}{@{}llclcl@{}}
        \toprule
        Method & Architecture & \shortstack{Model\\coefficients $P$} &
        \shortstack{Posterior /\\propagation} &
        \shortstack{Uncertainty\\parameters $U$} & Noise \\
        \midrule
        hoBBB & $d\!\to\!50\!\to\!50\!\to\!1$ & 2{,}851--3{,}301 & Diag.-G & 2{,}851--3{,}301 & fixed \\
        BBB & $d\!\to\!50\!\to\!50\!\to\!2$ & 2{,}902--3{,}352 & Diag.-G & 2{,}902--3{,}352 & input \\
        hodDVI & $d\!\to\!50\!\to\!1$ & 301--751 & Diag.-G / 50 & 301--751 & fixed \\
        dDVI & $d\!\to\!50\!\to\!2$ & 352--802 & Diag.-G / 50 & 352--802 & input \\
        hoDVI & $d\!\to\!50\!\to\!1$ & 301--751 & Diag.-G / 1{,}275 & 301--751 & fixed \\
        DVI & $d\!\to\!50\!\to\!2$ & 352--802 & Diag.-G / 1{,}275 & 352--802 & input \\
        hoBPC & $d\!\to\!50\!\to\!50\!\to\!1$ & 2{,}851--3{,}301 & MN--W & 5{,}221--5{,}311 & global \\
        hoIVON & $d\!\to\!50\!\to\!50\!\to\!1$ & 2{,}851--3{,}301 & Diag.-G & 2{,}851--3{,}301 & fixed \\
        NGMP & \shortstack{Hierarchy, $L=2$\\\Cref{fig:heteroscedastic-hierarchy-ffg}} & 3{,}015--3{,}042 & Full-G / layer & 1{,}516{,}545--1{,}543{,}815 & 2-level input \\
        \bottomrule
    \end{tabular}
    \vspace{3pt}
    \begin{minipage}{0.94\textwidth}
        \raggedright
        Diag.-G and Full-G denote diagonal- and full-covariance Gaussian
        weight beliefs; MN--W denotes Matrix-Normal--Wishart. Exactly,
        $U=P$ for Diag.-G, $U=((d+1)(d+2))/2+5{,}206$ for MN--W, and
        $U=3(d+1{,}001)(d+1{,}002)/2$ for NGMP. For DVI/dDVI, the number after the slash is
        the propagated hidden second-moment state per observation: 50 marginal
        variances for dDVI versus 1{,}275 unique covariance entries for DVI.
    \end{minipage}
\end{table*}

Each of the feature maps ($\phi$ for the mean layer, $\psi_\ell$ for the
precision layers) contains 1{,}000 random Fourier features approximating a
Mat\'ern-$3/2$ kernel, concatenated with the standardized input and an
intercept.  The frequencies are divided equally among three spectral scales
$(0.5,1,2)$, giving a multiscale representation.  The layer length scales
decrease geometrically as
$(1.5,1.5/\sqrt{2},0.75)$ from the mean layer to the deepest precision layer.
The feature maps are sampled once per train--test split and then held fixed.
Gaussian weight priors are used at all layers.  The first precision-layer
intercept is initialized from the homoscedastic residual variance, clipped to
$[10^{-3},10^3]$, and the upper carrier precision is fixed to
$\tau_{\mathrm{top}}=25$.

We use at most 60 sweeps, stop when the relative change of the concatenated
weight means falls below $10^{-5}$, and stabilize linear solves with diagonal
jitter $10^{-8}$.  Natural parameters are updated with the vector-transport
momentum rule of \Cref{app:vector-transport-momentum}, with step parameter
$\alpha=0.6$, momentum parameter $\beta=0.8$, and maximum step $0.5$.

The ho prefix in \Cref{tab:app-uci-nll,tab:uci-rmse} marks a homoscedastic
baseline. BBB, DVI, and IVON use fixed unit observation variance in
standardized target units; BPC instead learns one global output covariance
through its final Matrix-Normal--Wishart layer. The heteroscedastic BBB and
DVI variants add a clipped input-dependent log-variance head. The hoIVON
predictive is evaluated as a mixture over $K=20$ posterior network draws.

\begin{table*}[ht]
\centering
\caption{Full test negative log likelihood (NLL) in original target units for
all baseline variants; lower is better. Values are means \(\pm\) approximate
95\% confidence-interval half-widths over 20 paired splits. Bold values
include the best point estimate within their confidence interval.}
\label{tab:app-uci-nll}
\resizebox{\textwidth}{!}{%
\begin{tabular}{lrrrrrr}
\toprule
Method & Concrete & Energy & Boston & Power & Wine & Yacht \\
\midrule
hoBBB & $3.8613 \pm 0.0059$ & $3.2892 \pm 0.0043$ & $3.2488 \pm 0.0157$ & $3.7960 \pm 0.0019$ & $1.0336 \pm 0.0173$ & $3.6925 \pm 0.0059$ \\
BBB & $3.3328 \pm 0.0215$ & $2.4229 \pm 0.0350$ & $2.6555 \pm 0.0385$ & $2.7785 \pm 0.0143$ & $\bm{0.9438 \pm 0.0232}$ & $2.6565 \pm 0.0451$ \\
hodDVI & $3.8602 \pm 0.0066$ & $3.2970 \pm 0.0027$ & $3.2571 \pm 0.0111$ & $3.7918 \pm 0.0011$ & $1.0360 \pm 0.0164$ & $3.7236 \pm 0.0097$ \\
dDVI & $3.0477 \pm 0.0432$ & $\bm{1.1920 \pm 0.4625}$ & $\bm{2.4646 \pm 0.0921}$ & $2.8238 \pm 0.0183$ & $\bm{0.9440 \pm 0.0314}$ & $0.4641 \pm 0.1000$ \\
hoDVI & $3.8602 \pm 0.0066$ & $3.2971 \pm 0.0027$ & $3.2572 \pm 0.0111$ & $3.7918 \pm 0.0011$ & $1.0360 \pm 0.0164$ & $3.7238 \pm 0.0097$ \\
DVI & $3.0406 \pm 0.0438$ & $\bm{1.0094 \pm 0.2973}$ & $\bm{2.4489 \pm 0.0917}$ & $2.8249 \pm 0.0200$ & $\bm{0.9448 \pm 0.0307}$ & $0.4543 \pm 0.1096$ \\
hoBPC & $3.9874 \pm 0.0202$ & $2.6357 \pm 0.0526$ & $2.9819 \pm 0.0862$ & $3.5396 \pm 0.0159$ & $1.0545 \pm 0.0256$ & $3.0251 \pm 0.0712$ \\
hoIVON ($K=20$) & $3.7966 \pm 0.0082$ & $3.2693 \pm 0.0412$ & $3.3284 \pm 0.0855$ & $3.7839 \pm 0.0013$ & $1.0430 \pm 0.0156$ & $3.6584 \pm 0.0101$ \\
NGMP & $\bm{2.9845 \pm 0.0433}$ & $\bm{0.9783 \pm 0.0655}$ & $2.6720 \pm 0.0752$ & $\bm{2.7610 \pm 0.0226}$ & $\bm{0.9571 \pm 0.0293}$ & $\bm{0.2782 \pm 0.1437}$ \\
\bottomrule
\end{tabular}%
}
\end{table*}

\begin{table*}[ht]
\centering
\caption{Test root mean squared error (RMSE) in original target units for all
baseline variants; lower is better. Values are means \(\pm\) approximate 95\%
confidence-interval half-widths over 20 paired splits. Bold values include the
best point estimate within their confidence interval.}
\label{tab:uci-rmse}
\resizebox{\textwidth}{!}{%
\begin{tabular}{lrrrrrr}
\toprule
Method & Concrete & Energy & Boston & Power & Wine & Yacht \\
\midrule
hoBBB & $7.2822 \pm 0.2072$ & $2.8175 \pm 0.1459$ & $\bm{3.8392 \pm 0.4243}$ & $4.2831 \pm 0.0703$ & $\bm{0.6469 \pm 0.0185}$ & $3.4875 \pm 0.4063$ \\
BBB & $6.7086 \pm 0.3117$ & $2.7537 \pm 0.1548$ & $\bm{3.7959 \pm 0.4387}$ & $\bm{4.0251 \pm 0.0737}$ & $\bm{0.6470 \pm 0.0176}$ & $3.0494 \pm 0.4150$ \\
hodDVI & $6.7179 \pm 0.2764$ & $2.2486 \pm 0.1159$ & $\bm{3.4801 \pm 0.3374}$ & $4.3009 \pm 0.0719$ & $\bm{0.6484 \pm 0.0190}$ & $2.1779 \pm 0.2301$ \\
dDVI & $\bm{5.6591 \pm 0.3352}$ & $2.1088 \pm 0.3088$ & $\bm{3.4741 \pm 0.4119}$ & $4.1321 \pm 0.0763$ & $\bm{0.6463 \pm 0.0198}$ & $\bm{0.8491 \pm 0.1387}$ \\
hoDVI & $6.7179 \pm 0.2766$ & $2.2493 \pm 0.1161$ & $\bm{3.4802 \pm 0.3375}$ & $4.3009 \pm 0.0720$ & $\bm{0.6484 \pm 0.0190}$ & $2.1787 \pm 0.2301$ \\
DVI & $\bm{5.6654 \pm 0.3249}$ & $1.9434 \pm 0.3353$ & $\bm{3.4291 \pm 0.4013}$ & $4.1428 \pm 0.0801$ & $\bm{0.6469 \pm 0.0197}$ & $\bm{0.8528 \pm 0.1533}$ \\
hoBPC & $12.9488 \pm 0.2671$ & $3.3192 \pm 0.2053$ & $4.7579 \pm 0.5024$ & $8.2997 \pm 0.1162$ & $0.6971 \pm 0.0182$ & $4.8077 \pm 0.4786$ \\
hoIVON ($K=20$) & $\bm{5.6951 \pm 0.3808}$ & $2.2252 \pm 0.7692$ & $5.0444 \pm 1.1487$ & $\bm{3.9703 \pm 0.0843}$ & $0.6627 \pm 0.0161$ & $2.6593 \pm 0.3656$ \\
NGMP & $6.1986 \pm 0.3871$ & $\bm{0.7198 \pm 0.0534}$ & $\bm{3.8201 \pm 0.3952}$ & $\bm{3.9541 \pm 0.0729}$ & $\bm{0.6582 \pm 0.0189}$ & $\bm{0.9295 \pm 0.1548}$ \\
\bottomrule
\end{tabular}%
}
\end{table*}

\FloatBarrier
\subsection{Ensemble forecasting details}
\label{app:etth-details}

We consider univariate long-horizon forecasting on the ETTh1 and ETTh2
electricity-transformer data sets \citep{zhou_informer_2021}, whose observations
are hourly. From the preceding 96-hour multivariate context, the task is to
predict one standardized oil-temperature target (\texttt{OT})
$H\in\{96,192,336,720\}$ hours ahead. Following
\citet{lukashchuk_composing_2026}, the frozen expert bank consists of five
independently trained neural forecasters: CNN, DLinear, NLinear
\citep{zeng_are_2023}, LSTM \citep{hochreiter_long_1997}, and NConv. Constant
forecasts at the 10th and 90th training-set quantiles complete the frozen
seven-expert bank. The gate is fitted on the original validation partition and
evaluated on the original test partition; no expert or VAE representation is
retrained.
For the IVON gates, $K=1000$ denotes 1{,}000 draws of the gate weights from
the fitted variational posterior. For logits $z^{(k)}$ from draw $k$, the
Gaussian component mean is
$\mu^{(k)}=\sum_j\operatorname{softmax}(z^{(k)})_j f_j$, where $f_j$ is the
$j$th frozen expert prediction, and its variance is
$(\sigma^2)^{(k)}=(\sum_j\exp z_j^{(k)})^{-1}$. The reported posterior
predictive is the equally weighted mixture of these 1{,}000 components. Its
NLL is evaluated with log-sum-exp over the component densities, and its point
prediction is the average of the component means.
\Cref{tab:etth-precision-gated-full} reports the full ETTh results, including
the Adam-trained gates omitted from the main table.

\begin{table*}[p]
\centering\scriptsize
\caption{Full descriptive ETTh1 and ETTh2 RMSE and NLL results. Intervals are 95\% descriptive confidence-interval half-widths computed per test origin; RMSE intervals use the delta method. Counts are 3446, 3426, 3398, and 3321 at horizons 96, 192, 336, and 720. The PGE--PVMP rows are taken from \citet{lukashchuk_composing_2026}, computed on the same test split.}
\label{tab:etth-precision-gated-full}
\resizebox{\textwidth}{!}{%
\begin{tabular}{lrrrr}
\toprule
\multicolumn{5}{c}{ETTh1} \\
\cmidrule(lr){1-5}
Method & 96 & 192 & 336 & 720 \\
\midrule
\multicolumn{5}{l}{\textit{RMSE}} \\
MoE -- Adam, affine gate & \(0.3836 \pm 0.0090\) & \(0.3699 \pm 0.0091\) & \(0.4152 \pm 0.0099\) & \(0.4639 \pm 0.0111\) \\
MoE -- Adam, ReLU gate & \(0.3955 \pm 0.0097\) & \(0.3699 \pm 0.0091\) & \(0.4801 \pm 0.0123\) & \(0.4642 \pm 0.0112\) \\
MoE -- IVON, affine gate ($K{=}1000$) & \(0.3832 \pm 0.0090\) & \(0.3695 \pm 0.0091\) & \(0.3672 \pm 0.0084\) & \(0.5428 \pm 0.0127\) \\
MoE -- IVON, ReLU gate ($K{=}1000$) & \(0.3832 \pm 0.0090\) & \(0.3697 \pm 0.0091\) & \(0.3677 \pm 0.0084\) & \(0.5334 \pm 0.0125\) \\
\addlinespace[2pt]
PGE -- PVMP \citep{lukashchuk_composing_2026} & \(0.3583 \pm 0.0083\) & \(0.3386 \pm 0.0085\) & \(0.3105 \pm 0.0073\) & \(0.3347 \pm 0.0076\) \\
PGE -- NGMP (this work) & \(0.3554 \pm 0.0084\) & \(0.3337 \pm 0.0086\) & \(0.3115 \pm 0.0074\) & \(0.3300 \pm 0.0080\) \\
\midrule
\multicolumn{5}{l}{\textit{NLL}} \\
MoE -- Adam, affine gate & \(2.0553\!\times\!10^{25} \pm 1.5536\!\times\!10^{25}\) & \(1.1050\!\times\!10^{22} \pm 1.0354\!\times\!10^{22}\) & \(1.1159\!\times\!10^{22} \pm 6.0082\!\times\!10^{21}\) & \(9.0836\!\times\!10^{11} \pm 6.4103\!\times\!10^{11}\) \\
MoE -- Adam, ReLU gate & \(2.2759\!\times\!10^{24} \pm 1.1906\!\times\!10^{24}\) & \(1.2482\!\times\!10^{19} \pm 8.2942\!\times\!10^{18}\) & \(2.4524\!\times\!10^{14} \pm 1.7209\!\times\!10^{14}\) & \(1.5279\!\times\!10^{17} \pm 1.0479\!\times\!10^{17}\) \\
MoE -- IVON, affine gate ($K{=}1000$) & \(10.2777 \pm 0.7366\) & \(3.5515\!\times\!10^{4} \pm 1.1067\!\times\!10^{4}\) & \(429.0098 \pm 143.4073\) & \(168.4529 \pm 26.4000\) \\
MoE -- IVON, ReLU gate ($K{=}1000$) & \(2216.8905 \pm 676.2776\) & \(2.2344\!\times\!10^{8} \pm 9.2280\!\times\!10^{7}\) & \(58.6475 \pm 9.6980\) & \(94.2422 \pm 16.2638\) \\
\addlinespace[2pt]
PGE -- PVMP \citep{lukashchuk_composing_2026} & \(0.4120 \pm 0.0173\) & \(0.3701 \pm 0.0171\) & \(0.3141 \pm 0.0135\) & \(0.3763 \pm 0.0140\) \\
PGE -- NGMP (this work) & \(0.3888 \pm 0.0210\) & \(0.3378 \pm 0.0200\) & \(0.2877 \pm 0.0161\) & \(0.3571 \pm 0.0153\) \\
\bottomrule
\end{tabular}%
}
\medskip
\resizebox{\textwidth}{!}{%
\begin{tabular}{lrrrr}
\toprule
\multicolumn{5}{c}{ETTh2} \\
\cmidrule(lr){1-5}
Method & 96 & 192 & 336 & 720 \\
\midrule
\multicolumn{5}{l}{\textit{RMSE}} \\
MoE -- Adam, affine gate & \(0.5821 \pm 0.0134\) & \(0.5655 \pm 0.0139\) & \(0.5190 \pm 0.0121\) & \(0.6201 \pm 0.0129\) \\
MoE -- Adam, ReLU gate & \(0.5592 \pm 0.0128\) & \(0.5655 \pm 0.0139\) & \(0.5190 \pm 0.0121\) & \(0.6201 \pm 0.0129\) \\
MoE -- IVON, affine gate ($K{=}1000$) & \(0.5821 \pm 0.0134\) & \(0.5198 \pm 0.0129\) & \(0.5869 \pm 0.0142\) & \(0.8249 \pm 0.0181\) \\
MoE -- IVON, ReLU gate ($K{=}1000$) & \(0.5821 \pm 0.0134\) & \(0.5655 \pm 0.0139\) & \(0.5874 \pm 0.0142\) & \(0.6202 \pm 0.0129\) \\
\addlinespace[2pt]
PGE -- PVMP \citep{lukashchuk_composing_2026} & \(0.5882 \pm 0.0120\) & \(0.5798 \pm 0.0127\) & \(0.5940 \pm 0.0128\) & \(0.5669 \pm 0.0131\) \\
PGE -- NGMP (this work) & \(0.5631 \pm 0.0124\) & \(0.5408 \pm 0.0128\) & \(0.5463 \pm 0.0125\) & \(0.6101 \pm 0.0143\) \\
\midrule
\multicolumn{5}{l}{\textit{NLL}} \\
MoE -- Adam, affine gate & \(3.0584\!\times\!10^{23} \pm 2.2858\!\times\!10^{23}\) & \(2.8850\!\times\!10^{23} \pm 2.5057\!\times\!10^{23}\) & \(5.8244\!\times\!10^{19} \pm 3.7745\!\times\!10^{19}\) & \(2.0804\!\times\!10^{19} \pm 1.3583\!\times\!10^{19}\) \\
MoE -- Adam, ReLU gate & \(1.1333\!\times\!10^{31} \pm 9.7716\!\times\!10^{30}\) & \(7.0543\!\times\!10^{24} \pm 6.0491\!\times\!10^{24}\) & \(1.4496\!\times\!10^{23} \pm 1.2198\!\times\!10^{23}\) & \(4.6866\!\times\!10^{38} \pm 4.7607\!\times\!10^{38}\) \\
MoE -- IVON, affine gate ($K{=}1000$) & \(5.9904\!\times\!10^{21} \pm 6.4398\!\times\!10^{21}\) & \(304.6987 \pm 52.2098\) & \(3.7880\!\times\!10^{5} \pm 1.0683\!\times\!10^{5}\) & \(3.2209\!\times\!10^{5} \pm 7.4793\!\times\!10^{4}\) \\
MoE -- IVON, ReLU gate ($K{=}1000$) & \(1.9494\!\times\!10^{18} \pm 1.3836\!\times\!10^{18}\) & \(1.1309\!\times\!10^{5} \pm 6.8836\!\times\!10^{4}\) & \(2.0843\!\times\!10^{6} \pm 5.8045\!\times\!10^{5}\) & \(1.1263\!\times\!10^{6} \pm 2.1811\!\times\!10^{5}\) \\
\addlinespace[2pt]
PGE -- PVMP \citep{lukashchuk_composing_2026} & \(0.9342 \pm 0.0306\) & \(0.9237 \pm 0.0328\) & \(0.9612 \pm 0.0334\) & \(0.8699 \pm 0.0299\) \\
PGE -- NGMP (this work) & \(0.9397 \pm 0.0386\) & \(0.8602 \pm 0.0365\) & \(0.8702 \pm 0.0348\) & \(0.9774 \pm 0.0357\) \\
\bottomrule
\end{tabular}%
}
\end{table*}


\end{document}